\documentclass[twoside,11pt]{article}

\usepackage{jmlr2e}
\usepackage{microtype}
\usepackage{graphicx}
\usepackage{booktabs} 
\usepackage{hyperref}
\usepackage{thmtools} 
\usepackage{thm-restate}
\usepackage{bm}
\usepackage{soul}

\usepackage{amsmath}
\usepackage{amssymb}
\usepackage{mathtools}
\usepackage{comment}
\usepackage[capitalize,noabbrev]{cleveref}

\newtheorem{assumption}[theorem]{Assumption}
\usepackage[textsize=tiny]{todonotes}

\usepackage{influence-diagrams}
\usepackage{enumerate}
\usepackage[ruled,vlined,algo2e]{algorithm2e}
\usepackage{color,soul}
\usepackage{caption,subcaption}
\usepackage{thmtools}
\usepackage{thm-restate}
\usepackage{bbm}
\usepackage[ruled,vlined,algo2e]{algorithm2e}
\usepackage{algorithm}
\usepackage{lastpage}
\def\ms#1{\textcolor{blue}{#1}}

\usepackage{soul}

\def\mbb{\mathbb}
\def\mcal{\mathcal}

\def\reals{\mathbb{R}}
\def\ones{\boldsymbol{1}}
\def\abs#1{|#1|}

\usepackage{lastpage}
\jmlrheading{24}{2023}{1-\pageref{LastPage}}{9/22; Revised
5/23}{8/23}{22-1081}{Madhumitha Shridharan and Garud Iyengar}
\ShortHeadings{Scalable Computation of Causal Bounds}{Shridharan and Iyengar}
\firstpageno{1}

\begin{document}

\title{Scalable Computation of Causal Bounds}

\author{\name Madhumitha Shridharan \email ms6143@columbia.edu 
       \AND
       \name Garud Iyengar \email gi10@columbia.edu \\
       \addr Department of Industrial Engineering and Operations Research\\
       Columbia University\\ New York City, NY 10027, USA}

\editor{Jin Tian}

\maketitle

\begin{abstract}
We consider the problem of computing bounds for causal queries on causal graphs with unobserved confounders and discrete valued observed variables, where identifiability does not hold. Existing non-parametric approaches for computing such bounds use linear programming (LP) formulations that quickly become intractable for existing solvers because the size of the LP grows exponentially in the number of edges in the causal graph. We show that this LP can be significantly pruned, allowing us to compute bounds for significantly larger causal inference problems compared to existing techniques. This pruning procedure allows us to compute bounds in closed form for a special class of problems, including a well-studied family of problems where multiple confounded treatments influence an outcome. We extend our pruning methodology to fractional LPs which compute bounds for causal queries which incorporate additional observations about the unit. We show that our methods provide significant runtime improvement compared to benchmarks in experiments and extend our results to the finite data setting. For causal inference without additional observations, we propose an efficient greedy heuristic that produces high quality bounds, and scales to problems that are several orders of magnitude larger than those for which the pruned LP can be solved.
\end{abstract}

\section{Introduction}
\label{sec:introduction}
We are interested in answering the following counterfactual query about a
large-scale system of discrete variables: What will be the value of some
outcome variables
\(V_O\) if we intervene on variables \(V_I\), given the values of variables \(V_A\) are
known? Several meaningful questions in data-rich environments can be
formulated this way. As an example, % consider the well-studied setting
  % where
  consider 
the causal graph  in \Cref{fig:specialcasewithconfounders} that describes
the health outcome of a patient $Y$ as a function of a sequence of treatments administered $T_i$ by a physician, $i =
1,\ldots, 5$. The treatments chosen are functions of some patient characteristics $C_i$, $i = 1,2$,
e.g. sex and age. The variable $U_B$
refers to unknown variables, also known as confounders, that might impact
both the choice of treatments and the health outcome e.g. patient lifestyle, physician biases, etc. What will be the expected health outcome \(Y\) of the patient if she is administered treatments \(T_i, i=1,\ldots,5\), given her sex and age are known?~\citep{ranganath2019multiple, 
  janzing2018,pmlr-v89-d-amour19a,tran2017implicit}

Traditional approaches to estimate causal effects of interventions % require
% extensive collection of experimental data (e.g. via
involve randomized control
trials (RCTs) in order to remove the impact of confounders. However,
running experiments to identify personalized 
interventions for sub-populations of units is % financially
often expensive
and
practically infeasible. % We hence rely on
Therefore, there is a push to develop techniques that can use
observational data % in such
% settings, which
that
is far more readily available.  

The challenge in observational studies is to account for  % that several
% factors which influence outcomes can be 
% unknown or unmeasured, leading to  
unobserved confounders which can create
spurious correlations and adversely impact data-driven
decision-making~\citep{imbens2015causal,pearl2009causality}.  
% that does not account for these confounders appropriately
For example, the unknown confounder $U_B$
in~\Cref{fig:specialcasewithconfounders} % shows the influence of such 
% unobserved confounders in the setting where a patient is prescribed
% multiple treatments. The unobserved confounder \(U_B\)  denotes a
% patient's lifestyle choices and 
influences both the prescribed treatments
and the outcome: a patient who exercises regularly % could
may
have lower body
weight, and thus, require lower dosage of treatments, %  However, exercising
% regularly also improves her metabolic rate, which affects her health
% outcome
but also have an improved response to treatments. Hence, treatment dosage can be negatively correlated with treatment response, although administering lower dosage of treatments need not result in improved response. % When lifestyle is unobserved, the
% treatment prescribed to a patient will be correlated with 
% health outcomes, regardless of its true causal effect. 
Hence, alternate
methodologies need to be developed to compute the causal effect of
treatments on health outcomes in the presence of unobserved confounders.

\begin{figure}[t]
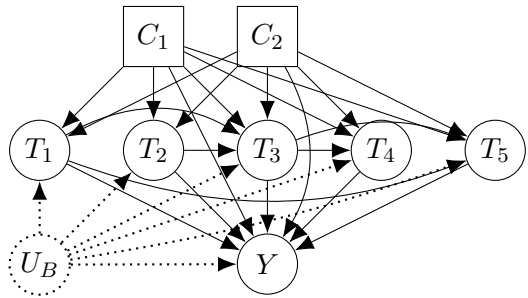


  \centering
  
  \begin{influence-diagram}
    \node (T1) [] {$T_{1}$};
    \node (T2) [right = of T1] {$T_{2}$};
    \node (T3) [right = of T2] {$T_{3}$};
    \node (T4) [right = of T3] {$T_{4}$};
    \node (T5) [right = of T4] {$T_{5}$};
    \node (Y) [below = of T3] {$Y$};
    \node (C1) [above = of T2, decision] {$C_1$};
    \node (C2) [above = of T3, decision] {$C_2$};
    \node (UB) [below = of T1, information] {$U_B$};

    \edge {T1} {Y};
    \edge {T2} {Y};
    \edge {T3} {Y};
    \edge {T4} {Y};
    \edge {T5} {Y};
    \edge {C1} {T1};
    \edge {C1} {T2};
    \edge {C1} {T3};
    \edge {C1} {T4};
    \edge {C1} {T5};
    \edge {C1} {Y};
    \edge {C2} {T1};
    \edge {C2} {T2};
    \edge {C2} {T3};
    \edge {C2} {T4};
    \edge {C2} {T5};
    \edge {T3} {T4}
    \edge {T2} {T3};
    
    \edge[information] {UB} {T1};
    \edge[information] {UB} {T2};
    \edge[information] {UB} {T3};
    \edge[information] {UB} {T4};
    \edge[information] {UB} {Y};
    
    \path (T1) edge[->, bend right=20] (T5);
    \path (T1) edge[->, bend left=30] (T3);
    \path (T3) edge[->, bend left=20] (T5);
    \path (C2) edge[->, bend left=30] (Y);
    \path (UB) edge[->, bend right=5, information] (T5);
  \end{influence-diagram}
  \caption{$C_1$ and $C_2$ denote the sex and age of the patient. These
    influence the treatments, $T_i, i= 1 \ldots 5$, which the patient is
    prescribed. Note that the prescription of one treatment can influence
    the prescription of another (e.g. \(T_3\) is prescribed to manage the
    side effects of \(T_1\)). Together, these variables influence $Y$, the health outcome of the patient. Unobserved confounder
    \(U_B\) denotes unobserved variables, e.g. patient lifestyle,
    physician biases, 
    that impact the choice of treatments, and the health outcome.}
  \label{fig:specialcasewithconfounders}
\end{figure}

While it is impossible to precisely identify causal effects in the presence of unobserved confounders, it is possible to obtain bounds on the \textit{causal query}, the causal effect of interest. There have been multiple such attempts to bound causal effects for small special graphs. \citet{evans2012} bound causal
effects in the special case where any two observed variables are neither
adjacent in the graph, nor share a latent parent. \citet{richardson2014}
bound the causal effect of a treatment on a parameter of interest by
invoking additional (untestable) assumptions and assess how inference
about the 
treatment effect changes as these assumptions are
varied. \citet{kilbertus2020} and \citet{bareinboim2021} develop algorithms to
compute causal bounds for extensions of the instrumental variable model in
a continuous 
setting. \citet{geiger2013} bound causal effects in a model under
specific parametric assumptions. \citet{finkelstein20a} develop a method
for obtaining bounds on causal parameters using rules of probability and
restrictions on counterfactuals implied by causal graphs.  

While fewer in number, there have also been attempts to bound causal
effects in large general graphs. \citet{poderini2020} propose techniques to
compute bounds in special large graphs with multiple instruments and
observed variables. \citet{finkelstein21b} propose a method for 
partial identification in a class of measurement error
models, and \citet{pmlr-v162-zhang22ab} and \citet{duarte2021} propose a polynomial programming based approach to solve general causal inference problems, but their procedure
is computationally intensive for large graphs. %  it reports and continually
% refines non-sharp ranges that are 
% guaranteed to contain the true value when the algorithm is not run to completion. 

In this work, we extend the class of large graphs for which causal effects
can be efficiently bounded. In particular, we focus on a class of causal inference
problems where causal bounds can be obtained  
using linear 
programming~(LP)~\citep{balke94,bareinboimtransfer2017,pearl2009causality, sjolanderbounds2014}. Recently, \citet{sachs2021general} % have
identified a large problem class for which LPs can be used to compute causal bounds, % can be computed
% via linear programming 
and have developed an 
% has been characterized, and a general, non-parametric
algorithm for formulating the objective function and the constraints of the
corresponding LP. This problem class is a generalization of the instrumental variable setting, and is thus widely applicable. 
% constructing the objective and constraints for this problem class has been
% developed~\citep{sachs2021general,sachs2021general21general}.
However, as we describe later, the size of the LP is exponential in the
number of edges in the causal graph, and therefore, the straightforward
formulation of the LP can be tractably solved only for very small causal graphs. 
% However, this algorithm becomes quickly
% intractable as the size of the set of response function variables
% increases exponentially with the number of edges in the causal
% graph. 
% There is hence a need for scalable algorithms which computes causal
% bounds for large causal graphs in this problem class.
In this work, we show how to use the structure of the causal query and the
underlying graph to significantly % reduce
prune 
% the size of
the LP, and as a
consequence, significantly increase the size of the graphs for which the 
LP method remains tractable. This work is the full version of~\citep{pmlr-v162-shridharan22a} and extends the pruning methodology to fractional linear programs that are used to compute bounds for causal inference problems with additional observations about the unit. Our main contributions are as follows:  
\begin{enumerate}[(a)]
\item In Section~\ref{sec:general,pruningLP} we show that the exponential
  number of variables in the LP used to 
    compute causal bounds can be 
  aggregated to reduce % the size of the problem
  the number of variables by several orders of magnitude
  without impacting the
  quality of the bound % See for
  % details. 
  % The reduction in size can be very subsetantial
  -- compare
  \(\abs{R}\) with \(\abs{H}\) in Table~\ref{tab:sizetable}. % (details in
  % Section~\ref{sec:general,pruningLP}).

\item Although we show the bounds can be computed by % there exists
  solving 
  a much smaller LP, % that can be solved
  % to compute the bounds,
  we
  get this
  computational advantage only if the
  pruned LP can be constructed efficiently. In
  Section~\ref{sec:general,pruningLP} we % also
  show that % Theorems~\ref{validityh} and 
  % \ref{thm:CCh} establish that one can construct 
  the pruned LP can be constructed directly, i.e. without
  first constructing the original LP or iterating over its
  variables. These results critically leverage 
  the structure of the LP corresponding to a causal inference problem. In
  particular, they leverage the fact that all possible functions from
  the parents \(pa(V)\) to a variable \(V\) are admissible.
  % Note that,
  % without this second result, the savings implied by the pruned LP cannot
  % be realized. See for details.
  
\item In Section~\ref{sec:closedform}
  we show that the structural
  results that help us construct the 
  pruned LP % allow us to compute the
  lead to closed form 
  bounds % in closed form
  for a % special
  class of causal inference problems. This class of problems includes as a
  special case the well-studied setting in
  \Cref{fig:specialcasewithconfounders} where multiple confounded
  treatments influence a 
  outcome. Moreover, we are able to compute these bounds even when there
  are causal relationships between the treatments. % See
  % Section~\ref{closedformexample} for details.

\item In Section~\ref{sec:general,observations} we extend our pruning methodology to compute bounds for causal queries with additional observations about the unit. In
  this setting, the bounds are computed using fractional LPs. We show that
  the fractional LP  can be converted
  into an LP, and then show how this LP can be pruned. 
\item In Section~\ref{sec:greedy}, we propose a simple greedy heuristic to
  compute approximate solutions for the pruned LPs when there are no additional
  observations. We show that this
  heuristic allows us to compute approximate bounds for much larger graphs with very minimal degradation in performance. 
\end{enumerate}
% \gi{In Section~\ref{sec:general}, we discuss why although we work with
%   causal graphs with binary variables in this paper, generalizing our
%   results to categorical variables is straightforward.} 
The organization of the rest of this paper is as follows. In 
Section~\ref{sec:general} we introduce our formalism. In
Section~\ref{sec:general,pruningLP} 
we introduce our main structural results for pruning the LPs. In
Section~\ref{sec:closedform} we show that the LP bounds can be
computed 
in closed form for a large class of problems % s, and in
% Section~\ref{closedform} we show
and also discuss
an example of this class of
problems. In Section~\ref{sec:general,observations} we show how to incorporate additional observations about the unit in the query to compute updated bounds. In Section~\ref{sec:numericals} we report the results of numerical experiments
validating and extending the methods proposed in Sections~\ref{sec:general,pruningLP} and \ref{sec:general,observations}. In particular, we show the 
significant runtime improvement provided by our methods compared to benchmarks and extend all results in earlier sections to the finite data setting. In Section~\ref{sec:greedy} we introduce our greedy heuristic
and benchmark its performance. Section~\ref{sec:conclusion} discusses
possible extensions. 
\section{Causal Inference Problems}
\label{sec:general}
% We now generalize our example.
Let $G$ denote the causal graph. % Let the
% set $\mbb{N} = \{1,\ldots, n\}$ index the variables
Let \(V_{1}, \ldots, V_{n}\) denote the variables in $G$ in
topologically sorted order, % That is,  $V_1,\ldots,V_{n}$ denotes the
% variables in topological order with respect to $G$.
% We
% let
and
\(N = \{1,\ldots, n\}\) denote the set of indices for the
variables. We assume that each variable $V_i \in \{0,1\}$. Later in this
section we discuss why our proposed techniques automatically apply to the case where
$V_i$ takes discrete values. 
We use lower case letters for the values for the variables, and
the notation \(V_i = v_i\) denotes that the variable \(V_i\) takes the
value \(v_i \in \{0,1\}\). 
For any subset 
% For any $\textbf{S}
% \subseteq \textbf{N}$, define the corresponding set of variables
\(S \subseteq N\), we define 
$V_{S} := \{V_i: i \in S\}$,
and the notation 
% and for a vector 
% % For any 
% $v \in
% \{0,1\}^{\abs{S}}$, 
$V_{S} = v_S$ denotes the variable  $V_i = v_i$, for all 
$i \in S$, % In order to avoid notation, we will assume that value vector
for some 
\(v \in \{0,1\}^{|N|}\). 

We consider a class of ``partitioned" causal graphs with two sets of
variables, \(V_A\) and \(V_B\), where \(V_B\) topologically follow \(V_A\)
(\cite{sachs20}). The \(V_A\) variables % provide
represent contextual or demographic variables (e.g. gender and age of
a patient, past 
purchases, etc.) for a unit, and are always observed. For example, in
\Cref{fig:specialcasewithconfounders}, $V_A = \{C_i: i = 1, 2\}$, the 
patient characteristics.
The \(V_B\)
variables can be observed, 
intervened upon or are the outcomes of interest in a query. In \Cref{fig:specialcasewithconfounders} $V_B
 = \{Y\} \cup \{T_i: i = 1, \ldots, 5\}$.
% We will work in a setting where
\begin{assumption}[\cite{sachs2021general}]\label{ass:validGraph}
    The index set $N$ is partitioned into two sets
    $N = A \cup B$, where 
  \begin{enumerate}[1.]
  \item  \(V_B\) topologically follow \(V_A\),
  \item each variable in \(V_A\) has no parents, but is the parent of some variable in \(V_B\).
  \item $V_B$ variables can have a common unobserved confounder $U_B$, and 
  \item  no pair of variables \((V_i, V_j)\), where \(i \in A\) and \(j
    \in B\), can share an unobserved confounder. 
  \end{enumerate}
\end{assumption}
For example, in the causal graph in Figure \ref{fig:causalgraph},
$V_A = \{Z\}$ and $V_B = \{X,Y\}$. 
% The index set $N$ is partitioned into two sets
% $N = A \cup B$. % A = \{1 \ldots, m\} and $B = \{m+1,\ldots,n\}$
% % i.e. $V_B$ topologically follows $V_A$.
% The \(V_A\) are interpreted as the ``input'' variables and \(V_B\) as the
% ``output'' variables of the causal graph. Consequently, we require
% % where the variable
%  \(V_B\) topologically follow \(V_A\). Thus,
% there exists \(m < n\) such that \(A = \{1, \ldots, m\}\) and \(B =
% \{m+1,\ldots, n\}\).
% The sets of variables 
% The two sets of
% variables $V_A$ and $V_B$ satisfy the following
% conditions:
% The variables \((V_A, V_B)\) satisfy the following conditions:

\begin{figure}[t]
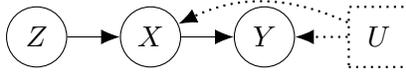

\begin{center}
  \begin{influence-diagram}
    \node (X) [] {$X$};
    \node (Y) [right = of X] {$Y$};
    \node (Z) [left = of X] {$Z$};
    \node (U) [right  = of Y, information, decision] {$U$};
    
    \edge {Z} {X};
    \edge {X} {Y};
    \path (U) edge[->, bend right=25, information] (X);
    \edge[information] {U} {Y};
  \end{influence-diagram}
\end{center}
\caption{Causal Graph for Running Example with query \(\mcal{Q} = \mathbb{P}(Y(X=1)=1|Z=1)\)}
\label{fig:causalgraph}
\end{figure}

We assume that the conditional probability distribution
$p_{v_B.v_A} = \mathbb{P}(V_B=v_B|V_A=v_A)$ is known. %  (or, can be estimated
% accurately from data). 
In Section~\ref{sec:finitedata} % we show how to
we extend our results to the finite data setting where the estimate
$\bar{p}_{v_B.v_A} \neq p_{v_B.v_A}$.  Our goal is to compute
bounds 
for % the 
counterfactual queries of the form  $\mcal{Q} = \mathbb{P}\big(V_{O}(V_{I} = q_I) = q_O|V_{A} = q_{A}\big)$, i.e. the probability of the outcome
$V_{O} = q_O$ after an intervention $do(V_{I} =
q_I)$ given contextual information for $V_A$.
% \gi{observation}
% \(v_{A \cap C(\mcal{Q})} = q_{A \cap C(\mcal{Q})}\).
This counterfactual is specific to an individual realization of \(V_A\),
and computes the outcome of \(V_O\) if \(V_I\) was set to \(q_I\). In
order to simplify our results we impose a restriction on % the set 
$I$ that 
does not impact the query.
% Let $G^{do(V_I = q_I)}$ denote the mutilated graph after intervention
% \(do(V_I = q_I)\), i.e. variables $V_{I}$ no longer have any incoming
% arcs. We want to limit ourselves to queries for which the value of \(V_O\)
% is completely determined when the values of \(V_I\), % \(V_C\)
% $V_A$
% and \(r \in
% R\) are fixed in $G^{do(V_I = q_I)}$. 
% \gi{This would only hold if the values
% of all parents of \(V_O\) are also completely determined when the values
% of \(V_I\), \(V_A\) and \(r \in R\) are fixed, and this would be the case
% if  the values of all
% their parents are completely determined, and so on. % Eventually,
% This
% recursion would terminate in a set $V_{crit}(\mcal{Q})$ such that each variable $V
% \in V_{crit}(\mcal{Q})$ is either % a variable
% in \(V_I\), since these
% variables have no parents,  or % a variable
% in \(V_A\). Thus, $V_{crit}(\mcal{Q}) \subseteq V_A \cup V_{I}$.
\begin{definition}[Critical variables for a query $\mcal{Q}$]
  Let $G^{do(V_I = q_I)}$ denote the mutilated graph after intervention
  \(do(V_I = q_I)\), i.e. variables $V_{I}$ no longer have any incoming
  arcs. % and let
 Then the  critical variables
  \(V_{C(\mcal{Q})}\) for the query $\mcal{Q}$ denote the set of variables in $V_A 
  \cup V_B$ that have a path to some variable in $V_O$ in  $G^{do(V_I = q_I)}$.
  %$V_{crit}(\mcal{Q})$ denote the set of variables in $V_A
  %\cup V_B$ that have a path to some variable in $V_O$ in  $G^{do(V_I =
   % q_I)}$. 
\end{definition}
% Hence, we require every variable \(V\) in \(V_A\) which has a
% directed path to \(V_O\) in \(G^{do(V_I = q_I)}\) to also be in \(V_C\) to
% ensure that \(V_O\) is completely determined when the values of \(V_I\),
% \(V_C\) and \(r \in R\) are fixed. A variable \(V\) is critical in the
% intervention if there is a 
%   directed path from $V$ to a variable in $V_{O}$ in \ms{$G^{do(V_I =
%       q_I)}$}.  
% This discussion is formalized in Assumption~\ref{ass:validQuery}.
\begin{assumption}[Valid Query]
  \label{ass:validQuery}
  The query $\mcal{Q} = \mathbb{P}\big(V_{O}(V_{I} = q_I) =
  q_O|V_A = q_A\big)$ satisfies the following conditions:
  % Let \ms{$G^{do(V_I = q_I)}$} denote the mutilated graph after intervention
  % \(do(V_I = q_I)\), i.e. variables $V_{I}$ no longer have any incoming
  % arcs. 
  \begin{enumerate}[(i)]
  \item $O$, $I \subseteq B$, with % $C \subseteq A$ and 
    \(O \cap I = 
    \emptyset\).
    % \item 
    % very variable in \(V_A\) which is critical to the intervention
    % provides context in the query. That is, for $i \in A$, either $V_i
    % \in V_C$, or $V_i$ is not critical.
  \item All variables $V_I \subseteq V_{C(\mcal{Q})}$, i.e. all $V_I$
    variables are critical for the query.  \label{ass:critical-I}
  \end{enumerate}
\end{assumption}
Lemma~\ref{lem:critical} in \Cref{app:critical} establishes that \ref{ass:critical-I}
is without loss of generality since any $V_i \in V_I$ that is not
critical can be removed from the query. Note that the class of problems
defined in \cite{sachs2021general} allow the variables in \(V_A\) to have
parents, but also require every variable in 
\(V_A\) % which
that
has a directed
path to some variable in \(V_B\) in \(G^{do(V_I = q_I)}\) to be intervened
upon, and so these variables therefore cannot have parents in \(G^{do(V_I = q_I)}\).
% That is, a variable in \(V_A\) which influences the value of \(V_B\)
% in \(G^{do(V_I = q_I)}\) is intervened upon and cannot have parents in
% \(G^{do(V_I = q_I)}\). 
% Hence, we have replaced this
Thus, the 
assumption in~\cite{sachs2021general}  % with the
is effectively equivalent to our 
assumption that % no
variables in \(V_A\) do not have parents. Furthermore, we assume that every
variable in \(V_A\) is the  parent of some variable in \(V_B\), and provides
context in the query. Our assumption is more interpretable and broadly
applicable, yet maintains the expressibility of the
\cite{sachs2021general} formulation.

% The
% $do$-operation ``mutilates'' the underlying graph by removing all 
% incoming arcs to the variables $V_{I}$ that are intervened upon. 
% We assume
% that the query $\mcal{Q}$ satisfies Assumption~\ref{ass:validQuery}. Next, 
% we motivate the need for this assumption. 

% Hence, for the rest of this section, we will explain the intuition
% behind our results via a running example where the causal graph is given
% by Figure \ref{fig:causalgraph}, the conditional probability
% distribution \(p_{xy.z}\) is known, and our goal is to compute bounds
% for the  query \(\mcal{Q} =\mathbb{P}(Y = 
% 1|do(X = 1))\). 

% We now describe how to compute a lower bound for the causal query. 
% Recall
% $V_{m+1},\ldots,V_{n}$ denotes the elements of $V_B$ in
% topological order with respect to the causal graph.

%The presence of \(U_B\) implies that the values  for variables in \(V_B\)
%are not completely defined when \(V_A = 
%v_A\). 
For graphs satisfying Assumption~\ref{ass:validGraph}, the unobserved
confounder $U_B$ can 
potentially be very  high dimensional % unobserved confounder 
with an 
unknown structure. % and dimension. 
We circumvent the difficulty of modeling $U_B$ directly, by instead
modeling the \emph{impact} of the confounder on the relationships between
the observed variables. For \(j \in B\), let $pa(V_j)$ denote the parents
of $V_j$ in the causal 
graph. Then  the variable \(V_j = 
f_j(pa(V_j), U_B)\) for some unknown but fixed function \(f_j:
\{0,1\}^{\abs{pa(V_j)}} \times \mathcal{U}_B \mapsto \{0,1\}\), where
$\mcal{U}_B$ denotes the domain of $U_B$. % Since
Therefore, the confounder \(U_B\) 
impacts the relationship between \(pa(V_j)\) and \(V_j\) by selecting a
function \(f_j \in \mcal{F}_j = \{f: f \text{ is a function from }
pa(V_j) \mapsto V_j\}\).
% For each fixed value \(u_B\) for the
% unknown 
% confounder \(U_B\), the variable \(V_j\) is a function of
% \(pa(V_j)\). Thus, for each \(j \in B\), the confounder \(U_B\)
% effectively selects one function from the set \(\mcal{F}_j = \{m: m \text{
%   is a function from } pa(V_j) \mapsto V_j\}\).
%   For example, consider the
In the
causal graph in Figure~\ref{fig:causalgraph},  for each fixed value for the
unknown confounder \(U\), the variable \(X\) is a function of \(Z\); thus,
\(U\) effectively selects one function from the set $\mcal{F} =
\{f \text{ is a function from } Z \mapsto X\}$. Similarly, \(U\) selects one
function from the set \(\mcal{G} = \{g: g \text{ is a function from }  X
\mapsto Y\}\).

Since each variable \(V_k \in pa(V_j)\) takes values in
\(\{0,1\}\) and \(V_j \in \{0,1\}\), the cardinality of the set
\(\abs{\mcal{F}_j} = 2^{2^{\abs{pa(V_j)}}} \). Therefore, the elements of
\(\mcal{F}_j\) can be indexed by the set $% r_{V_j} \in 
R_{V_j} = \{1, \ldots,
\abs{\mcal{F}_j}\}$. Let  \(f_j(\cdot, r_{V_j}): pa(V_j) \mapsto V_j\)
denote the \(r_{V_j}\)-th 
function % from
in
\(\mcal{F}_j\). 
Let the set \(R = \prod_{j \in B} R_{V_j}\) index all possible mappings
from \(pa(V_j) \mapsto V_j\) for all \(j \in B\). Thus, the response
function variable \(r = (r_{V_{\abs{A}+1}}, \ldots r_{V_n}) \in R\) % can be
% used to
completely 
models the impact of \(U_B\), i.e.  the values of variables \(V_B\)  is a
deterministic function of \(V_A\) and \(r\). % uniquely
%We denote this function by
%\begin{equation}
%  \label{eq:outcomeFunc}
%  F_B: \{0,1\}^{\abs{A}} \times R \mapsto
%  \{0,1\}^{\abs{B}}.
%\end{equation}
% determine the values of variables \(V_B\). 
% Thus, the unobserved
% confounder \(U_B\) effectively selects a particular function from all possible
% functions mapping \(V_A\) to \(V_B\) subject to 
% the graph structure via the response function variable \(r \in R\). 
For example, in the causal
graph in Figure~\ref{fig:causalgraph}, % \(Z\) and For % example, in
% the causal graph in Figure
% \ref{fig:causalgraph},  
it is easy to see that \(\abs{\mcal{F}}=\abs{\mcal{G}}=4\), and
the set 
% therefore, the elements can be indexed by 
\(R = \{(r_X, r_Y)= \{1, \ldots,
4\}^2\}\), where \(f_{r_X}\) denotes the \(r_X\)-th function from
\(\mcal{F}\) and \(g_{r_Y}\) denotes the \(r_Y\)-th function from
\(\mcal{G}\).  Note
that the cardinality \(\abs{R} = \prod_{j \in B} 2^{2^{|pa(V_j)|}}\) is 
exponential in the number of arcs in the causal graph.
% \(r = (r_X, r_Y)\)
% uniquely determine \(X = f_{r_X}(Z)\) and  
% \(Y = g_{r_Y}(X) = g_{r_Y}(f_{r_X}(Z))\). 

Note that although we work with causal graphs with binary variables in
this paper, generalizing all subsequent  
results to categorical variables is straightforward. The critical property
that we exploit is that the set of response function variables index
possible mappings between variables. Therefore, our approach can be
extended to general categorical variables by suitably defining response
function variables. For example, suppose in Figure~\ref{fig:causalgraph},
\(X,Y \in \{0,\ldots,m\}\). Then the response function variable \(r_Y \in
\{0,\ldots,m^m-1\}\). All our results generalize to this more general
setting.

The unknown distribution over the high dimensional \(U_B\) can be
equivalently modeled via % the joint 
the distribution
% $q_{ij} = \mathbb{P}(r_X
% =  
% i, r_Y = j)$.
\(\bm{q} \in \reals^{\abs{R}}_{+}\) over the set $R$. % = \mbb{P}(r)\).
For \(r \in
R\), let \(F_T(V_S =
v_S, r)\) % \ms{can we use $F_O(V_S=v_S, r)$ instead? $G$ is already associated with graph, not function.}
denote the value of \(V_T \subseteq V_B\) when \(V_S = v_S\) provided it is well
defined.
As discussed,  setting \(V_A = v_A\) and
choosing \(r \in R\)
completely defines the values for \(V_B\), i.e. \(F_B(V_A = v_A, r)\) is
well defined. 
Let  
\begin{equation}
  R_{v_B.v_A} = \{r: F_B(V_A = v_A,r) = v_B\}.
\end{equation}
 Hence, \(p_{v_B.v_A} =\sum_{r \in R_{v_B.v_A}} q_{r}\). For example, in
 the causal graph in Figure \ref{fig:causalgraph}, \(R_{xy.z} = \{(r_X,
 r_Y): f_{r_X}(z) = x, g_{r_Y}(x) =  y\}\) denotes the set of \(r\)-values
 that map \(z \mapsto (x,y)\). Hence, \(\mbb{P}(X=x, Y=y|Z=z)
 =\sum_{(r_X,r_Y) \in R_{xy.z}} q_{r_Xr_Y}\).  
 
 % Furthermore, let
The set 
 \begin{equation}
   \label{eq:rQdef}
   R_{\mathcal{Q}} = \big\{r \in R: F_O\big((V_{A},V_{I}) = (q_{A},q_{I}),r\big) = q_O\big\}
 \end{equation}
 % \(R_{\mcal{Q}}\) 
 denotes the set
 of \(r\) values consistent with the query \(\mcal{Q} =
 \mathbb{P}(V_O(V_I = q_I)=q_O|V_A = q_A)\). % i.e.
 Hence, \(\mathbb{P}(V_O(V_I = q_I)=q_O|V_A = q_A) =\sum_{r \in R_{\mcal{Q}}}
 q_{r}\). For the query \(\mcal{Q} = \mathbb{P}(Y(X=1)=1|Z=1)\) in the 
 % For example, for
 causal graph in Figure \ref{fig:causalgraph}, the set
 \(R_{\mathcal{Q}} =
 \left\{(r_X, r_Y): g_{r_Y}(1) = 1\right\}\). % Hence,
 % \(\mathbb{P}(Y(X=1)=1|Z=1) =\sum_{(r_X, r_Y) \in R_{\mcal{Q}}} q_{r_Xr_Y}\). 

Then lower and upper bounds for the causal query can be obtained by solving the following pair of linear programs (\cite{balke94,sachs2021general}): 
\begin{equation}
  \begin{array}{rl}
    \alpha_L / \alpha_U = \min_q / \max_q \ & \sum_{r \in R_{\mcal{Q}}} q_r\\
    \text{s.t.}\ & \sum_{r \in R_{v_B.v_A}} q_{r} = 
                   p_{v_B.v_A},\ \forall  v_A, v_B,\\ 
                                            & q\geq 0.
  \end{array}
  \label{eq:generalprimal}
\end{equation}
Recall for \(V_A = v_A\), \(r \in R\) uniquely determines the value of
\(V_B\). Hence, for fixed \(v_A\), \(\cup_{v_B} R_{v_B.v_A}\) is a partition of
\(R\). Thus, the constraint \(\sum_{r \in R} q_r = \sum_{v_B}\sum_{r \in
  R_{v_B.v_A}}q_r = \sum_{v_B} p_{v_B.v_A} = 1\) is
implied by the other constraints in the LP, and therefore, is not
explicitly added to the LP.
% or in the case of the running example,

% \begin{equation}
%   \begin{array}{rll}
% \alpha_L / \alpha_U = \min_{q} / \max_{q} \ & \sum_{(r_X,r_Y) \in R_{\mcal{Q}}} q_{r_Xr_Y}\\
%     \text{s.t.}\ & \sum_{(r_X,r_Y) \in R_{xy.z}} q_{r_Xr_Y} = 
%                    p_{xy.z}, \forall (x,y,z) % \in
%                    % \{0,1\}^3,
%     \\ 
%     & q\geq 0,
%   \end{array}
%   \label{eq:primal}
% \end{equation}

Note that our bounds are valid even if faithfulness assumptions are
violated~\citep{andersen2013expect}. However, our results only leverage
conditional independence relationships encoded in the graph, and not additional conditional
independence relationships in the data. Hence, if the data displays additional
independence relationships, our method is unable to leverage it to compute tighter
bounds. Note also that we do not impose any additional constraints on
the unobserved confounder.

\section{Pruning the LP}
\label{sec:general,pruningLP}
In this section, we show how to reduce the size of the 
LPs~\eqref{eq:generalprimal} by aggregating 
variables. Let $h: V_A \rightarrow V_B$ % where $h(0) = (f_i(0), g_j(f_i(0)))$ and
% $h(1) = (f_i(1), g_j(f_i(1)))$.
denote any function \(V_A \mapsto V_B\). We also 
% The mapping 
refer to $h$ % s also called
as a \emph{hyperarc} since it can be interpreted as an arc in a
hypergraph. We show how to reformulate LP~\eqref{eq:generalprimal}
into another equivalent LP with variables indexed by hyperarcs, instead of
response function variables. % Note that the maximum possible
The 
number of possible
hyperarcs is $(2^{\abs{B}})^{2^{\abs{A}}} \ll \abs{R}$ (see
Table~\ref{tab:sizetable}). Hence, the new LP has at most
\((2^{\abs{B}})^{2^{\abs{A}}}\) variables, and is thus exponentially
smaller than LP~\eqref{eq:generalprimal}.
Furthermore, we show % that not
we only need to consider a smaller set of hyperarcs that are ``valid'' given 
% every
% hyperarc is relevant for 
the structure of the causal graph. %  i.e. they are
% invalid. 
% Hence the new LP only contains variables that correspond to
% ``valid'' hyperarcs as defined below. 
 
% Furthermore, we show that some hyperarcs 
%\gi{Total number of hyperarcs is $(2^{\abs{B}})^{2^{\abs{A}}} \ll
%\abs{R}$ (see Table~\ref{tab:sizetable}). Note that 
%hyperarcs completely abstract away the causal graph and
%focus just on $V_A$ and $V_B$ variables. Clearly, some hyperarcs are not
%possible given the structure of the causal graph, i.e they are invalid. %  We can, therefore,
% further prune the LP by focusing only
%We show that LP can be reformulated in terms of variables that correspond
%to ``valid'' hyperarcs defined below.} 
% hyperarcs that are consistent with
% the causal graph.}

Recall that \(F_B(V_A =
v_A, r)\) denotes the value of \(V_B\) when \(V_A = v_A\) and \(r \in R\). Let
\begin{equation}
  \label{eq:simpleHyperarcDef}
  R_h = \left\{r \in R: F_B(V_A = v_A, r) = h(v_A),\ \forall v_A \in \{0,1\}^{\abs{A}}\right\}
\end{equation}
denote the set of $r$ values consistent with the
hyperarc $h$, i.e. the set of \(r\)'s that map \(V_A=v_A\)
to % the same output 
% value for
\(V_B = h(v_A)\) % as the hyperarc 
for all inputs $v_A \in
\{0,1\}^{\abs{A}}$.
%output of causal graph is the same as the hyperarc
%for all inputs $v_A \in \{0,1\}^{\abs{A}}$.} 
% In the running example,
For the graph in Figure~\ref{fig:causalgraph},
\begin{equation*}
  R_h = \left\{ (r_X, r_Y) \in R: 
    \begin{array}{l}
      \big(f_{r_X}(0),g_{r_Y}(f_{r_X}(0))\big) = h(0)\\
      \big(f_{r_X}(1),g_{r_Y}(f_{r_X}(1))\big) = h(1)
    \end{array}
  \right\}
\end{equation*}
The causal graph structure implies that \(R_h \neq
\emptyset\) only for a 
subset of hyperarcs. 
% We call such a hyperarc \textit{valid}.
\begin{definition}[Valid Hyperarc]
A hyperarc \(h\) is valid if $R_h \neq \emptyset$.
\end{definition}
In Section \ref{sec:validH}, we discuss why \(R_h \neq \emptyset \) only
for a subset of hyperarcs, and how to efficiently check the validity of a
hyperarc. We denote the set of valid hyperarcs by $H$. 

Next, we show how to write the LP in terms of variables $q_h = \sum_{r \in
  R_h} q_r$ corresponding
to hyperarcs $h \in H$ by aggregating variables $q_{r}$ for $r \in
R_h$. In Lemma~\ref{lemma:partition} in Appendix~\ref{app:critical} we
establish that \(R = \cup_{h \in H} R_h\) is a partition of 
\(R\).  
% Since $R = \cup_{h \in H} R_h$,
Therefore, 
% it follows that 
\[
  \sum_{r \in R_{v_B.v_A}} q_{r}  =  \sum_{h \in H} \sum_{r \in R_h \cap
    R_{v_B.v_A}} q_r = \sum_{\{h \in H: h(v_A) = v_B\}} \Big[ \sum_{r \in
    R_h} q_r \Big] = \sum_{\{h \in H: h(v_A) = v_B\}}  q_h.
\]
Thus, the constraints in \eqref{eq:generalprimal} can all be formulated in
terms of the variables $q_h$ corresponding to hyperarcs.

Next, consider
the objective for the minimization LP:
\[
  \min\ \sum_{r \in R}  \ones\{r \in R_{\mcal{Q}}\} q_r
  =  \min\ \sum_{h \in H} \sum_{r \in R_h}
    \ones\{r \in R_{\mcal{Q}}\} q_r
\]
% Suppose the equality constraints in
% \eqref{eq:generalprimal}, written in a column form, are given by
% \begin{align*}
%   \sum_{r \in R} A_r q_r &= p,
% \end{align*}
% for $A_r \in \{0,1\}^{2^{\abs{N}}}$, $r \in R$, and $p \in
% \reals^{2^{\abs{N}}}_{+}$. From the definition of $R_h$, it
% follows that all $q_r$  such that \(r \in R_h\) % contributes to
% % every constraint which satisfies
% have a coefficient $1$ in the same set of constraints, namely those indexed
% by $\big\{\big(v_B = h(v_A), 
% v_A\big): v_A \in \{0,1\}^{\abs{A}}\big\}$. Thus, the constraints can be
% written as 
% \begin{align*}
%   \sum_{h \in H} A_h \big(\sum_{r \in R_h} q_r\big) &= p,
% \end{align*}
% where $A_h$ denotes the column corresponding to any $r \in R_h$ since they
% are all identical.  %
% Thus,  the equality constraints can all be defined in terms of the
% aggregated variable $q_h = \sum_{r \in R_h} q_r$. 
% Next, we consider the
% objective. First, consider the lower bound LP. Define
% \begin{equation}
%   \label{eq:cLdef}
%   c^L_h := \min_{s \in R_h} \ones\{s \in
% R_{\mcal{Q}}\} = \ones\{R_h \subseteq R_{\mcal{Q}}\}.
% \end{equation}
% From the definition of $R_h$, it
% follows that all $q_r$  such that \(r \in R_h\) % contributes to
% % every constraint which satisfies
% have a coefficient $1$  in the same set of constraints, namely those indexed
% by $\big\{\big(v_B = h(v_A), 
% v_A\big): v_A \in \{0,1\}^{\abs{A}}\big\}$. That is,
From the definition of $R_h$, it
follows that all $q_r$, \(r \in R_h\), % contributes to
% every constraint which satisfies
have a coefficient $1$  in the same set of constraints, namely those indexed
by $\{(v_B = h(v_A), 
v_A): v_A \in \{0,1\}^{\abs{A}}\big\}$. % That is, the variables $q_r$, for $r \in
% R_h$, are all identical in terms of their impact on the
% constraints. 
Hence, for any fixed value % for the variable 
$q_h$ for the variable corresponding to
hyperarc \(h\), any allocation in the set $\{[q_r]_{r \in R_h}: q_h = \sum_{r \in
  R_h} q_r, q_r \geq 0, r\in R_h\}$ % such that  
% % the value of 
% $q_h = \sum_{r \in R_h} q_r$ 
% among would 
% remain
is
feasible. Hence,
any % any
optimal % solution
allocation
satisfies
% \(q_r = q_h\) for \(r \in R_h\) with objective coefficient
% \(\ones\{r \in R_{\mcal{Q}}\} = \min_{s \in R_h} \ones\{s \in
% R_{\mcal{Q}}\}\), while for every other \(r \in R_h\), we have \(q_r = 0\). That is,
\[
  \min\Big\{\sum_{r \in R_h} 
    \ones\{r \in R_{\mcal{Q}}\} q_r: \sum_{r \in R_h} q_r = q_h, q_r \geq 0, \forall r
  \in R_h\Big\}
  = \Big(\min_{r \in R_h} \ones\{r \in  R_{\mcal{Q}}\}\Big) q_h 
\]
% \[
%   \min\ \sum_{r \in R_{\mcal{Q}}} q_r
%   =  \min\ \sum_{h \in H} \sum_{r \in R_h \cap
%     R_{\mcal{Q}}} q_r 
%   =  \min\ \sum_{h \in H} \ones\{R_h \subseteq
%     R_{\mcal{Q}}\}  q_h,
% \]
% where the second equality follows from the following argument.
Hence the lower bound LP % can be reformulated as
can be reformulated as 
\begin{equation}
  \begin{array}{rll}
    \min_{q} \ & \sum_{h \in H} c^L_h q_h\\
    \text{s.t.}\ & \sum_{h \in H: h(v_A) = v_B} q_{h} = 
                   p_{v_B.v_A},  &\forall v_A, v_B, \\% \in
                   % \{0,1\}^3,\\ 
    & q\geq 0,
  \end{array}
  \label{eq:prunedprimal}
\end{equation}
where
\begin{equation}
  \label{eq:cLdef}
  c^L_h := \min_{r \in R_h} \ones\{r \in  R_{\mcal{Q}}\} = \ones\{R_h \subseteq R_{\mcal{Q}}\}.
\end{equation}
% where $c_h^L = \ones\{R_h \subseteq R_{\mcal{Q}}\}$. In the case of the example, we have,
% \begin{equation}
%   \begin{array}{rll}
%     \min_{q} \ & \sum_{h \in H} c_h^Lq_h\\
%     \text{s.t.}\ & \sum_{h \in H: h(z) = (x,y)} q_{h} = 
%                    p_{xy.z},  \forall (x,y,z), \\% \in
%                    % \{0,1\}^3,\\ 
%     & q\geq 0,
%   \end{array}
%   \label{eq:prunedprimal}
% \end{equation}
The objective for the upper bound is given by
\[
  \max\ \sum_{r \in R}  \ones\{r \in R_{\mcal{Q}}\} q_r
  =  \max\ \sum_{h \in H} \sum_{r \in R_h}
  \ones\{r \in R_{\mcal{Q}}\} q_r = \max\ \sum_{h \in H} \Big(\max_{r \in
    R_h} \ones\{r \in R_{\mcal{Q}}\}\Big) q_h,
\]
where the second equality follows from an argument similar to the one used
to establish the objective for the lower bound LP.
% where the second equality follows from the fact that for a fixed $q_h$ it
% is optimal to assign the weight to some $r \in R_h \cap R_{\mcal{Q}}$, and
% if $R_h \cap R_{\mcal{Q}} = \emptyset$, the hyperarc does not
% contribute to the objective. 
Thus, 
 the % upper
% bound 
upper bound LP % for the upper bound \(\alpha_U\)
% can be reformulated as:
is given by 
\begin{equation}
  \label{eq:prunedPrimalDualUpper}
  \begin{array}{rll}
  \max_q \ & \sum_{h \in H} c_h^U  q_h\\
  \mbox{s.t.} \ & \sum_{h \in H: h(v_A) = v_B} q_{h} = p_{v_B.v_A},  &\forall v_A, v_B\\ 
           & q\geq 0,
  \end{array}
\end{equation}
where
\begin{equation}
\label{eq:cUdef}
  c^U_h =  \max_{r \in R_h} \ones\{r \in R_{\mcal{Q}}\} = 
  \ones\{R_h
  \cap R_{\mcal{Q}} \neq~\emptyset\}.
\end{equation}
% Similarly, the upper bound LP can be reformulated as
% \begin{align}
%   \label{eq:prunedPrimalDualUpper}
%   \begin{split}
%   \alpha_U &=
%   \begin{array}[t]{rll}
%   \max_q \ & \sum_{h \in H} c_h^Uq_h\\
%   \mbox{s.t.} \ & \sum_{h \in H: h(v_A) = v_B} q_{h} = 
%                    p_{v_B.v_A}, & \forall (v_A,v_B) \in
%                    \{0,1\}^{2^{|A|}} \times \{0,1\}^{2^{|B|}}  ,\\  
%     & q\geq 0,
%   \end{array}
% \end{split}
% \end{align}
% where \(c_h^U = \max\{c_{r}:
% r \in R_h\}\). 
Both reformulations
have \emph{exponentially} fewer variables since $\abs{H} \ll \abs{R}$;
however, they are useful only 
if the set of valid hyperarcs \(H\) and the corresponding coefficients \(
\ones\{R_h \subseteq R_{\mcal{Q}}\}\)
and \(\ones\{R_h \cap R_{\mcal{Q}} \neq \emptyset\}\) can be
efficiently computed, i.e. in particular, without formulating the original LPs or
iterating over~\(R\). % Hence, we propose the following steps to obtain
% the smaller LPs \eqref{eq:prunedprimal} and
% \eqref{eq:prunedPrimalDualUpper}:
In Section~\ref{sec:validH} % discusses how to do so.
we describe how to efficiently check the validity of a hyperarc and
efficiently compute~$H$, and in Section~\ref{sec:computing chL}
(resp. Section~\ref{sec:computingchU}) we show how to efficiently compute
$c^L_h$ (resp. $c^U_h$). Finally, in Section~\ref{sec:finalprocedure}, we
describe our procedure that uses results established in 
Sections~\ref{sec:validH}, \ref{sec:computing chL} and
\ref{sec:computingchU}
to efficiently construct pruned LPs \eqref{eq:prunedprimal} and
\eqref{eq:prunedPrimalDualUpper} without formulating the original LPs or
iterating over \(R\). 
% \begin{enumerate} 
%     \item The total number of hyperarcs is the number of binary functions mapping \(V_A\) to \(V_B\). First, we efficiently compute the set of valid hyperarcs \(H\) by determining the validity of each of these hyperarcs. Section~\ref{sec:validH} discusses how to do so.
%     \item Then, we need to efficiently compute the objective cofficients \(c_h^L, c_h^U\) for each \(h \in H\). Sections~\ref{sec:computing chL} and \ref{sec:computingchU} discuss how to do so.
% \end{enumerate}

\subsection{Characterizing Valid Hyperarcs}
\label{sec:validH}
We now discuss why \(R_h \neq \emptyset\) only for a subset of hyperarcs.
We also show how to efficiently check the validity of hyperarc $h$ without
enumerating all values in the set \(R\) to check if there exists \(r \in
R\) such that \(F_B(V_A =v_A, r) = h(v_A)\) for all \(v_A \in
\{0,1\}^{|A|}\) i.e. $R_h \neq \emptyset$. Instead, we show that the
outputs of the function \(h\) alone are 
sufficient to determine its validity. We motivate the main result of this
section % Theorem~\ref{validityh}
by first considering the simple causal 
graph in Figure~\ref{fig:causalgraph}. 
Consider a hyperarc \(h\) % in the example such that
with
$h(0) = (x_0, y_0)$ and $h(1) = (x_1, y_1)$ for \(x_i, y_i \in
\{0,1\}\), \(i = 0, 1\). Note that for any choice of $x_i$ and $y_i$, $i
= 0, 1$, $h$ is a hyperarc from $Z \mapsto (X,Y)$. % Since \((r_X, r_Y)\)
% index all possible functions in
% \(\mcal{F}\) and \(\mcal{G}\),
For \(h\) to be a valid hyperarc, it should be of the form $h(z) = (f_h(z), g_h(f_h(z)))$ 
% well-defined functions $f_h: Z \mapsto X$ and $g_h: X \mapsto Y$.
for some $f_h \in \mcal{F}$ and $g_h \in \mcal{G}$. 
The
functions (if they exist) must satisfy:
% Then the ``maps'' (may not be functions) $f_h$ and $g_h$ implied by $h$
% are as follows:  
\[
  f_h(z) = \begin{cases} 
    x_0 & \text{if } z=0 \\
    x_1 & \text{if } z=1
  \end{cases}
  \qquad 
  g_h(x) = \begin{cases} 
    y_0 & \text{if } x=x_0 \\
    y_1 & \text{if } x=x_1
  \end{cases}
\]
Clearly, $f_h$ is well defined % function 
for any choice of $x_i$, $i =
0, 1$. However, there exists % a well defined function which satisfies the
% given condition for
\(g_h\) satisfying the conditions above
% Since $R$ indexes the set of all possible functions in $\mathcal{F}$ and
% $\mathcal{G}$, \(h\) is a valid hyperarc if, and only if, the ``maps''
% \(f_h\)  (resp. \(g_h\)) is consistent with some function \(f \in 
% \mathcal{F}\) (resp. \(g \in \mathcal{G}\)).  
% The latter is true
if, and
only if, \(y_0 =
y_1\)  % it is the case that
whenever
\(x_0 = x_1\). % implies that
Hence, % to check the validity of
$h \not \in H$ % , it is sufficient to check
if, and only if, 
$x_0=x_1$ but $y_0 \neq y_1$. For example, consider the hyperarc \(h_1\)
with \(h_1(0) = (0,0)\) and \(h_1(1) = (0,0)\). To check if \(h_1\) is
valid, we need to check if there exists well defined functions \(f_{h_1}:Z
\mapsto X\) and \(g_{h_1}:X \mapsto Y\) which satisfy: 
\[
  f_{h_1}(z) = \begin{cases} 
    0 & \text{if } z=0 \\
    0 & \text{if } z=1
  \end{cases}
  \qquad 
  g_{h_1}(x) = 
    0, \quad \text{if } x=0 \\
\]
% Note that
Clearly, 
\(f_{h_1}\) is a well defined function. %  On the other hand, consider the
% function
Any function
\(g: X \mapsto Y\) in \(\mcal{G}\) % where 
% \[g(x) = \begin{cases} 
%     0 & \text{if } x=0 \\
%     0 & \text{if } x=1 \\
%   \end{cases}\]
with $g(0) = 0$ satisfies the conditions for $g_{h_1}$. Hence, $h_1 \in
H$.
 
% Note that \(g\) satisfies the given condition for \(g_{h_1}\). Hence,
% \(h_1\) is valid.
On the other hand, consider
the hyperarc \(h_2\) with \(h_2(0) = (0,0)\) and \(h_2(1) = (0,1)\). The hyperarc \(h_2\) is valid % we need to check 
if there exists a
well defined function \(g_{h_2}:X \mapsto Y\) which satisfies: 
\[
  g_{h_2}(x) = \begin{cases} 
    0 & \text{if } x=0 \\
    1 & \text{if } x=0 \\
  \end{cases}
\]
Clearly, there cannot be such a function, and so \(h_2 \not \in H\).

% Hence, we do not need to
Note that we did not have to 
iterate over the set \(R\) % and check if each value lies in \(R_{h}\) 
to check if a hyperarc \(h\) is valid. Instead, we recognize that a
hyperarc \(h\) only partially specifies a function mapping from $pa(V_j)$
to  
$V_j$. Therefore, in order to check the validity of \(h\), we only need to
check if there
exists some binary function $pa(V_j)  
\rightarrow V_j$ which satisfies this partial specification. Hence, the
outputs of the function \(h\) alone are sufficient to determine its
validity. The following theorem generalizes this observation to general
causal graphs.

\begin{restatable}[Valid hyperarcs]{theorem}{validityh}
\label{validityh}
Let $a_{A} \in \{0,1\}^{\abs{A}}$ denote the values set for the variables
$V_A$, and let 
% let the notation 
$a_B = h(a_A) \in \{0,1\}^{\abs{B}}$ denote the
values for the variables $V_B$ when the hyperarc $h$ is evaluated at
$a_A$. A hyperarc $h$ is valid % i.e. $R_h \neq \emptyset$, 
if, and only if, for
all
$\big(a_A, a_B = h(a_A)\big), \big(b_A, b_B = h(b_A)\big) \in
\{0,1\}^{|N|}$ % such that \(h(a_A) = a_B\) and \(h(b_A) = 
% b_B\),
and 
% for all
\(j \in B\), 
\begin{eqnarray}
  \label{eq:functionalconsistency}
  a_{P_j} = b_{P_j} \implies 
  a_j = b_j,
\end{eqnarray}
where \(P_j
\subseteq N\) denote the indices of \(pa(V_j)\). % , $j \in B$.
\end{restatable}
\begin{proof}
It is clear that if a hyperarc is valid, then it is satisfies
\eqref{eq:functionalconsistency}.

Suppose a hyperarc satisfies
\eqref{eq:functionalconsistency}. Then we are given a partial
specification for a \emph{function} from $pa(V_j) \rightarrow V_j$,
i.e. the same input values are always mapped to the same output value;
however, the output is only specified for possibly a subset of input
values. 
% Since, 
For each $V_j \in V_B$, $r_{V_j}$
indexes the set of all possible functions $pa(V_j) \rightarrow V_j$;
therefore, for every node $j$
% Then, for each $V_j \in V_B$, 
there exists a binary function 
$pa(V_j) \rightarrow V_j$ which satisfies the partial specification % of a mapping from $pa(V_j)$ to
% $V_j$ implied
given 
by the hyperarc $h$. % Since, for each $V_j \in V_B$, $r_{V_j}$
% indexes the set of all possible functions $pa(V_j) \rightarrow V_j$, it
% follows that the set of $r$-values which are consistent with the mapping
% $h$ i.e. $R_h$ is non-empty. Equivalently
Thus, $R_h \neq \emptyset$, or equivalently, 
$h$ is valid.
\end{proof}
Theorem~\ref{validityh} implies that we only have to search through the set of 
possible hyperarcs from $V_A \mapsto V_B$ of cardinality $2^{|B|2^{|A|}}$ for hyperarcs which satisfy \eqref{eq:functionalconsistency} to identify the set $H$ of valid hyperarcs. 
The first two columns of Table~\ref{tab:sizetable} compare \(|R|\)
with the maximum possible 
number of hyperarcs  
\(2^{|B|2^{|A|}}\) for seven different causal inference problems (details
in Appendix). Note that the reduction in size can be several orders of
magnitude, and it increases with the complexity of the causal graph, see
e.g. Examples D and E. Thus, there is a very significant reduction in size even
if all hyperarcs are valid. The last
column in Table~\ref{tab:sizetable} lists % the cardinality of 
\(|H|\). Considering only the valid hyperarcs further decreases the
size of the LP by at least \(1\) order of magnitude, and sometimes
more. The LPs corresponding to Examples~B and~C can be solved without
pruning; however, the LPs corresponding to Examples~A, F and G can only be solved
after pruning the problem, and the LPs for Examples~D and E are too large
even after pruning. In Section~\ref{sec:greedy} we propose a greedy
heuristic to compute bounds for these problems. 
Next, we show how to
efficiently compute % \(c^L\) and \(c^U\) without 
% iterating over \(R\). Recall that  
\(c^L_h = \ones\{R_h \subseteq R_{\mcal{Q}}\}\) and \(c_h^U = \ones\{R_h
\cap R_{\mcal{Q}} \neq \emptyset\}\).

\begin{table}[t]
\begin{center}
\begin{tabular}{||r|r|r|r||} 
  \hline
  \multicolumn{1}{|c|}{Graph}
  &\multicolumn{1}{|c|}{\(|R|\)} &
                                   \multicolumn{1}{|c|}{${2^{|B|}}^{2^{|A|}}$} &
                                                                                 \multicolumn{1}{|c|}{\(\abs{H}\)}\\ [0.5ex] 
 \hline\hline
 Ex A & $1.3 \times 10^{8}$
  & $1.0 \times 10^{6}$ & $2.3 \times 10^3$ \\
 \hline
 Ex B & $4.2 \times 10^6$ & $1.0 \times 10^{6}$ & $7.1 \times 10^4$ \\
 \hline
 Ex C & $4.2 \times 10^6$
 & $1.0 \times 10^6$ & $4.4 \times 10^4$ \\
 \hline
 Ex D & $6.3 \times 10^{57}$
 & $1.7 \times 10^7$ & $9.4 \times 10^6$ \\
 \hline
  Ex E & $3.2 \times 10^{32}$
 & $1.7 \times 10^7$ & $9.4 \times 10^6$ \\
 \hline
 Ex F & $1.8 \times 10^{13}$
 & $6.5 \times 10^4$ & $5.8 \times 10^4$ \\
 \hline
 Ex G & $3.0 \times 10^{23}$
 & \(1.0 \times 10^6\) & $9.2 \times 10^3$ \\
 \hline
\end{tabular}
\end{center}
\caption{The naive LP for computing causal bounds has $\abs{R}$
    variables, where~\(\abs{R}\)  denotes the cardinality of the set of all
    possible values for the response  
    function variables. The number of variables % can be reduced to
    drops to 
    $2^{\abs{B}2^{\abs{A}}}$ when the LP is formulated in terms of hyperarcs, and the
    number of variables can be further reduced to 
    \(\abs{H}\), the cardinality of the set of \emph{valid} hyperarcs. Note that $\abs{R} \gg 2^{\abs{B}2^{\abs{A}}}\gg
    \abs{H}$. See Section~\ref{sec:general,pruningLP} for details.}
\label{tab:sizetable}
\end{table}

\subsection{Efficiently computing $c_h^L = \ones\{R_h \subseteq R_{\mcal{Q}}\}$}
\label{sec:computing chL}
% To check if $h$ is completely consistent with
% $R_{\mcal{Q}}$ for general causal inference problems, we begin by defining
% complete consistency for general conditional probabilities
% \begin{definition}[Complete Consistency of Probability]
% The conditional probability \(p_{v_B.v_A}\) is completely consistent with
% the query \(\mcal{Q}\)
% % \(S \subseteq R\) 
% if \(R_{v_B.v_A} \subseteq R_{\mcal{Q}}\). 
% \end{definition}
% Recall \(c_h^L = 1\) if, and only if, \(R_h \subseteq R_{\mcal{Q}}\).  % Next, we describe
% % how to check whether a hyperarc \(h\) is completely consistent.
We show how to check if \(R_h \subseteq R_{\mcal{Q}}\) efficiently
i.e. without iterating over the set \(R_h\) and checking if each value
lies in \(R_{\mcal{Q}}\). Instead, we show that the outputs of \(h\) alone
are sufficient to determine if \(R_h \subseteq R_{\mcal{Q}}\). As 
before, % we introduce
we illustrate 
the main ideas using the graph in
Figure~\ref{fig:causalgraph} with % nd then generalize them. 
% We explain the intuition behind this result by establishing it for the
% running example with 
query \(\mcal{Q} = \mathbb{P}(Y(X=1)=1|Z=1)\), and then prove them. % We
\begin{restatable}[]{theorem}{CChyperarc}
\label{thm:CCh}
% or graphs satisfying
Suppose the causal graph satisfies
Assumption~\ref{ass:validGraph} and the query $\mcal{Q}$ satisfies
Assumption~\ref{ass:validQuery}. Then % he following results hold. 
% \begin{enumerate}[(i)]
% \item $R_{v_B.v_A}  = \{r \in R: F_B(V_A = v_A, r) = v_B\}\subseteq
%   R_{\mcal{Q}}$ if, and only if, 
% $v_{I   \cap A} = q_{C}$, and 
%   $v_{I} = q_{I}$, 
%   $v_{O} = q_O$.
% \item
\(R_h \subseteq R_{\mcal{Q}}\) if, and
only if, there exists 
\(v \in \{0,1\}^{|N|}\) such that \(h(v_A) = v_B\), $v_{A \cap C(\mcal{Q})} = q_{A \cap C(\mcal{Q})}$, $v_{I} = q_{I}$ and $v_{O} = q_O$.
% 
% R_{v_B.v_A}  
% \subseteq R_{\mcal{Q}}.
% 
% \end{enumerate}
% where \(I_C\) denotes the set of critical variables defined in Assumption~\ref{ass:validQuery}.
\end{restatable}

% Again, we explain the intuition behind this result by establishing that
% for the running example, the 
% hyperarc $h$ \(\subseteq R_{\mcal{Q}}\), 
% if and only if, there exists \(z\) such that \(h(z) = (x,y)\) and the
% conditional probability \(p_{xy.z}\) is completely consistent with
% \(\mcal{Q}\). Thus, for
Consider 
% Theorem~\ref{thm:CCh} applied to 
the query \(\mcal{Q} = \mbb{P}(Y(X = 1)=1|Z=1)\) in the causal
graph in Figure~\ref{fig:causalgraph}. Here $V_A = \{Z\}$, $V_{B} =
\{X, Y\}$, $V_I = \{X\}$ and $V_O = \{Y\}$.  For a hyperarc $h:Z \mapsto (X,Y)$, Theorem~\ref{thm:CCh} implies  %that $R_h \subseteq
%R_{\mcal{Q}}$ for a hyperarc $h:Z \mapsto (X,Y)$ if there exists $z \in
%\{0,1\}$ such that $h(z) = (x, y)$ with $x= 1$ and $y = 1$, i.e. 
% implies that
\begin{equation}
  \label{eq:hConsistent}
  R_h \text{ \(\subseteq R_{\mcal{Q}}\)}
  \iff  h(z) = (1,1) \text{ for
  some } z \in \{0,1\}.
\end{equation}
Hence, Theorem~\ref{thm:CCh} implies that we can efficiently compute $c^L_h = \ones\{R_h \subseteq R_{\mcal{Q}}\}$ for a hyperarc \(h\) by considering only the outputs of the hyperarc, instead of $R_h$.
  \begin{proof}
  Suppose there exists $v \in \{0,1\}^{|N|}$ such that $h(v_A)=v_B$, $v_{A \cap C(\mcal{Q})} = q_{A \cap C(\mcal{Q})}$, $v_{I} = q_{I}$ and $v_O = q_O$. Then, every $r \in R_{h}$
  maps $(V_{A \cap C(\mcal{Q})}, V_{I}) = (q_{A \cap C(\mcal{Q})}, q_{I})$ to $V_{O}=q_O$, and thus, 
  $r \in R_{\mcal{Q}}$ i.e. $R_{h} \subseteq R_{\mcal{Q}}$. To establish the opposite direction, suppose \(R_h \subseteq R_{\mcal{Q}}\), but there does not exist $v \in
  \{0,1\}^{|N|}$  % which satisfies Theorem \ref{thm:CCh}.
  such that $h(v_A) = v_B$, $v_{A \cap C(\mcal{Q})} = q_{A \cap C(\mcal{Q})}$, $v_{I} = q_{I}$ and $v_O = q_O$. Equivalently, 
  for all $v
  \in \{0,1\}^{|N|}$ such that $h(v_A) = v_B$ and $v_{A \cap C(\mcal{Q})} = q_{A \cap C(\mcal{Q})}$, we have either $v_{I} \neq q_{I}$ or $v_O
  \neq q_O$. We consider these two cases separately. 
  \begin{enumerate}[(a)]
  \item For all \(v \in \{0,1\}^{|N|}\) such that \(h(v_A) = v_B\) and \(v_{A \cap C(\mcal{Q})} = q_{A \cap C(\mcal{Q})}\), we have $v_{I} \neq q_{I}$: % $r$ maps $V_A = v_A$ to $V_B = v_B$ for all $(v_A, v_B)$ such that
    % $h(v_A) = v_B$ i.e. $r \in R_h$
    Since $R$ indexes all possible functions on the causal graph, % exhaustive, 
    there exists $r \in R$ such that:
    \begin{itemize}
    \item $r$ maps $V_A = v_A$ to $V_B = h(v_A)$ for all $v_A$,  % such that
      % $h(v_A) = v_B$
      i.e. $r \in R_h$.
    \item $r$ maps $V_{A \cap C(\mcal{Q})} = q_{A \cap C(\mcal{Q})}$, and  $V_{I} = q_{I} \neq v_{I}$ to $V_{O} \neq q_O$.
    \end{itemize}
    Hence $R_h \nsubseteq R_{\mcal{Q}}$, a contradiction.
  \item There exists \(v \in \{0,1\}^{|N|}\) such that \(h(v_A) = v_B\), \(v_{A \cap C(\mcal{Q})} = q_{A \cap C(\mcal{Q})}\) and $v_{I} = q_{I}$ but $v_ O \neq q_O$: In this case, every $r \in
    R_{h}$  maps $(V_{A \cap C(\mcal{Q})},V_{I}) = (q_{A \cap C(\mcal{Q})}, q_{I})$ to $V_{O} \neq q_O$, and therefore, $r \not
    \in R_{\mcal{Q}}$.   Hence $R_{h} \nsubseteq R_{\mcal{Q}}$.
  \end{enumerate}
\end{proof}

% Now we show how these results imply that we can efficiently check if \(R_h
% \subseteq R_{\mcal{Q}}\) for hyperarc \(h\). Consider the graph in Figure
% \ref{fig:causalgraph}, query \(\mcal{Q} = \mathbb{P}(Y(X=1)=1|Z=1)\) and
% hyperarc \(h_1:Z \mapsto (X,Y)\) which satisfies \(h_1(0) = (0,0)\) and
% \(h_1(1)=(0,0)\). From Theorem~\ref{thm:CCh}, \(R_{h_1} \nsubseteq
% R_{\mcal{Q}}\) since \(h_1(z) \neq (1,1)\) for all \(z \in \{0,1\}\). On
% the other hand, consider hyperarc \(h_2:Z \mapsto (X,Y)\) which satisfies
% \(h_2(0) = (1,1)\) and \(h_2(1)=(0,0)\). \(R_{h_2} \subseteq
% R_{\mcal{Q}}\) since \(h_2(0) = (1,1)\). Hence, we do not need to iterate
% over the set \(R_h\) and check if each value lies in \(R_{\mcal{Q}}\) to
% check if \(R_h \subseteq R_{\mcal{Q}}\). Instead, the outputs of \(h\)
% alone are sufficient to determine if \(R_h \subseteq R_{\mcal{Q}}\).  

\subsection{Efficiently computing $c_h^U = \ones\{R_h \cap R_{\mcal{Q}}
  \neq 
\emptyset\}$}

\label{sec:computingchU}

We now show how to % compute $c_h^U = \ones\{R_h \cap R_{\mcal{Q}} \neq
% \emptyset\}$ using an  
efficiently check  % algorithm for checking 
if
\(R_h \cap
R_{\mcal{Q}} = \emptyset\). % In particular,
As in the case with $c^L_h$, 
we show that % we can check if
the condition
\(R_h \cap R_{\mcal{Q}} = \emptyset\) can be checked without iterating over the set
\(R_h\). % and checking if each value lies in \(R_{\mcal{Q}}\). 
Instead, the
outputs of \(h\) alone are sufficient to determine if \(R_h \cap
R_{\mcal{Q}} = \emptyset\).

\begin{restatable}[]{theorem}{SChyperarc}
\label{SChyperarc}
Suppose the causal graph satisfies Assumption~\ref{ass:validGraph} and the
query satisfies Assumption~\ref{ass:validQuery}. Then
% \begin{enumerate}[(i)]
% \item $R_{v_B.v_A}  \cap R_{\mcal{Q}} = \emptyset$ if, and only if,
% $v_{I   \cap A} = q_{C}$, and 
%   $v_{I} = q_{I}$, 
%   $v_{O} \neq q_O$.
% \item 
\(R_h \cap R_{\mcal{Q}} = \emptyset\) if, and
only if, there exists 
\(v \in \{0,1\}^{|N|}\) such that \(h(v_A) = v_B\), 
$v_{A \cap C(\mcal{Q})} = q_{A \cap C(\mcal{Q})}$, $v_{I} = q_{I}$ and $v_{O} \neq q_O$.
% $
% R_{v_B.v_A}  
%   \cap R_{\mcal{Q}} = \emptyset.
%   $
% \end{enumerate}
\end{restatable}

Again, consider 
% Theorem~\ref{thm:CCh} applied to 
the query \(\mcal{Q} = \mathbb{P}(Y(X=1)=1|Z=1)\) in the causal
graph in Figure~\ref{fig:causalgraph}. For a hyperarc $h:Z \mapsto (X,Y)$, Theorem~\ref{SChyperarc} implies  %that $R_h \subseteq
%R_{\mcal{Q}}$ for a hyperarc $h:Z \mapsto (X,Y)$ if there exists $z \in
%\{0,1\}$ such that $h(z) = (x, y)$ with $x= 1$ and $y = 1$, i.e. 
% implies that
\begin{equation}
  \label{eq:hInconsistent}
  R_h \cap R_{\mcal{Q}} = \emptyset
  \iff  h(z) = (1,0) \text{ for
  some } z \in \{0,1\}.
\end{equation}
Hence, Theorem~\ref{SChyperarc} implies that we can efficiently compute $c^U_h = \ones\{R_h \cap R_{\mcal{Q}} \neq \emptyset\}$ for a hyperarc \(h\) by considering only the outputs of 
the hyperarc, instead of $R_h$.
\begin{proof}
%Suppose $R_{v_B.v_A}$ satisfies $v_{C} =
%  q_{C}$, $v_{I
%    \cap B} = q_{I}$ and $v_O \neq q_O$. Then, every $r \in R_{v_B.v_A}$
%  maps $V_{I}=q_{I}$ to $V_{O} \neq q_O$, and thus, 
%  $r \not \in R_{\mcal{Q}}$ i.e. $R_{v_B.v_A} \cap R_{\mcal{Q}} = \emptyset$. To
%  establish the other direction, let $v = (v_A, v_B)$ be the given assignment
%  to the variables $V_N$, and  consider the following
%  two cases:
%  \begin{enumerate}[(a)]
%  \item  % either  $v_{C} \neq q_{C_C}$ or $v_{I  \cap
      % B} \neq q_{I}$ i.e. the value of
%    $v_{I} \neq q_I $ % in $R_{v_B.v_A}$.
%    but $v_O \neq q_O$:
    % is not $q_{I_C}$
%    Since $R$ % is exhaustive,
%    indexes all possible functions on the causal graph, 
%    there
%    exists $r$ such that: 
%    \begin{itemize}
%    \item $r$ maps $V_A = v_A$ to $V_B = v_B$, i.e. $r \in R_{v_B.v_A}$. % In
      % particular, 
      % $F_B(v_A, r) = v_B $
      % $r$ maps $V_{I} \neq q_{I}$ to $V_{O} = q_O$. 
 %   \item  $r$ maps $v_{I} = q_{I} \neq v_I$ to $V_B$ such that $V_{O} = q_O$,
      % $(r_X,r_Y)$ maps $Z=1$ to $(X\neq \bar{x}, Y=0)$ 
%      i.e. $r \in R_{\mcal{Q}}$.
%    \end{itemize}
%    Thus, there exists $r$ such that $r \in
%    R_{v_B.v_A}$, but $r \in R_{\mcal{Q}}$ i.e. $R_{v_B.v_A}
%    \cap R_{\mcal{Q}} \neq \emptyset$.  
%  \item $v_O = q_O$: In this case, every $r \in
 %   R_{v_B.v_A}$  maps $v_{I} = q_{I}$ to $V_{O} = q_O$, and therefore, $r 
 %   \in R_{\mcal{Q}}$.  Hence $R_{v_B.v_A} \cap R_{\mcal{Q}} \neq \emptyset$.
%  \end{enumerate}
%  Next, we establish (ii). 
Suppose there exists $v \in \{0,1\}^{|N|}$ such that $h(v_A)=v_B$, $v_{A \cap C(\mcal{Q})} = q_{A \cap C(\mcal{Q})}$, $v_{I} = q_{I}$ and $v_O \neq q_O$. Then, every $r \in R_{h}$
  maps $(V_{A \cap C(\mcal{Q})}, V_{I})=(q_{A \cap C(\mcal{Q})}, q_{I})$ to $V_{O} \neq q_O$, and thus, 
  $r \not \in R_{\mcal{Q}}$ i.e. \(R_h \cap R_{\mcal{Q}} = \emptyset\). To establish the opposite direction, suppose \(R_h \cap R_{\mcal{Q}} = \emptyset\), but there does not exist $v \in
  \{0,1\}^{|N|}$  % which satisfies Theorem \ref{thm:CCh}.
  such that $h(v_A) = v_B$, $v_{A \cap C(\mcal{Q})} = q_{A \cap C(\mcal{Q})}$, $v_{I} = q_{I}$ and $v_{O} \neq q_O$. Equivalently, for all $v
  \in \{0,1\}^{|N|}$ such that $h(v_A) = v_B$ and $v_{A \cap C(\mcal{Q})} = q_{A \cap C(\mcal{Q})}$, we have either $v_{I} \neq q_{I}$ or $v_O
  = q_O$. We consider these two cases separately. 
  \begin{enumerate}[(a)]
  \item For all \(v \in \{0,1\}^{|N|}\) such that \(h(v_A) = v_B\) and \(v_{A \cap C(\mcal{Q})} = q_{A \cap C(\mcal{Q})}\), we have $v_{I} \neq q_{I}$: % $r$ maps $V_A = v_A$ to $V_B = v_B$ for all $(v_A, v_B)$ such that
    % $h(v_A) = v_B$ i.e. $r \in R_h$
    Since $R$ indexes all possible functions on the causal graph, % exhaustive, 
    there exists $r \in R$ such that:
    \begin{itemize}
    \item $r$ maps $V_A = v_A$ to $V_B = h(v_A)$ for all $v_A$,  % such that
      % $h(v_A) = v_B$
      i.e. $r \in R_h$.
    \item $r$ maps $V_{A \cap C(\mcal{Q})} = q_{A \cap C(\mcal{Q})}$, and  $V_{I} = q_{I} \neq v_{I}$ to $V_{O} = q_O$.
    \end{itemize}
    Hence $R_h \cap R_{\mcal{Q}} \neq \emptyset$, a contradiction.
  \item There exists \(v \in \{0,1\}^{|N|}\) such that \(h(v_A) = v_B\), \(v_{A \cap C(\mcal{Q})} = q_{A \cap C(\mcal{Q})}\) and $v_{I} = q_{I}$ but $v_O = q_O$: In this case, every $r \in
    R_{h}$  maps $(V_{A \cap C(\mcal{Q})}, V_{I}) = (q_{A \cap C(\mcal{Q})}, q_{I})$ to $V_{O} = q_O$, and therefore, $r 
    \in R_{\mcal{Q}}$.   Hence $R_{h} \cap R_{\mcal{Q}} \neq \emptyset$.
  \end{enumerate}
\end{proof}
%Now we show how these results imply that we can efficiently check if \(R_h
%\cap R_{\mcal{Q}} = \emptyset\) for hyperarc \(h\). Consider the graph in
%Figure \ref{fig:causalgraph}, query \(\mcal{Q} = \mathbb{P}(Y(X=1)=1|Z=1)\)
%and hyperarc \(h_1:Z \mapsto (X,Y)\) which satisfies \(h_1(0) = (0,0)\)
%and \(h_1(1)=(0,0)\). From Theorem~\ref{SChyperarc}, \(R_{h_1} \cap
%R_{\mcal{Q}} \neq \emptyset\) since \(h_1(z) \neq (1,0)\) for all \(z \in
%\{0,1\}\). On the other hand, consider hyperarc \(h_2:Z \mapsto (X,Y)\)
%which satisfies \(h_2(0) = (1,0)\) and \(h_2(1)=(0,0)\). \(R_{h_2} \cap
%R_{\mcal{Q}} = \emptyset\) since \(h_2(0) = (1,0)\). Hence, we do not need
%to iterate over the set \(R_h\) and check if each value lies in
%\(R_{\mcal{Q}}\) to check if \(R_h \cap R_{\mcal{Q}} =
%\emptyset\). Instead, the outputs of \(h\) alone are sufficient to
%determine if \(R_h \cap R_{\mcal{Q}} = \emptyset\).
\subsection{Computing the Pruned LPs}
\label{sec:finalprocedure}

\begin{algorithm}[t]
\noindent \textbf{Input:} (i) causal graph \(G\),  (ii) query \(\mcal{Q} =
\mathbb{P}(V_O(V_I=q_I)=q_O|V_A =q_A)\),  (iii) conditional probability
distribution $p_{v_B.v_A} = \mathbb{P}(V_B = 
v_B|V_A = v_A)$, for all $v_A, v_B$.\\
\textbf{Output}: Pruned LPs \eqref{eq:prunedprimal} and
\eqref{eq:prunedPrimalDualUpper}\\
$H \gets \emptyset$\\
\For{$h:  V_A \rightarrow V_B$}{
  \If{$h$ is valid (Theorem~\ref{validityh})}{
    $H \gets H \cup \{h\}$\\
    Compute \(c_h^L\) using Theorem \ref{thm:CCh}\\ [0.1em]
    Compute \(c_h^U\) using Theorem \ref{SChyperarc}\\  
  }
}
\Return{LPs \eqref{eq:prunedprimal} and \eqref{eq:prunedPrimalDualUpper}
  constructed using \((H, c^L, c^U)\)} 
\caption{Procedure to efficiently construct LPs \eqref{eq:prunedprimal}, \eqref{eq:prunedPrimalDualUpper}} 
\label{alg:ourprocedure}
\end{algorithm}

Algorithm~\ref{alg:ourprocedure}
describes a
procedure that uses results established in
Sections~\ref{sec:validH}, \ref{sec:computing chL} and
\ref{sec:computingchU}
to efficiently construct pruned LPs \eqref{eq:prunedprimal} and \eqref{eq:prunedPrimalDualUpper} without formulating the original LPs or iterating over \(R\).\\

\section{Bounds in Closed Form}
\label{sec:closedform}
% Recall that 
% for the class of
% problems identified by \cite{sachs20},
% Assumption~\ref{ass:validQuery} allows queries $\mcal{Q}$ where 
% each variable in $V_A$ either provides context for the query, or is not % redundant
% critical
% for the intervention.
Now we show that
% for a subclass of problems which satisfy Assumption \ref{ass:closedForm}
% i.e. where the entire set
if $A \subseteq \mcal{C}(\mcal{Q})$, i.e. all $V_A$ variables are critical for the query, %  is involved in the intervention and all
% =======
% Recall that % for the class of
% % problems identified by \cite{sachs2021general},
% Assumption~\ref{ass:validQuery} allows queries $\mcal{Q}$ where 
% each variable in $V_A$ either provides context for the query, or is not % redundant
% critical
% for the intervention. Now we show that
% % for a subclass of problems which satisfy Assumption \ref{ass:closedForm}
% % i.e. where the entire set
% if 
% $A \subseteq C$, i.e. all $V_A$ variables are critical to the intervention, %  is involved in the intervention and all
% >>>>>>> b5c9d5e3c9b3819937836557655e4292591d5461
% intervention variables are critical, 
the bounds can be computed in closed
form by simply adding appropriate conditional probabilities in the input data. 
% Note that % while
% bounds % have been provided in
% were written in
% closed form via vertex
% enumeration in
This is in contrast to the closed form bounds in 
\cite{balke94} and \cite{sachs2021general} that are computed by enumerating vertices of the constraint polytope. These bounds cannot be computed for large graphs
% computing
% these bounds is not scalable,
since the total number of vertices grows
exponentially in the size of the associated LP, which itself is large.
% linear program, which itself is
% large.
On the other hand, our bounds, which involve simply adding
probabilities in the input data, scale significantly better to larger
problems. % as long as the query satisfies Assumption
% \ref{ass:closedForm}.
  
% where
% \begin{assumption}
%   \label{ass:closedForm}
%   The query \(\mcal{Q} = \mbb{P}(V_O = q_O \mid do(V_I = q_I))\) satisfies the
%   following: 
%   \begin{enumerate}[S1]
%   \item % for every $j \in A$, 
%     $A \subseteq I$, i.e. % every node in
%     the entire set 
%     $A$ is involved in the intervention, \label{assumption:S1} 
%   \item % for every $j \in q$ 
%     there is a directed path from every $V \in V_{I}$ to some
%     variable in $V_{O}$ in 
%     $G^{\mcal{I}}$, i.e. all intervention variables are critical, or
%     equivalently \(I_C = I\). \label{assumption:S2}
%   \end{enumerate}
% \end{assumption}
We motivate % the  rationale for  % behind this result
% this assumption using the 
% is best understood with a new
% example. Consider
the main ideas of the result using the 
causal graph in Figure~\ref{fig:closedformgraph} and query \(\mcal{Q} = \mathbb{P}(Y(X=1)=1|Z=1)\). Note that \(Z\) is critical for this query in this graph due to the edge \(Z \rightarrow Y\).
% Given the input data $p_{xy.z}$, we show % how to compute
% that the lower
% bound for the modified
% query $\mcal{Q}_1 = \mathbb{P}(Y=1|do(Z=1, X=1))$ can be computed in
% closed form. More  
% precisely, the optimal value of the pruned LP for the lower bound % for this
% % problem  
% is the sum of probabilities \(p_{xy.z}\) in the data which satisfy \(R_{xy.z} \subseteq R_{\mcal{Q}_1}\).
% Retracing
% the steps in
The results in
Section~\ref{sec:computing chL} % we get
imply that 
$c_h^L = 1$ % i.e. $h$
% is completely consistent with $\mcal{Q}_1$
if, and only if, \(h(1) =
(1,1)\). Thus, it follows that the objective of the pruned LP % for the
% problem is:
for the lower bound can be written as
\begin{eqnarray}
\sum_{\{h \in H: c_h^L=1\}} q_h &= & \sum_{\{h \in H: h(1) = (1,1)\}} q_h\nonumber \\
  &= & p_{11.1}\label{eq:eg2}
\end{eqnarray}
where \eqref{eq:eg2} follows from the constraints of the pruned
LP~\eqref{eq:prunedprimal}. Hence, the lower bound can be computed in closed form.   
% Next, we motivate the definition of a subclass of problems for which
% closed form bounds can be computed by establishing why these bounds cannot
% be computed for
% n order better understand
% As a step towards characterizing
% graphs and queries for which such closed form
% bounds are possible, we understand why closed form bounds do not exist for
% the original

On the other hand, consider the (original)
causal graph in Figure~\ref{fig:causalgraph}
and the query 
$\mcal{Q}= \mathbb{P}(Y(X=1)=1|Z=1)$. Now \(Z\) is no longer critical for the query in this graph. From~\eqref{eq:hConsistent} we have
that 
\(c_h^L=1\) % i.e. $h$ \(\subseteq R_{\mcal{Q}}\),
if, and only
if,   
 there exists $z \in \{0,1\}$ such that  $h(z) = (1,1)$. The objective of the pruned LP is:
\begin{eqnarray*}
\sum_{\{h \in H: c_h^L=1\}} q_h &=& \sum_{\{h \in H: \exists z \in \{0,1\},h(z) = (1,1)\}} q_h\\
&\neq& \sum_{h \in H: h(0) = (1,1)} q_h + \sum_{h \in H: h(1) = (1,1)} q_h,
\end{eqnarray*}
since \(\{h \in H: h(0) = (1,1)\} \cap \{h \in H: h(1) = (1,1)\} \neq
\emptyset\).
% where (5) fails because there \textit{can} exist a hyperarc $h \in H$
% such that $h(0) = (1,1)$ \textit{and} $h(1) = (1,1)$.  
% Hence,
Hence, we cannot compute the lower bound in closed form for the query. It appears that we need the input variable to the hyperarc \(h\) to be critical for the query in order to compute closed form bounds, which is formalized by the condition \(A \subseteq C(\mcal{Q})\).
\begin{figure}[t]
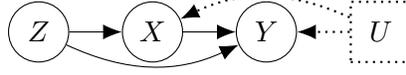

\begin{center}

  \begin{influence-diagram}
    \node (X) [] {$X$};
    \node (Y) [right = of X] {$Y$};
    \node (Z) [left = of X] {$Z$};
    \node (U) [right  = of Y, information, decision] {$U$};
    
    \edge {Z} {X};

    \edge {X} {Y};
    \path (Z) edge[->, bend right=25] (Y);
    \path (U) edge[->, bend right=25, information] (X);
    \edge[information] {U} {Y};
  \end{influence-diagram}
\end{center}
 \caption{Causal Graph for Query \(\mathcal{Q}_1\)}
 \label{fig:closedformgraph}
\end{figure}

%\begin{restatable}[Bounds for Special Class of Problems]{theorem}{speciallower}
%  \label{theorem:speciallower}
%  Suppose the query \(\mcal{Q}\) 
  % Suppose the intervention $\{do(V_{I} = q)\}$ 
%  satisfies
%  S\ref{assumption:S1} and  S\ref{assumption:S2}. Then \ms{should we redefine complete and strict consistencies for better intuition? or is it not necessary?}
%  \begin{align}
%    \alpha_L & = \sum_{\{v: v_O = q_O, v_I = q_{I}\}} \mbb{P}(V_B = v_B \mid
%               V_A = v_A), \label{eq:alphaLclosedForm}\\
 %   \alpha_U &= 1-\sum_{\{v: v_O \neq q_O, v_I = q_I\}}
 %              \mathbb{P}(V_B = v_B|V_A =
 %              v_A). \label{eq:alphaUclosedForm}
%  \end{align}
  % Then $\alpha_L$ is given by
  % the sum of the probabilities which are completely consistent with
  % \(R_\mcal{Q}\). 
%\end{restatable}
\begin{restatable}[Closed Form Bounds for Special Class of Problems]{theorem}{speciallower}
  \label{theorem:speciallower}
  % Suppose the query $\mcal{Q}$ satisfies S\ref{assumption:S1} and
  % S\ref{assumption:S2}.
  Suppose $A \subseteq C(\mcal{Q})$.
  Then the optimal values of LPs \eqref{eq:prunedprimal} and \eqref{eq:prunedPrimalDualUpper} are given by%  $\alpha_L$ is equal to the sum of the
  % conditional  
  % probabilities \(p_{v_B.v_A}\) which satisfy \(R_{v_B.v_A} \subseteq
  % R_{\mcal{Q}}\).
  \begin{eqnarray*}
    \alpha_L & = & \sum_{\{v_B : v_{I} = q_{I},v_{O} = q_O\}} p_{v_B.q_A},\\
    \alpha_U & = & 1 - \sum_{\{v_B: v_{I} = q_{I}, v_{O} \neq q_O\}} p_{v_B.q_A}.
  \end{eqnarray*}
\end{restatable}
\begin{proof}
% Lemma~\ref{CCprobability} and 
Theorem~\ref{thm:CCh} implies \(R_h \subseteq R_{\mcal{Q}}\) if, and only if, there exists \(v =
(v_A, v_B)\) 
such that \(h(v_A) = v_B\), \(v_{A \cap C(\mcal{Q})} = q_{A \cap C(\mcal{Q})}\), \(v_{I}  = q_{I} \), and \(v_O = q_O\). Since $A \subseteq C(\mcal{Q})$, it
follows that ${A \cap C(\mcal{Q})} = A$, and therefore, 
% \(\mcal{Q}\) 
% satisfying S\ref{assumption:S1} and
% S\ref{assumption:S2},
% it follows that
for 
a hyperarc $h$, the set \(R_h \subseteq R_{\mcal{Q}}\) if, and only if,
there exists \(v_B\) such that 
\(h(q_A) = v_B\), \(v_{I} = q_{I}\), and \(v_O = q_O\).
Thus, it follows that 
\begin{eqnarray}
 \alpha_L &=& \sum_{\{h \in H: c_h^L=1\}} q_h \nonumber\\
  &= & \sum_{\{h \in H: \text{ \(R_h \subseteq R_{\mcal{Q}}\)\}}} q_h \nonumber \\
  % &= & \sum_{\{h \in H: \subsetack{\exists v:  h(v_A) =v_B, v_A = q_A,\\
  % v_{I_C
  % \cap B} = q_{I_C \cap B}, v_{O} = q_O}\}} q_h\label{eq:2} \
  &= & \sum_{\{v_B : v_{I} = q_{I},v_{O} = q_O\}} \sum_{\{h \in H: h(q_A) = v_B\}} q_h \label{eq:2}\\ 
  &= & \sum_{\{v_B : v_{I} = q_{I},v_{O} = q_O\}} p_{v_B.q_A} \label{eq:4}  
\end{eqnarray}
where  \eqref{eq:2} follows from the discussion above and \eqref{eq:4} from
the constraints of the pruned LP.
% Lemma~\ref{CCprobability}.

Theorem~\ref{SChyperarc} implies that \(R_h \cap R_{\mcal{Q}} = \emptyset\) if, and only if, there exists \(v =
(v_A, v_B)\) 
such that \(h(v_A) = v_B\), \(v_{A \cap C(\mcal{Q})} = q_{A \cap C(\mcal{Q})}\), \(v_{I}  = q_{I} \), and \(v_O \neq q_O\). % Therefore,
Since $A \subseteq C(\mcal{Q})$, 
% for
% any \(\mcal{Q}\) satisfying S\ref{assumption:S1} and 
% S\ref{assumption:S2},
it follows that \(R_h \cap R_{\mcal{Q}} = \emptyset\) if, and only if, there exists \(v_B\) such that
  \(h(q_A) = v_B\), \(v_{I} = q_{I}\), and \(v_O \neq q_O\).
% The objective of the pruned LP can be rewritten as follows:
Thus, it follows that
\begin{eqnarray}
\sum_{\{h \in H: c_h^U=0\}}q_h &= & \sum_{\{h \in H: \text{ \(R_h \cap R_{\mcal{Q}} = \emptyset\)} \}}
       q_h \nonumber \\ 
  % &= 1- \sum_{\subsetack{\{h \in H: h(q_I) =v_B \text{ such that }\\
  % v_{I_C \cap B} = q_{I_C \cap B}, v_{O} \neq q_O\}}}
  % q_h  \label{ubeq:2}\\ 
  &= & \sum_{\{v_B: v_{I} = q_{I}, v_{O} \neq q_O\}}
       \sum_{\{h \in H: h(q_A) = v_B\}} q_h  \label{ubeq:2}\\ 
  &= & \sum_{\{v_B: v_{I} = q_{I}, v_{O} \neq q_O\}} p_{v_B.q_A} \label{ubeq:4} 
\end{eqnarray}
where \eqref{ubeq:2} follows
from % Theorem \ref{SChyperarc}
the discussion above and \eqref{ubeq:4} from the constraints of
the pruned LP. % Lemma~\ref{SCprobability}. 
The result follows by the fact that \(\alpha_U = \sum_{\{h \in H: c_h^U=1\}}q_h = 1 -
\sum_{\{h \in H: c_h^U=0\}}q_h\).
\end{proof}
An important example of a class of problems % that
% satisfy S\ref{assumption:S1} and  S\ref{assumption:S2}
where $A \subseteq C(\mcal{Q})$ 
is the well-studied setting where multiple
confounded treatments influence one outcome~\citep{ranganath2019multiple,
  janzing2018,pmlr-v89-d-amour19a,tran2017implicit}. An example from this class of problems was discussed in Section~\ref{sec:introduction}, where the treatments are medications and procedures, and the outcome is the health outcome of the patient. 

\citet{blei2019blessings,pmlr-v139-wang21c} % Wang and Blei 
introduced the 
\emph{deconfounder} as a method to predict the expected value of the outcome variable
under treatment interventions in this setting. One of the limitations of 
the deconfounder approach is that it
% this approach is that the deconfounder method 
cannot be applied % to
% real world applications
in a setting 
where the treatment variables
have causal relationships between % each other
them~\citep{OST2019blessings,
  blessingsrejoinder2019, imai2019discussion}.  For example, in our context,  % However,
the side effects of
one treatment can influence the prescription of another
treatment~\citep{OST2019blessings}, as implied by arrows $T_1 \rightarrow
T_3$, $T_2 \rightarrow T_3$, $T_3 \rightarrow T_4$ and $T_3 \rightarrow T_5$
in Figure~\ref{fig:ExampleSpecialClass}.
The deconfounder cannot be used for inference in this setting. However, since the entire set $V_{A} = \{C_1,C_2\}$ is critical for the query, Theorem \ref{theorem:speciallower} % and (\ref{theorem:specialupper}) 
can be used to compute bounds for the 
% be applied to compute bounds in closed form for the 
query $\mathbb{E}[Y(\textbf{T} = t)|C_1 = c_1,C_2 = c_2]$ in
closed form. In particular,
\begin{eqnarray*}
\alpha_L &=& \mathbb{P}(\textbf{T} = t, Y = 1|C_1 = c_1,C_2 = c_2)\\
\alpha_U &=& 1-\mathbb{P}(\textbf{T} = t, Y = 0|C_1 = c_1,C_2 = c_2)
\end{eqnarray*}
\begin{figure}[t]
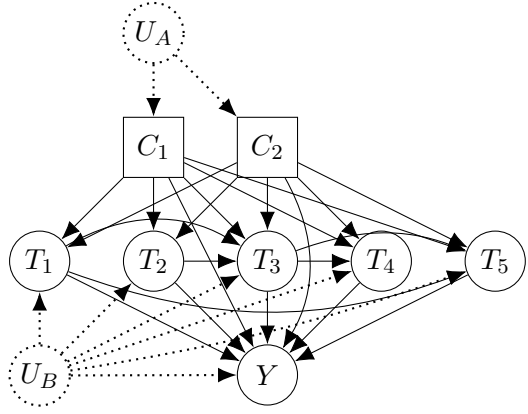


  \centering
  
  \begin{influence-diagram}
    \node (T1) [] {$T_{1}$};
    \node (T2) [right = of T1] {$T_{2}$};
    \node (T3) [right = of T2] {$T_{3}$};
    \node (T4) [right = of T3] {$T_{4}$};
    \node (T5) [right = of T4] {$T_{5}$};
    \node (Y) [below = of T3] {$Y$};
    \node (C1) [above = of T2, decision] {$C_1$};
    \node (C2) [above = of T3, decision] {$C_2$};
    \node (UB) [below = of T1, information] {$U_B$};
    \node (UA) [above = of C1, information] {$U_A$};
    
    \edge {T1} {Y};
    \edge {T2} {Y};
    \edge {T3} {Y};
    \edge {T4} {Y};
    \edge {T5} {Y};
    \edge {C1} {T1};
    \edge {C1} {T2};
    \edge {C1} {T3};
    \edge {C1} {T4};
    \edge {C1} {T5};
    \edge {C1} {Y};
    \edge {C2} {T1};
    \edge {C2} {T2};
    \edge {C2} {T3};
    \edge {C2} {T4};
    \edge {C2} {T5};
    \edge {T3} {T4}
    \edge {T2} {T3};
    
    \edge[information] {UB} {T1};
    \edge[information] {UB} {T2};
    \edge[information] {UB} {T3};
    \edge[information] {UB} {T4};
    \edge[information] {UB} {Y};
    \edge[information] {UA} {C1};
    \edge[information] {UA} {C2};
    \path (T1) edge[->, bend right=20] (T5);
    \path (T1) edge[->, bend left=30] (T3);
    \path (T3) edge[->, bend left=20] (T5);
    \path (C2) edge[->, bend left=30] (Y);
    \path (UB) edge[->, bend right=5, information] (T5);
  \end{influence-diagram}
  \caption{Example of $G$ where $V_A = \{C_1,C_2\}$ and $V_B = \{T_1,
      T_2, T_3, T_4, T_5, Y\}$} 
  \label{fig:ExampleSpecialClass}
  
\end{figure}

\section{Bounds with Additional Observations}
\label{sec:general,observations}
So far, we have allowed only observations of \(V_A\) variables to provide context in a query. In this section, we show how to incorporate additional observed variables from \(V_B\) in the query. We show how to efficiently compute bounds for the 
counterfactual query $\mathcal{Q}_{W} = \mathbb{P}\big(V_{O}(V_{I} = q_I) = q_O|V_{A} = q_{A}, V_W = q_W \big)$ where \(V_{W} \subseteq V_{B}\). % However, we want to limit ourselves to queries for
Define 
\begin{align*}
  R_{\mcal{W}} &= \left\{r: F_{W}(V_{A} = q_{A}, r) = q_{W}\right\}.
\end{align*}
 to be the set of $r$ values that are consistent with the observation $(V_A, V_W)
 = (q_A, q_W)$. Note that since $V_{W}$ topologically follow $V_A$, the function \(F_{W}(V_{A} = q_{A}, r)\), and thus the set $R_{\mcal{W}}$, are well defined. 

Let \(\mcal{Q} = \mathbb{P}(V_O(V_I = q_I)=q_O| V_A = q_A)\). The
revised objective is then given by:    
\begin{eqnarray}
  \lefteqn{\mathbb{P}(V_{O}(V_{I} = q_I) = q_O|V_A = q_A, V_{W} = q_{W})}
  \hspace*{2in} \nonumber\\[0.5em]
 && \hspace*{-1.5in} =   \frac{\mathbb{P}(\{V_{O}(V_{I} = q_I, V_A = q_A) = q_O\} \cap \{V_{W}(V_{A} = q_{A}) = q_{W}\})}{\mathbb{P}(V_{W}(V_{A} = q_{A}) = q_{W})} \label{eq:fracObj1}\\ 
  && \hspace*{-1.5in}= \frac{\sum_{r \in R_{\mathcal{Q}}\cap R_{\mcal{W}}}
      q_{r}}{\sum_{r \in R_{\mcal{W}}} q_{r}} \label{eq:fracObj2}
\end{eqnarray}
where \eqref{eq:fracObj1} follows from Bayes' rule and \eqref{eq:fracObj2}
follows from the definitions of $R_{\mathcal{Q}}$ and
$R_{\mcal{W}}$. Note that our objective is now a \emph{fractional
  linear} expression in $q_{r}$, while our constraints remain the same.
However, since the fractional objective is always non-negative and
bounded, the fractional LPs can be reformulated into the following
LPs (\cite{charnes1962programming}):
% However, it can be shown that our optimization problem is equivalent to
% the following LP:  
\begin{equation}
  \begin{array}{rll}
    \min_{\alpha,q} / \max_{\alpha,q} \ & \sum_{r \in R_{\mathcal{Q}} \cap R_{\mcal{W}}} q_{r}\\
    \text{s.t.}\ & \sum_
    {r \in R_{v_B.v_A}} q_{r} - \alpha
                   p_{v_B.v_A} = 0, & \forall (v_A,v_B) \in
                   \{0,1\}^{|A|} \times \{0,1\}^{|B|},\\ 
    & \sum_{r \in R_{\mcal{W}}} q_{r} = 1, \\
    & \alpha, q\geq 0,
  \end{array}
  \label{eq:fracPrimalgeneral}
\end{equation}
% Consider the running example in Section \ref{sec:general} with causal
For the causal 
graph in Figure \ref{fig:causalgraph} and counterfactual \(\mathbb{P}(Y(X=1)=1|X=0, Z=1)\),
the set 
% In the running example, 
\(R_{\mcal{W}} = \{(r_X, r_Y): f_{r_X}(1) = 0\}\), and 
% In the running example, 
the reformulated LPs are given by:
\begin{equation}
  \begin{array}{rll}
    \min_{\alpha,q} / \max_{\alpha,q} \ & \sum_{r \in R_{\mathcal{Q}} \cap R_{\mcal{W}}} q_{r}\\
    \text{s.t.}\ & \sum_
    {r \in R_{xy.z}} q_{r} - \alpha
                   p_{xy.z} = 0, & \forall (x,y,z) \in
                   \{0,1\}^{3} \\ 
    & \sum_{r \in R_{\mcal{W}}} q_{r} = 1, \\
    & \alpha, q\geq 0,
  \end{array}
  \label{eq:fracPrimal}
\end{equation}
where \(\mcal{Q} = \mathbb{P}(Y(X=1)=1|Z=1)\).
Next, we show that the LP~\eqref{eq:fracPrimal} can be reformulated in
terms of variables $q_h = \sum_{r \in R_h} q_r$ indexed by valid hyperarcs
$h \in H$. For that we need the following result.
\begin{restatable}[]{theorem}{lemmahobservations}
\label{thm:lemmahobservations}
% or graphs satisfying
For every valid hyperarc \(h \in H\), either \(R_h \subseteq
R_{\mcal{W}}\) or \(R_h \cap R_{\mcal{W}} = \emptyset\). 
\end{restatable}
\begin{proof}
  Fix \(h \in H\). Suppose there exists \(v \in \{0,1\}^N\) such that
  \(h(v_A) = v_B\), \(v_{C(\mcal{W})} = q_{C(\mcal{W})}\) and \(v_{W} =
  q_{W}\). Then from Theorem 11\ref{thm:CChRN}, we have \(R_h \subseteq
  R_{\mcal{W}}\). % On the other hand,
  Instead, 
  suppose for every \(v \in \{0,1\}^N\)
  such that \(h(v_A) = v_B\) and \(v_{C(\mcal{W})} = q_{C(\mcal{W})}\), we have
  \(v_{W} \neq q_{W}\). Then, Theorem 11\ref{thm:SChRN} implies
  that 
  % we have 
  \(R_{h }\cap R_{\mcal{W}} = \emptyset\).
  Hence, either \(R_h
  \subseteq R_{\mcal{W}}\) or \(R_h \cap R_{\mcal{W}} = \emptyset\).
\end{proof}
Since \(R=\bigcup_h R_h\) is a partition of \(R\) (see Lemma \ref{lemma:partition}), it follows
that the constraint
\begin{equation*}
  \sum_{r \in R_{\mcal{W}}} q_r =
  \sum_{h \in H} \Big[\sum_{r \in
    R_{\mcal{W}} \cap R_{h}} q_r \Big] =
  \sum_{h \in H: R_h \subseteq R_{\mcal{W}}} \Big[\sum_{r \in R_{h}}
  q_r \Big] =  \sum_{h \in H: R_h \subseteq R_{\mcal{W}}} q_h
  \label{eq:rewriteRWthm12} 
\end{equation*}
where % \eqref{eq:rewriteRWthm12}
the second equality
follows from Theorem
\ref{thm:lemmahobservations}.  % Recall from
The results in Section~\ref{sec:general,pruningLP} imply that % the
% variables $q_r$, for $r \in  
% R_h$, are all identical in terms of their impact on 
% the constraints
for any $(v_A, v_B) \in \{0,1\}^{|A|} \times \{0,1\}^{|B|}$
\[
  \sum_{r \in R_{v_B.v_A}} q_r  =  \sum_{\{h \in H: h(v_A)  = v_B\}}
  \sum_{r \in R_h} q_r = \sum_{\{h \in H: h(v_A)  = v_B\}} q_h.
\]
Thus, the constraints can be written in terms of $q_h$. 
% - \alpha p_{v_B.v_A} = 0,
% \forall (v_A, v_B) \in \{0,1\}^{|A|} \times \{0,1\}^{|B|}\).
% Hence, the
% variables \(q_r, r\in R_h\) are identical in terms of their impact on all
% constraints in LP \eqref{eq:fracPrimalgeneral}. 
An argument identical to the one in
Section~\ref{sec:general,pruningLP} implies that the objective coefficient of $q_h$
in the minimization LP is given by 
\[
d^L_h := % \min_{s \in R_h} \ones\{s \in R_{\mcal{Q}} \cap R_{\mcal{W}} \} =
\ones\{R_h \subseteq R_{\mcal{Q}} \cap R_{\mcal{W}}\} = \ones\{R_h \subseteq R_{\mcal{W}}\} \ones\{R_h \subseteq
  R_{\mcal{Q}}\}
\]
% Then the results in Section~\ref{sec:general,pruningLP} imply that
Thus, the lower bound LP can be reformulated as follows. 
% From Theorem \ref{thm:lemmahobservations}, 
% the variables $q_r$, for $r \in
% R_h$, are all identical in terms of their impact on the constraint
% \(\sum_{r \in R_{\mcal{W}}} q_r = 1\). Therefore, in any optimal
% solution \(q_r = 0\) for all \(r \in R_h\) with objective coefficient
% \(\ones\{r \in R_{\mcal{W}} \cap R_{\mcal{Q}}\} > \min_{s \in R_h}
% \ones\{s \in R_{\mcal{W}} \cap R_{\mcal{Q}}\} = d^L_h\).  
\begin{equation}
  \begin{array}{rll}
    \min_{q} \ & \sum_{h \in H} d_h^Lq_h\\
    \text{s.t.}\ & \sum_{h \in H: h(v_A) = v_B} q_{h} = 
                   \alpha p_{v_B.v_A},  &\forall v_A, v_B, \\
   \ & \sum_{h \in H: R_h \subseteq R_{\mcal{W}}} q_h = 1,\\
    & q, \alpha \geq 0,
  \end{array}
  \label{eq:prunedprimalobservationslower}
\end{equation}
Similarly, the % upper
% bound 
upper bound LP % for the upper bound \(\alpha_U\)
can be reformulated as: 
\begin{equation}
  \begin{array}{rll}
    \max_{q} \ & \sum_{h \in H} d_h^Uq_h\\
    \text{s.t.}\ & \sum_{h \in H: h(v_A) = v_B} q_{h} = 
                   \alpha p_{v_B.v_A},  &\forall v_A, v_B, \\
   \ & \sum_{h \in H: R_h \subseteq R_{\mcal{W}}} q_h = 1,\\
    & q, \alpha \geq 0,
  \end{array}
  \label{eq:prunedprimalobservationsupper}
\end{equation}
where
\begin{equation}
  \label{eq:dUdef}
  d^U_h := % \max_{s \in R_h} \ones\{s \in R_{\mcal{W}} \cap R_{\mcal{Q}}\} = 
  \ones\{R_h \cap (R_{\mcal{W}} \cap R_{\mcal{Q}}) \neq~\emptyset\} = 1
  - \ones\{R_h \cap 
  (R_{\mcal{W}} \cap R_{\mcal{Q}}) = \emptyset\}.
\end{equation}
% Both reformulations
The new LPs in terms of the $q_h$ variables
have \emph{exponentially} fewer variables since $\abs{H} \ll \abs{R}$;
however, they are useful only 
if, for each hyperarc \(h \in H\), $d^L_h$, $d^U_h$, and $\ones\{R_h
\subseteq R_{\mcal{W}}\}$ 
% the following three values
can be
efficiently computed.

\subsection{Efficiently computing $d^L_h$, $d^U_h$, and $\ones\{R_h
\subseteq R_{\mcal{W}}\}$}
\label{sec:pruningobservations}
% As before, we can aggregate variables using hyperarcs. In order to do
% so, we first establish Theorem \ref{thm:CChRN} (resp. \ref{thm:SChRN})
% which allows us to efficiently check if \(R_h \subseteq R_{\mcal{W}}\)
% (resp. \(R_h \cap R_{\mcal{W}} = \emptyset\)) using only the outputs of
% \(h\). Then, Theorem \ref{thm:lemmahobservations} uses these results to  
% characterize a special relationship between each hyperarc and \(R_{\mcal{W}}\), which is then used to prune \eqref{eq:fracPrimalgeneral}.
We begin by defining the set of critical variables for the observation
$V_W$. 
\begin{definition}[Critical variables for observation \{\(V_W=q_W\)\}]
The set \(V_{C(\mcal{W})} \subseteq V_A\) of critical variables for the observation
 $\{V_W=q_W\}$ is defined as the set of variables in $V_A$ that have a
 path in the graph $G$  to some variable in $V_W$.
  %$V_{crit}(\mcal{Q})$ denote the set of variables in $V_A
  %\cup V_B$ that have a path to some variable in $V_O$ in  $G^{do(V_I =
   % q_I)}$. 
\end{definition}
Next, we define conditions under which $R_h \subseteq R_{\mcal{W}}$,
$R_h \cap R_{\mcal{W}}  = \emptyset$, and $R_h \cap
(R_{\mcal{Q}} \cap R_{\mcal{W}}) = \emptyset$.
\begin{restatable}[]{theorem}{CChRN}
  \label{thm:ChRN}
  Let $V_{C(\mcal{W})} \subseteq V_A$ denote the set of critical
  variables for the observation $\{V_W=q_W\}$. Then the following results hold. 
  % or graphs satisfying
  \begin{enumerate}[(a)]
  \item 
    \(R_h \subseteq R_{\mcal{W}}\) if, and
    only if, there exists 
    \(v \in \{0,1\}^{|N|}\) such that \(h(v_A) = v_B\), $v_{C(\mcal{W})} =
    q_{C(\mcal{W})}$ and  
    $v_{W} = q_{W}$. %  and
    \label{thm:CChRN}
  \item 
    \(R_h \cap R_{\mcal{W}} = \emptyset\) if, and
    only if, there exists 
    \(v \in \{0,1\}^{|N|}\) such that \(h(v_A) = v_B\), 
    $v_{C(\mcal{W})} = q_{C(\mcal{W})}$ and 
    $v_{W} \neq q_{W}$. \label{thm:SChRN}
  \end{enumerate}
  % $
  % R_{v_B.v_A}  
  % \subseteq R_{\mcal{W}}.
  % $
  % \end{enumerate}
  % where \(I_C\) denotes the set of critical variables defined in Assumption~\ref{ass:validQuery}.
\end{restatable}
\begin{proof}
%  Suppose $R_{v_B.v_A}$ satisfies $v_{C(\mcal{W})} =
%  q_{C(\mcal{W})}$ and $v_{W} = q_{W}$. Then, every $r \in R_{v_B.v_A}$
%  maps $V_{C(\mcal{W})}=q_{C(\mcal{W})}$ to $V_{W}=q_{W}$, and thus, 
%  $r \in R_{\mcal{W}}$ i.e. $R_{v_B.v_A} \subseteq R_{\mcal{W}}$. To
%  establish the other direction, let $v = (v_A, v_B)$ be the given assignment
%  to the variables $V_N$, and  consider the following
%  two cases:
%  \begin{enumerate}[(a)]
%  \item  % either  $v_{C} \neq q_{C_C}$ or $v_{I  \cap
      % B} \neq q_{I}$ i.e. the value of
%    $v_{C(\mcal{W})} \neq q_{C(\mcal{W})} $ % in $R_{v_B.v_A}$.
%    but $v_{W} = q_{W}$:
    % is not $q_{I_C}$
%    Since $R$ % is exhaustive,
%    indexes all possible functions on the causal graph, 
%    there
%    exists $r$ such that: 
%    \begin{itemize}
%    \item $r$ maps $V_A = v_A$ to $V_B = v_B$, i.e. $r \in R_{v_B.v_A}$. % In
      % particular, 
      % $F_B(v_A, r) = v_B $
      % $r$ maps $V_{I} \neq q_{I}$ to $V_{O} = q_O$. 
%    \item  $r$ maps $V_{C(\mcal{W})} = q_{C(\mcal{W})}$ to $V_B$ such that $V_{W} \neq q_{W}$,
      % $(r_X,r_Y)$ maps $Z=1$ to $(X\neq \bar{x}, Y=0)$ 
 %     i.e. $r \not \in R_{\mcal{W}}$.
 %   \end{itemize}
 %   Thus, there exists $r$ such that $r \in
 %   R_{v_B.v_A}$, but $r \not \in R_{\mcal{W}}$ i.e. $R_{v_B.v_A}
 %   \nsubseteq R_{\mcal{W}}$.  
 % \item $v_{W} \neq q_{W}$: In this case, for  every $r \in
 %   R_{v_B.v_A}$  maps $V_{C(\mcal{W})} = q_{C(\mcal{W})}$ to $V_{W} \neq q_{W}$, and therefore, $r \not
 %   \in R_{\mcal{W}}$.  Hence $R_{v_B.v_A} \nsubseteq R_{\mcal{W}}$.
 % \end{enumerate}
 % Next, we establish (ii).
Suppose there exists $v \in \{0,1\}^{|N|}$ such that $h(v_A)=v_B$,
  $v_{C(\mcal{W})} = q_{C(\mcal{W})}$ and   
  $v_{W} = q_{W}$. Then every $r \in R_{h}$ maps
  $V_{C(\mcal{W})}=q_{C(\mcal{W})}$ to $V_{W}=q_{W}$, and thus, $r \in 
  R_{\mcal{W}}$ i.e. \(R_h \subseteq R_{\mcal{W}}\). To establish the
  reverse direction, suppose \(R_h \subseteq R_{\mcal{W}}\), but there does
  not exist $v \in  
  \{0,1\}^{|N|}$ such that $h(v_A) = v_B$, $v_{C(\mcal{W})} = q_{C(\mcal{W})}$ and 
$v_{W} = q_{W}$. Equivalently, for all $v
  \in \{0,1\}^{|N|}$ such that $h(v_A) = v_B$ and $v_{C(\mcal{W})} = q_{C(\mcal{W})}$, we have $v_{W} \neq q_{W}$. In this case,
  every $r \in 
  R_{h}$  maps $V_{C(\mcal{W})} = q_{C(\mcal{W})}$ to $V_{W} \neq q_{W}$, and therefore, $r \not 
  \in R_{\mcal{W}}$.   Hence $R_{h} \nsubseteq R_{\mcal{W}}$.

Suppose there exists $v \in \{0,1\}^{|N|}$ such that $h(v_A)=v_B$, $v_{C(\mcal{W})} = q_{C(\mcal{W})}$ and $v_{W} \neq q_{W}$. Then, every
$r \in R_{h}$ maps $V_{C(\mcal{W})}=q_{C(\mcal{W})}$ to $V_{W}\neq q_{W}$, and thus, $r \not \in R_{\mcal{W}}$ i.e. $R_{h} \cap
R_{\mcal{W}} = \emptyset$. To establish the opposite direction, suppose
\(R_h \cap 
R_{\mcal{W}} = \emptyset\), but there does not exist $v \in \{0,1\}^{|N|}$
such that $h(v_A) = v_B$, $v_{C(\mcal{W})} = q_{C(\mcal{W})}$ and $v_{W}
\neq q_{W}$. Equivalently, for all $v 
  \in \{0,1\}^{|N|}$ such that $h(v_A) = v_B$ and $v_{C(\mcal{W})} = q_{C(\mcal{W})}$, we have $v_{W} = q_{W}$. In this case, every
  $r \in 
    R_{h}$  maps $V_{C(\mcal{W})} = q_{C(\mcal{W})}$ to $V_{W} = q_{W}$, and therefore, $r  
    \in R_{\mcal{W}}$.   Hence $R_{h} \cap R_{\mcal{W}} \neq \emptyset$.
\end{proof}
If $R_{\mcal{Q}} \cap R_{\mcal{W}} = \emptyset$, % then
the probability of 
% seeing
both the observation and the query would be $0$. Hence, % seeing
the observation would invalidate the query, and there 
would be no need for bounds. Therefore, we assume that $R_{\mcal{Q}} \cap
R_{\mcal{W}} \neq \emptyset$, and in that we have the following result.
\begin{restatable}{theorem}{thmfracupperbound}
\label{thm:fracupperbound}
Suppose  $R_{\mcal{Q}} \cap R_{\mcal{W}} \neq \emptyset$. Then $R_h \cap
(R_{\mcal{Q}} \cap R_{\mcal{W}}) = \emptyset$ if, and only if, either $R_h
\cap R_{\mcal{Q}} = \emptyset$ or $R_h \cap R_{\mcal{W}} = \emptyset$. 
\end{restatable}
\begin{proof}
  % Note that if we have
  Clearly, if either $R_h \cap R_{\mcal{Q}} =
  \emptyset$ or $R_h \cap R_{\mcal{W}} = \emptyset$, it follows that $R_h \cap
  (R_{\mcal{Q}} \cap R_{\mcal{W}}) = \emptyset$.
  
  Suppose $R_h \cap
  (R_{\mcal{Q}} \cap R_{\mcal{W}}) = \emptyset$, but $R_h \cap R_{\mcal{Q}}
  \neq \emptyset$ and $R_h \cap R_{\mcal{W}} \neq
  \emptyset$. By Theorem~\ref{SChyperarc} and Theorem~\ref{thm:SChRN}, there cannot
  exist either 
  \(v^{\mcal{W}} \in 
  \{0,1\}^{|N|}\) such that \(h(v_A^{\mcal{W}}) = v_B^{\mcal{W}}\),
  \(v_{C(\mcal{W})}^{\mcal{W}} = q_{C(\mcal{W})}\) and \(v_{W}^{\mcal{W}} \neq q_{W}\), or \(v^{\mcal{Q}} \in 
  \{0,1\}^{|N|}\) such that \(h(v_A^{\mcal{Q}}) = v_B^{\mcal{Q}}\), \(v_{A \cap C(\mcal{Q})}^{\mcal{Q}} = q_{A \cap C(\mcal{Q})}\),
\(v_{I}^{\mcal{Q}} = q_{I}\) and \(v_{O}^{\mcal{Q}} \neq q_{O}\). 

Since $R$ indexes the set of all possible functions, there
exists $r$ such that: 
\begin{enumerate}[(i)]
\item $r$ maps $V_A = v_A$ to $V_B = v_B$ for all $(v_A,v_B)$ such that
  $h(v_A) = v_B$ i.e. $r \in R_h$. \label{enum2:case1}
\item $r$ maps $V_{C(\mcal{W})} = q_{C(\mcal{W})}$ to $V_{W} = q_{W}$. This does not contradict
  \ref{enum2:case1} since there cannot exist \(v^{\mcal{W}} \in
\{0,1\}^{|N|}\) such that \(h(v_A^{\mcal{W}}) = v_B^{\mcal{W}}\), \(v_{C(\mcal{W})}^{\mcal{W}} = q_{C(\mcal{W})}\) and \(v_{W}^{\mcal{W}} \neq q_{W}\).  \label{enum2:case2}
\item $r$ maps \(V_{A \cap C(\mcal{Q})} = q_{A \cap C(\mcal{Q})}\) and $V_{I} = q_{I}$ to $V_{O} = q_O$. This does not contradict
  \ref{enum2:case1} since there cannot exist \(v^{\mcal{Q}} \in
\{0,1\}^{|N|}\) such that \(h(v_A^{\mcal{Q}}) = v_B^{\mcal{Q}}\), \(v_{A \cap C(\mcal{Q})}^{\mcal{Q}}=q_{A \cap C(\mcal{Q})}\), \(v_{I}^{\mcal{Q}} = q_{I}\) and \(v_{O}^{\mcal{Q}} \neq q_{O}\), and it does not contradict
  \ref{enum2:case2} since $R_{\mcal{Q}} \cap R_{\mcal{W}} \neq \emptyset$.    
\end{enumerate}

Hence, $R_h \cap (R_{\mcal{Q}} \cap R_{\mcal{W}}) \neq \emptyset$, a contradiction.
\end{proof}
Thus, we have established that $\ones\{R_h \subseteq R_{\mcal{W}}\}$ can
be efficiently 
computed using Theorem~\ref{thm:ChRN}~(\ref{thm:CChRN}),
$d_h^L = 
\ones\{R_h \subseteq R_{\mcal{W}}\} \ones\{R_h \subseteq
R_{\mcal{Q}}\}$ can  be efficiently computed using
Theorem~\ref{thm:ChRN}~(\ref{thm:CChRN}) and 
Theorem~\ref{thm:CCh}, and 
$d^U_h = 1 - \ones\{R_h \cap (R_{\mcal{W}} \cap R_{\mcal{Q}})
=~\emptyset\}$ %  Suppose  $R_{\mcal{Q}} \cap R_{\mcal{W}} \neq
% \emptyset$, i.e. the observation does not invalidate the query. Then
% Theorem~\ref{thm:fracupperbound}  establishes that $\ones\{R_h \cap
% (R_{\mcal{W}} \cap R_{\mcal{Q}}) = \emptyset\}$ 
can be efficiently
computed using Theorem~\ref{SChyperarc} and
Theorem~\ref{thm:ChRN}~(\ref{thm:SChRN}) provided $R_{\mcal{Q}} \cap
R_{\mcal{W}} \neq \emptyset$, i.e. the observation does not invalidate the query.
% \end{enumerate}
% \(\ones\{R_h \subseteq R_{\mcal{W}}\},
% d_h^L = \ones\{R_h \subseteq R_{\mcal{W}} \cap R_{\mcal{Q}}\}\), and
% \(d_h^U = \ones\{R_h \cap (R_{\mcal{W}} \cap R_{\mcal{Q}}) \neq
% \emptyset\} \) can be 
% efficiently computed, i.e. in particular, without formulating the original LPs or
% iterating over~\(R\) and using only the outputs of \(h\). 
% Note that \(\ones\{R_h \subseteq R_{\mcal{W}}\}\) can be efficiently
% computed via Theorem \ref{thm:CChRN}. 
% Furthermore, note  
% \[R_h \subseteq 
% R_{\mcal{Q}} \cap R_{\mcal{W}} \iff R_h \subseteq
% R_{\mcal{Q}} \text{ and } R_h \subseteq R_{\mcal{W}}\] 
% Since  Theorem~\ref{thm:CCh} 
% allows us to efficiently check
% if \(R_h \subseteq R_{\mcal{Q}}\), and 
% Theorem~\ref{thm:CChRN} allows us to efficiently
% check if \(R_h \subseteq R_{\mcal{W}}\), both theorems can be used to
% efficiently compute \(d_h^L = \ones\{R_h \subseteq R_{\mcal{W}} \cap
% R_{\mcal{Q}}\}\). 
% Finally we show that 
% \(R_h
% \cap (R_{\mcal{Q}}\cap R_{\mcal{W}}) = \emptyset\) if, and only if,
% either $R_h \cap R_{\mcal{Q}} = \emptyset$ or \(R_h \cap R_{\mcal{W}} =
% \emptyset\), as long as the observation does not invalidate the
% query. Whether \(R_h \cap R_{\mcal{Q}} = \emptyset\) is true can be
% efficiently checked via Theorem~\ref{SChyperarc} while whether \(R_h
% \cap R_{\mcal{W}} = \emptyset\) is true can be efficiently checked via
% Theorem~\ref{thm:SChRN}. Hence, we can efficiently compute \(d_h^U
% =\ones\{R_h \cap (R_{\mcal{Q}}\cap R_{\mcal{W}}) \neq \emptyset\}\). 
% \subsection{Computing the Pruned LPs}
% \label{sec:finalprocedureobservations}
Algorithm~\ref{alg:ourprocedureobservations}
describes a
procedure that uses results established in
Section~\ref{sec:pruningobservations}
to efficiently construct pruned LPs
\eqref{eq:prunedprimalobservationslower} and
\eqref{eq:prunedprimalobservationsupper} without formulating the original
LPs or iterating over~\(R\).

\begin{algorithm}[t]
\noindent \textbf{Input:} (i) causal graph \(G\) (ii) query \(\mcal{Q}_W = \mathbb{P}(V_O(V_I=q_I)=q_O|V_A = q_A, V_{W} =q_{W})\) (iii) conditional probability distribution $p_{v_B.v_A} = \mathbb{P}(V_B =
v_B|V_A = v_A)$, for all $v_A, v_B$.

\textbf{Output:} Pruned LPs \eqref{eq:prunedprimalobservationslower} and \eqref{eq:prunedprimalobservationsupper}\\
\(H \gets \emptyset\)\\
\For{\(h: V_A \rightarrow V_B\)}{
\If{\(h\) is valid (Theorem~\ref{validityh})}{
$H \gets H \cup \{h\}$\\
Compute \(\ones\{R_h \subseteq R_{\mcal{W}}\}\) using Theorem \ref{thm:CChRN}\\
Compute \(d_h^L\) using Theorems \ref{thm:CCh} and
          \ref{thm:CChRN}\\
Compute \(d_h^U\) using Theorems
          \ref{SChyperarc} and \ref{thm:SChRN} 
}

}
\Return{LPs \eqref{eq:prunedprimalobservationslower} and \eqref{eq:prunedprimalobservationsupper} constructed using \((H, \{\ones\{R_h \subseteq R_{\mcal{W}}\}\}_{h \in H}, d^L, d^U)\)}

\caption{Procedure to efficiently construct LPs \eqref{eq:prunedprimalobservationslower}, \eqref{eq:prunedprimalobservationsupper}}
\label{alg:ourprocedureobservations}
\end{algorithm}

\section{Numerical Experiments}
\label{sec:numericals}
In this section, we report the results of our numerical experiments
validating and extending the benefits from the new methods proposed in
Sections~\ref{sec:general,pruningLP} 
and~\ref{sec:general,observations}. % with numerical experiments.  

\subsection{Run time improvements}
In this section, we numerically verify the computational savings from
using the structural results in Sections~\ref{sec:general,pruningLP} 
and~\ref{sec:general,observations} compared to 
% to our results compared to 
the benchmark methods that iterate over the set $R$ % -based methods 
to compute the
pruned LPs.

% \subsubsection{Bounds without Observations}
% We first numerically verify the computational savings coming from the fact
% that we no longer have to iterate over the set $R$ to compute \(H\) and
% \(c_h^L, c_h^U, \forall h \in H\) in LPs \eqref{eq:prunedprimal} and
% \eqref{eq:prunedPrimalDualUpper}.

In Table~\ref{tab:numericalexperiments} we report the results of the
experiments for constructing the pruned LPs \eqref{eq:prunedprimal} and
\eqref{eq:prunedPrimalDualUpper}. Here
$t_H$ denotes the time (in 
seconds) taken by % Step 1 in 
Algorithm~\ref{alg:ourprocedure} in Section
\ref{sec:finalprocedure}, and $t_R$ denotes 
the time (in seconds) taken by Algorithm~\ref{alg:bechmarkR}, the
benchmark that iterates over the set $R$. % -based method to compute \(H\)
% and $c_h^L$, $c_h^U$ for all  $h 
% \in H$.
The results clearly show the 
significant runtime  improvement % in efficiency
provided by our method compared to the benchmark on the five examples. Note
that the benchmark methods were terminated after 1200 seconds. 

\begin{algorithm}[t]
\noindent \textbf{Input:} (i) Causal graph \(G\) (ii) Query \(\mcal{Q} = \mathbb{P}(V_O(V_I=q_I)=q_O|V_A =q_A)\) (iii) Conditional probability distribution $p_{v_B.v_A} = \mathbb{P}(V_B =
v_B|V_A = v_A)$, for all $v_A, v_B$\\
\textbf{Output:} Pruned LPs \eqref{eq:prunedprimal} and \eqref{eq:prunedPrimalDualUpper}\\
\(R_h \leftarrow \emptyset, \forall h: V_A \rightarrow V_B\)\\
\(c_h^L \leftarrow 1, c_h^U \leftarrow 0, \forall h: V_A \rightarrow V_B\)\\
\For{\(r \in R\)}{
Compute \(F_B(V_A = v_A,r), \forall v_A \in \{0,1\}^{|A|}\)\\
Compute \(F_O((V_{A},V_{I}) = (q_{A},q_{I}), r)\)\\
\(R_{\bar{h}}
\gets R_{\bar{h}} \cup \{r\}\) for the  hyperarc $\bar{h}: v_A \mapsto F_B(V_A = v_A, r)$\\
\If{\(c_{\bar{h}}^L = 1\) and \(F_O((V_{A},V_{I}) = (q_{A},q_{I}), r) = 0\)}{ \(c_{\bar{h}}^L \gets 0\) 
}
\If{\(c_{\bar{h}}^U = 0\) and \(F_O((V_{A},V_{I}) = (q_{A},q_{I}), r) = 1\)}{\(c_{\bar{h}}^U \gets 1\)}
}
\(H \gets \{h: R_h \neq \emptyset\}\)\\
\Return{LPs \eqref{eq:prunedprimal} and \eqref{eq:prunedPrimalDualUpper}
  constructed using \((H, c^L, c^U)\)}
 \caption{Benchmark R-based method to construct LPs \eqref{eq:prunedprimal}, \eqref{eq:prunedPrimalDualUpper}}

 \label{alg:bechmarkR}
\end{algorithm}

\begin{table}[t]
\begin{center}
\begin{tabular}{||c|c|c||} 
 \hline
 \multicolumn{1}{|c|}{Graph}&
 \multicolumn{1}{|c|}{\(t_H\) (s)}&
 \multicolumn{1}{|c|}{\(t_R\) (s)} \\ [0.5ex]

 \hline\hline
 Ex A & 2.0 & \(>1200\) \\
 \hline
 Ex B & 3.9 & 72.6 \\
 \hline
 Ex C & 3.7 & \(>1200\) \\
 \hline
 Ex F & 0.6 & \(>1200\) \\
 \hline
 Ex G & 2.4 & \(>1200\)\\
 \hline
\end{tabular}
\end{center}
\caption{
Runtime of % Step 1 of 
Algorithm \ref{alg:ourprocedure} in Section \ref{sec:finalprocedure} % to
% compute \(H\) and \(c_h^L, c_h^U,
% \forall h \in H\)
to construct the pruned LPs \eqref{eq:prunedprimal} and
\eqref{eq:prunedPrimalDualUpper}
as compared to Algorithm~\ref{alg:bechmarkR}, the benchmark that iterates over the set $R$ for
Examples~A, B, C, F and G (details in
Appendix~\ref{causalinferenceprobs}). Here \(t_{H}\) (resp. $t_R$) denotes the % total
runtime in seconds of our proposed method  (resp. iterating over $R$).}  % and 
% \(t_{R}\) denotes the runtime in seconds of computing the pruned LP by
% iterating over \(R\)
\label{tab:numericalexperiments}
\end{table}

% \subsubsection{Bounds with Observations}

% We numerically verify the computational savings coming from the fact that
% we no longer have to iterate over the set $R$ to compute the set of valid hyperarcs \(H\), and for each hyperarc \(h \in H\), \(\ones\{R_h \subseteq R_{\mcal{W}}\}, d_h^L = \ones\{R_h \subseteq R_{\mcal{W}} \cap R_{\mcal{Q}}\}\), and \(d_h^U = \ones\{R_h \cap (R_{\mcal{W}} \cap R_{\mcal{Q}}) \neq \emptyset\} \) in LPs \eqref{eq:prunedprimalobservationslower} and \eqref{eq:prunedprimalobservationsupper}. 
In
Table~\ref{tab:numericalexperimentsobservations} we report the runtimes
for constructing the pruned LPs~\eqref{eq:prunedprimalobservationslower}
and \eqref{eq:prunedprimalobservationsupper}.  Here 
$t_H$ denotes the time
(in seconds) taken by % Step 1 in 
Algorithm~\ref{alg:ourprocedureobservations} in
Section~\ref{sec:pruningobservations} and $t_R$ denotes the time (in 
seconds) taken by Algorithm~\ref{alg:bechmarkRobservations} that iterates
over the set $R$.  % the
% benchmark $R$-based method to compute
% \(H\) and \(\ones\{R_h \subseteq
% R_{\mcal{W}}\}, d_h^L = \ones\{R_h
% \subseteq R_{\mcal{W}} \cap
% R_{\mcal{Q}}\}\), and \(d_h^U =
% \ones\{R_h \cap (R_{\mcal{W}} \cap
% R_{\mcal{Q}}) \neq \emptyset\}, \forall
% h \in H\). 
The results
clearly show the significant runtime improvement provided by our method
compared to the benchmark on the five examples. Note that the benchmark
methods were terminated after 1200 seconds.

\begin{algorithm}[htp!]
\noindent \textbf{Input:} (i) causal graph \(G\) (ii) query \(\mcal{Q} = \mathbb{P}(V_O(V_I=q_I)=q_O|V_A= q_A, V_{W} =q_{W})\) (iii) conditional probability distribution $p_{v_B.v_A} = \mathbb{P}(V_B =
v_B|V_A = v_A)$, for all $v_A, v_B$

\textbf{Output:} Pruned LPs \eqref{eq:prunedprimalobservationslower} and \eqref{eq:prunedprimalobservationsupper}\\
\(H \leftarrow \emptyset\)\\
\(R_h \leftarrow \emptyset, \forall h: V_A \rightarrow V_B\)\\
\(c_h^L \leftarrow 1, c_h^U \leftarrow 0, \forall h: V_A \rightarrow V_B\)\\
\(d_h^L \leftarrow 1, d_h^U \leftarrow 0, \forall h: V_A \rightarrow V_B\)

\For{\(r \in R\)}{
Compute \(F_B(V_A = v_A,r), \forall v_A \in \{0,1\}^{|A|}\)\\
Compute \(F_O((V_{A},V_{I}) = (q_{A},q_{I}), r)\)\\
Compute \(F_{W}(V_{A} = q_{A}, r)\)\\
\(R_{\bar{h}}
\gets R_{\bar{h}} \cup \{r\}\) for the  hyperarc $\bar{h}: v_A \mapsto F_B(V_A = v_A, r)$\\
\If{\(c_{\bar{h}}^L = 1\) and \(F_O((V_{A},V_{I}) = (q_{A},q_{I}), r) = 0\)}{\(c_{\bar{h}}^L \leftarrow 0\)}
\If{\(c_{\bar{h}}^U = 0\) and \(F_O((V_{A},V_{I}) = (q_{A},q_{I}), r) = 1\)}{\(c_{\bar{h}}^U \leftarrow 1\)}
\If{\(d_{\bar{h}}^L = 1\) and \(F_{W}(V_{A} = q_{A}, r) = 0\)}{\(d_{\bar{h}}^L \leftarrow 0 \)}
\If{\(d_{\bar{h}}^U = 0\) and \(F_{W}(V_{A} = q_{A}, r) = 1\)}{\(d_{\bar{h}}^U \leftarrow 1 \)}
}
\(H \leftarrow \{h: R_h \neq \emptyset\}\)\\
\For{\(h \in H\)}{
\(\ones\{R_h \subseteq R_{\mcal{W}}\} \leftarrow d_h^L\)\\ 
\(d_h^L \leftarrow d_h^L \times c_h^L\)\\
\(d_h^U \leftarrow d_h^U \times c_h^U\)
}
\Return{LPs \eqref{eq:prunedprimalobservationslower} and \eqref{eq:prunedprimalobservationsupper} constructed using \((H, d^L, d^U, \{\ones\{R_h \subseteq R_{\mcal{W}}\}\}_{h \in H})\)}
 \caption{Benchmark R-based method to construct LPs \eqref{eq:prunedprimalobservationslower} and \eqref{eq:prunedprimalobservationsupper}}  
 \label{alg:bechmarkRobservations}
\end{algorithm}

\begin{table}[t]
\begin{center}
\begin{tabular}{||c|c|c||} 
 \hline
 \multicolumn{1}{|c|}{Graph}&
 \multicolumn{1}{|c|}{\(t_H\) (s)}&
 \multicolumn{1}{|c|}{\(t_R\) (s)} \\ [0.5ex]

 \hline\hline
 Ex A & 2.1 & \(>1200\) \\
 \hline
 Ex B & 4.5 & 74.4 \\
 \hline
 Ex C & 4.3 & \(>1200\) \\
 \hline
 Ex F & 0.7 & \(>1200\) \\
 \hline
 Ex G & 3.1 & \(>1200\)\\
 \hline
\end{tabular}
\end{center}
\caption{Runtime of % Step 1 of 
Algorithm \ref{alg:ourprocedureobservations} in Section \ref{sec:pruningobservations} % to
% compute \(H\) and \(c_h^L, c_h^U,
% \forall h \in H\)
to construct the pruned LPs \eqref{eq:prunedprimalobservationslower} and
\eqref{eq:prunedprimalobservationsupper}
as compared to Algorithm~\ref{alg:bechmarkRobservations}, the benchmark that iterates over the set $R$ for
Examples~A, B, C, F and G (details in
Appendix~\ref{causalinferenceprobs}). Here \(t_{H}\) (resp. $t_R$) denotes the % total
runtime in seconds of our proposed method  (resp. iterating over $R$).} % and 
% \(t_{R}\) denotes the runtime in seconds of computing the pruned LP by
% iterating over \(R\).
\label{tab:numericalexperimentsobservations}
\end{table}

% Let 
% \(t_{H}\) (resp. $t_R$) denote the total runtime % in seconds of our
% % proposed method to
% our hyperarc-based (resp. $R$-based) method for constructing 
% compute the pruned LP and \(t_R\) denote the runtime in seconds of
% computing the pruned LP by iterating over \(R\). 

\subsection{Finite Data Setting}
\label{sec:finitedata}
In this section, we show how to extend the pruning procedure to LPs % to the setting 
where 
% the input data
the conditional probabilities
$p_{v_B.v_A}$ % \forall (v_A, v_B) \in \{0,1\}^{|A|} \times
% \{0,1\}^{|B|}
are estimated from finite amount of data, and therefore, % are
are known only within % a confidence interval.
a tolerance. 
% \subsubsection{Bounds without Observations}
% \label{sec:finitedatanoobservations}
% We first show how to extend LPs~\eqref{eq:prunedprimal} and \eqref{eq:prunedPrimalDualUpper} to the setting where
% the input data $p_{v_B.v_A}$ is estimated from finite amount of data.
% We now show that the results in Section~\ref{sec:general,pruningLP}
% automatically apply even in the finite data setting. 
Let % \(p\) denote the
% vector of true input probabilities and
\(\bar{p}\) denote a sample
estimate of \(p\). % Let \(\delta \in \mathbb{R}\) be a hyperparameter
% appropriately chosen such that \(||p - \bar{p}||_{\infty} \leq \delta\)
% i.e. \(p\) is contained in a ball of radius \(\delta\) around
% \(\bar{p}\). Note:
Suppose the  (unknown) true conditional
probability $p_{v_B.v_A} \in 
[\bar{p}_{v_B.v_A} - \delta, \bar{p}_{v_B.v_A} + \delta]$ with high 
confidence. % Then we have that 
% \[
%   \norm{p - \bar{p}}_{\infty} \leq \delta \quad \iff \quad  % & \max_{v_B,
%   % v_A} |p_{v_B.v_A} - \bar{p}_{v_B.v_A}|
%   % \leq \delta\\ 
%   % &\iff&
%   |p_{v_B.v_A} - \bar{p}_{v_B.v_A}| \leq \delta, \quad \forall\ v_A, v_B.
%   % &\iff& p_{v_B.v_A} \leq \bar{p}_{v_B.v_A}+\delta \text{ and }
%   % p_{v_B.v_A} \geq \bar{p}_{v_B.v_A} - \delta, \forall v_A, v_B\\
% \]
The bounds in this setting can be computed by solving the following LPs:
% Then we can solve the linear program \eqref{lp:finitedata}  to compute bounds:
\begin{equation}
  \begin{array}{rrll}
    \alpha_L^F / \alpha_U^F =
    & \min / \max_{q,p} \ & \sum_{r: r \in R_{\mcal{Q}}} q_r\\
    & \text{s.t.}\ & \sum_{r \in R_{v_B.v_A}} q_{r} \leq 
                     \bar{p}_{v_B.v_A}+\delta, & \forall v_A, v_B\\
    &&\sum_{r \in R_{v_B.v_A}} q_{r} \geq 
       \bar{p}_{v_B.v_A}-\delta, & \forall v_A, v_B\\ 
    && \sum_{r \in R} q_r = 1,\\
    && q \geq 0.                    
  \end{array}
  \label{eq:finitedataprimal}
\end{equation}
% Note that
Note that the structure of the objective function and the constraints in
the $q$ variables is the same as those
in~\eqref{eq:generalprimal}. Therefore, the results 
in Section~\ref{sec:general,pruningLP} % can be % automatically
% used
can be used 
% leverage the structure of the constraint matrix, and not the data
% $p$. Therefore, the previous results allow us 
% % applied 
to aggregate the $q$-variables in~\eqref{eq:finitedataprimal} and rewrite the
LPs in terms of variables corresponding to hyperarcs.
\begin{equation}
  \begin{array}{rrll}
    \alpha_L^F / \alpha_U^F =
    & \min / \max_{q,p} \ & \sum_{h \in H} c^{L/U}_h q_h\\
    & \text{s.t.}\ & \sum_{h: h(v_A) = v_B} q_{h} \leq
                     \bar{p}_{v_B.v_A}+\delta, & \forall v_A, v_B\\ 
    && \sum_{h: h(v_A) = v_B} q_{h} \geq \bar{p}_{v_B.v_A}-\delta, & \forall v_A, v_B\\
    && \sum_{h \in H} q_h = 1\\
    && q \geq 0.                    
  \end{array}
  \label{eq:finitedataprimal2}
\end{equation}
% Hence, the
% lower bound LP \eqref{eq:finitedataprimal} can be reformulated in terms of
% hyperarcs: 
% \begin{equation}
% \label{eqref:lowerboundfinitedata}
%   \hspace*{-10pt}
%   \begin{array}{rl}
%     \alpha_L^F = \min_{q,p}  \ & \sum_{h \in H: c_h^L = 1} q_h\\
%     \text{s.t.}\ & \sum_{h \in H: h(v_A) = v_B} q_{h} = 
%                    p_{v_B.v_A},\ \forall  v_A, v_B,\\
%                    & p_{v_B.v_A} \leq \bar{p}_{v_B.v_A}+\delta, \forall v_A, v_B\\
%                    &  p_{v_B.v_A} \geq \bar{p}_{v_B.v_A}-\delta, \forall v_A, v_B\\
%                    & \sum_{h \in H: c_h^L = 1} q_h \leq 1\\ 
%                    & q, p\geq 0.
%   \end{array}
% \end{equation}
% Similarly, the upper bound LP can be reformulated as:
% \begin{equation}
% \label{eqref:upperboundfinitedata}
%   \hspace*{-10pt}
%   \begin{array}{rl}
%     \alpha_U^F = \max_{q,p}  \ & \sum_{h \in H: c_h^U = 1} q_h\\
%     \text{s.t.}\ & \sum_{h \in H: h(v_A) = v_B} q_{h} = 
%                    p_{v_B.v_A},\ \forall  v_A, v_B,\\
%                   & p_{v_B.v_A} \leq \bar{p}_{v_B.v_A}+\delta, \forall v_A, v_B\\
%                    &  p_{v_B.v_A} \geq \bar{p}_{v_B.v_A}-\delta, \forall v_A, v_B\\
%                    & \sum_{h \in H: c_h^U = 1} q_h \leq 1\\ 
%                                             & q,p \geq 0.                                    
%   \end{array}
% \end{equation}
% We now show how LPs \eqref{eqref:lowerboundfinitedata} and
% \eqref{eqref:upperboundfinitedata} can be used to compute bounds for large
% examples involving uncertainty in input data.
In Figure~\ref{fig:boundsdeltaexA} we plot the upper and lower bounds computed by solving \eqref{eq:finitedataprimal2} as a
function of $\delta$ for $10$ randomly generated instances of \(\bar{p}\)
in Example A. Note that as the uncertainty in \(p\) increases, the computed bounds become wider. Recall that  for this example 
optimal bounds can only be computed after pruning the LP. % In
% Figure~\ref{fig:boundsdeltaexA} we present the optimal solutions of
% \eqref{eqref:lowerboundfinitedata} and \eqref{eqref:upperboundfinitedata}
% for 10 randomly generated samples \(\bar{p}\) and different values of
% \(\delta\). As expected, the width of the bounds computed increases as
% \(\delta\) increases (i.e. there is more uncertainty in the sample
% estimate). 
Furthermore, standard results in LP
duality~\citep{linearprogbertsimas}  imply that %  the change in the optimal value of the 
% primal LP as a function of the accuracy of $\bar{p}_{v_B.v_A}$ can be
% approximated by 
% % of $\delta$ can be approximated by 
% $(\lambda^{+}_{v_B.v_A} - \lambda^{-}_{v_B.v_A}) \delta$. 
% % Then, standard results in
% % LP duality imply that \(\frac{d\alpha_L^F}{dp_{v_B.v_A}} =
% % \lambda^*_{v_B.v_A}\) and \(\frac{d\alpha_U^F}{dp_{v_B.v_A}} =
% % \lambda^*_{v_B.v_A}\) \citep{linearprogbertsimas}.
% Hence, 
one can utilize
the optimal duals % obtained 
to identify % probabilities in the input data
% whose perturbations lead to the greatest change in objective values and
% are most critical to tighten the bound.
\(\bar{p}_{v_B.v_A}\) values that impact the bounds most, and concentrate additional
measurement on these values. 
% We can then focus on sampling more
% for those probability values.
We would like to reiterate that the analysis presented here is only
possible because we were able to prune the benchmark $R$-based LP for
Example A.

\begin{figure}[t]
  \centering
  \includegraphics[height=3in]{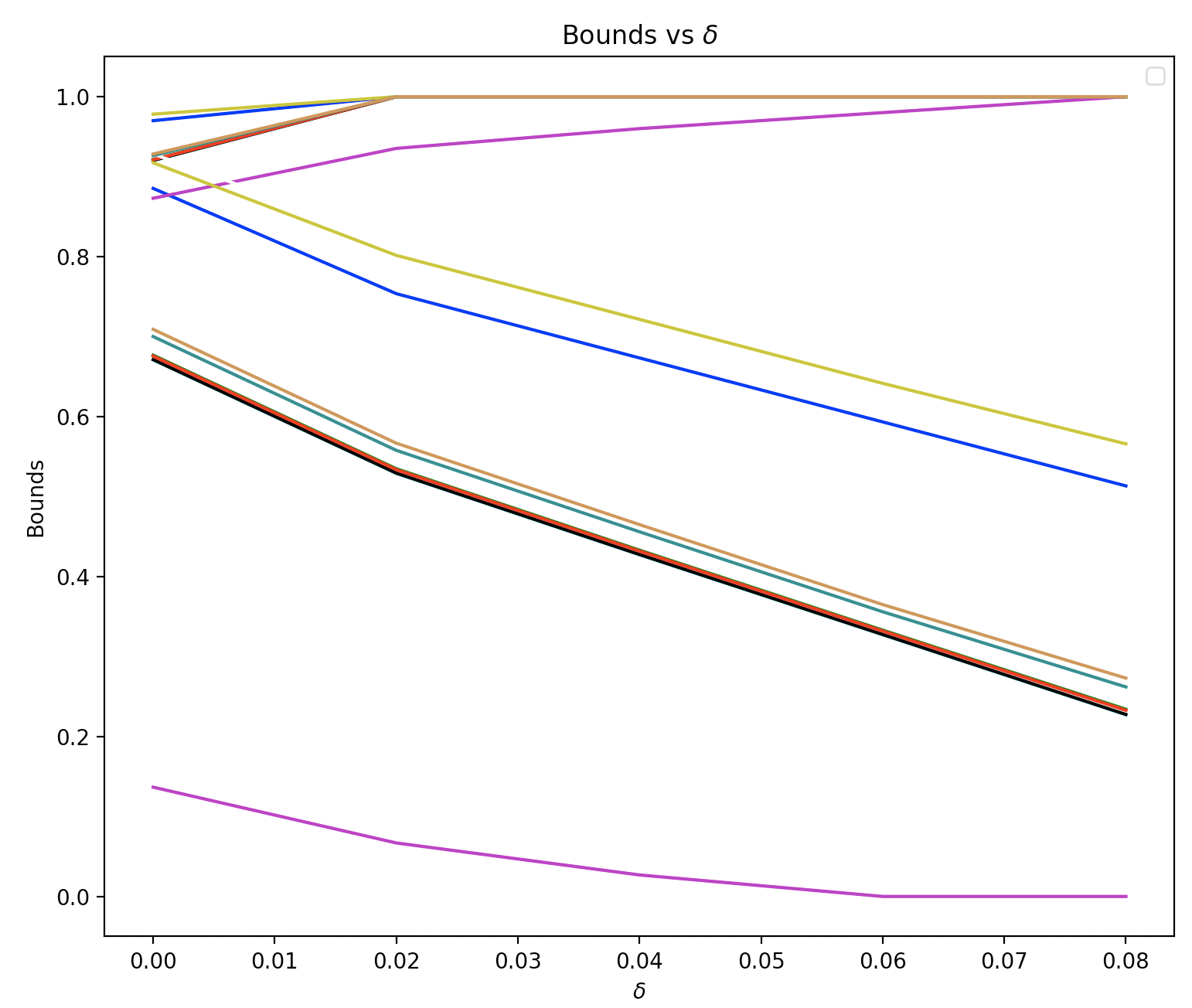}
  \caption{Lower and upper bounds vs \(\delta\) for 10 instances of
    \(\bar{p}\) in Example A. Lines of the same color denotes lower and
    upper bounds for the same instance.} 
  \label{fig:boundsdeltaexA}
\end{figure}

% Since we are able to solve the pruned LPs
% \eqref{eqref:lowerboundfinitedata} and \eqref{eqref:upperboundfinitedata}
% using standard LP solvers, we can also obtain optimal dual solutions. 
% et
% \(\lambda^{+}_{v_B.v_A}\) (resp. $\lambda^{-}_{v_B.v_A}$) denote the
% optimal dual solution corresponding to 
% the constraint % with probability \(p_{v_B.v_A}\).
% $p_{v_B.v_A} \leq \bar{p}_{v_B.v_A} + \delta$ (resp. $p_{v_B.v_A} \geq
% \bar{p}_{v_B.v_A}+\delta$). Then,

% \subsubsection{Bounds with Observations}
We now extend LPs~\eqref{eq:prunedprimalobservationslower} and
\eqref{eq:prunedprimalobservationsupper} to the finite data
setting. % Similar to Section \ref{sec:finitedatanoobservations}, 
Bounds in
this setting can be computed by solving the following LPs: 
\begin{equation}
  \begin{array}{rll}
    \min_{q} / \max_{q} \
    & \frac{ \sum_{r \in
      R_{\mathcal{Q}} \cap R_{\mcal{W}}} q_{r}}{\sum_{r \in R_{\mcal{W}}} q_{r}}\\ 
    \text{s.t.}\ & \sum_{r \in R_{v_B.v_A}} q_{r} \leq \bar{p}_{v_B.v_A} +
                   \delta, & \forall (v_A,v_B), \\ 
    & \sum_{r \in R_{v_B.v_A}} q_{r}  \geq \bar{p}_{v_B.v_A} - \delta, & \forall (v_A,v_B),\\
    & \sum_{r \in R} q_{r} = 1,\  q\geq 0.
  \end{array}
  \label{eq:fracPrimalFiniteData}
\end{equation}
% Equivalently, we have,
% \begin{equation}
%   \begin{array}{rll}
%     \min_{q} / \max_{q} \
%     & \frac{ \sum_{r \in
%       R_{\mathcal{Q}} \cap R_{\mcal{W}}} q_{r}}{\sum_{r \in R_{\mcal{W}}} q_{r}}\\ 
%     \text{s.t.}\ & \sum_
%     {r \in R_{v_B.v_A}} q_{r} \leq \bar{p}_{v_B.v_A} + \delta, & \forall (v_A,v_B), \\
%     & \sum_
%     {r \in R_{v_B.v_A}} q_{r} \geq \bar{p}_{v_B.v_A} - \delta, & \forall (v_A,v_B),\\
%     & \sum_{r \in R} q_{r} = 1,\  q\geq 0.
%   \end{array}
%   \label{eq:fracPrimalFiniteData}
% \end{equation}
This fractional LP can be linearized as follows.
\begin{equation}
  \begin{array}{rll}
    \min_{\alpha,q} / \max_{\alpha,q} \
    &\sum_{r \in
      R_{\mathcal{Q}} \cap R_{\mcal{W}}} q_{r}\\ 
    \text{s.t.}\ & \sum_
    {r \in R_{v_B.v_A}} q_{r} \leq (\bar{p}_{v_B.v_A} + \delta)\alpha, & \forall (v_A,v_B),\\
    &\sum_
    {r \in R_{v_B.v_A}} q_{r} \geq (\bar{p}_{v_B.v_A} - \delta)\alpha, & \forall (v_A,v_B),\\
    & \sum_{r \in R} q_{r} = \alpha,  \\
    & \sum_{r \in R_{\mcal{W}}} q_{r} = 1, \\
    & \alpha, q \geq 0.
  \end{array}
  \label{eq:LinPrimalFiniteDataObservations}
\end{equation}
Note that this is a linear program in $(\alpha, q)$. The previous results
allow us to aggregate the $q$-variables
in~\eqref{eq:LinPrimalFiniteDataObservations} and rewrite the
LPs in terms of variables corresponding to hyperarcs.
\begin{equation}
  \begin{array}{rll}
    \min / \max_{q,p} \ & \sum_{h \in H} d^{L/U}_h q_h\\
    \text{s.t.}\ & \sum_{h: h(v_A) = v_B} q_{h} \leq
                   (\bar{p}_{v_B.v_A}+\delta)\alpha, & \forall  (v_A,
                                                       v_B),\\ 
    & \sum_{h: h(v_A) = v_B} q_{h} \geq (\bar{p}_{v_B.v_A}-\delta)\alpha, & \forall (v_A, v_B),\\
    & \sum_{h \in H} q_h = \alpha,\\
    & \sum_{h \in H: R_h \subseteq R_{\mcal{W}}} q_h = 1,\\
    & \alpha, p\geq 0.                    
  \end{array}
  \label{eq:finitedataprimal2observations}
\end{equation}
 Recall that the lower and upper bounds computed by solving LPs \eqref{eq:finitedataprimal2observations} incorporate additional observations of some variables in \(V_B\). In Figure~\ref{fig:boundswithobservationsdeltaexA} we plot these bounds as a
function of $\delta$ for the same 10 instances of \(\bar{p}\), but for Example A with additional observations. (details in Appendix B). We also replicate
Figure~\ref{fig:boundsdeltaexA} with bounds computed without these additional observations for comparison; note that our bounds have shifted significantly after the
observation. As in \cite{balke94}, we have numerically verified the distinction between causal inference for the entire population and the sub-population consisting of units consistent with the observation.

\begin{figure}[t]
  \centering
  \begin{subfigure}{0.45\textwidth}
    \centering
    \includegraphics[height=3in,width=\linewidth]{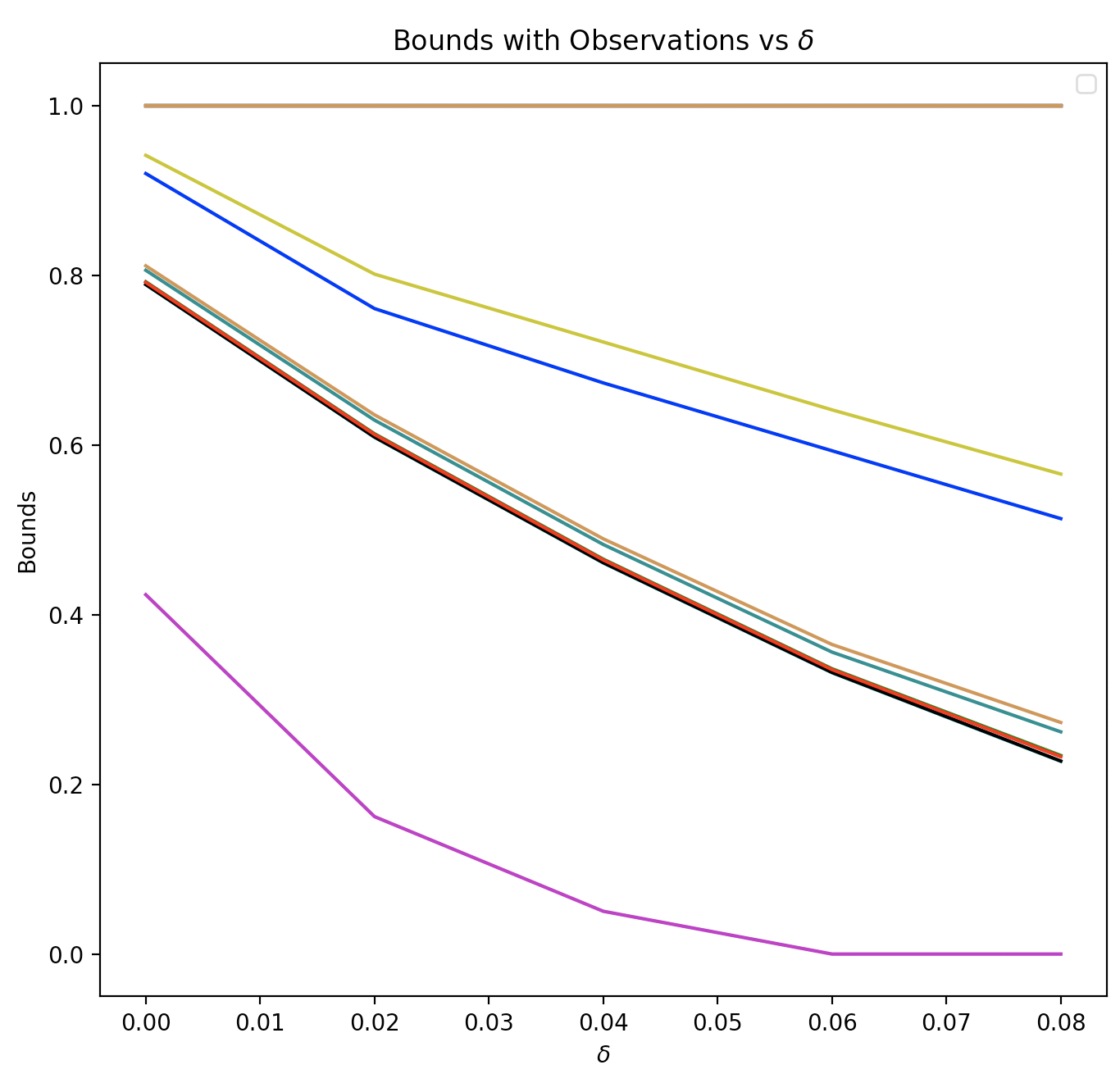}
    \caption{Lower and upper bounds with additional observations vs \(\delta\) for 10
      instances of \(\bar{p}\) in Example A. Lines of the same color
      denotes lower and upper bounds for the same instance.} 
    \label{fig:boundswithobservationsdeltaexA}  
  \end{subfigure}
  \hfill
  \begin{subfigure}{0.45\textwidth}
  \includegraphics[height=3in,width=\linewidth]{boundsnew.png}
  \caption{Lower and upper bounds without additional observations vs \(\delta\) for 10 instances of \(\bar{p}\) in Example A. Lines of the same color denotes lower and upper bounds for the same instance.}
  \label{fig:boundssdeltaexA}  
  \end{subfigure}
 \caption{Lower and upper bounds with and without additional observations of some variables in \(V_B\) vs \(\delta\) for 10 instances of \(\bar{p}\) in Example A (details in Appendix B). Note that our bounds have shifted after the observation.}
\end{figure}
\subsection{Greedy Heuristic}
\label{sec:greedy}
In this section, we propose a greedy heuristic (Algorithm~\ref{algo:greedy})
to approximately compute the bounds for  
% For
problems \emph{without} additional observations when even the pruned LPs are intractable, and the bounds cannot
be computed in closed form since $A \not \subseteq C(\mcal{Q})$. Note that the
bounds computed by our heuristic are guaranteed to contain the optimal
bounds, and therefore, the true query value. % conditions 
% in Section~\ref{sec:general,closedform} are not satisfied, we propose 
% Algorithm~\ref{algo:greedy} as a greedy heuristic to compute the bounds
% \(\alpha_L\) and \(\alpha_U\). 
This heuristic is motivated by the duals of
LPs \eqref{eq:prunedprimal} and~\eqref{eq:prunedPrimalDualUpper}
% which are given by:
that are defined as follows:
\begin{align}
  \label{eq:prunedDualLower}
  \begin{split}
  \alpha_L &=
  \begin{array}[t]{rll}
  \max_{\lambda} \ & \sum_{(v_A,v_B)\in\{0,1\}^{|A|} \times \{0,1\}^{|B|}} p_{v_B.v_A} \lambda_{v_B.v_A} \\
  \mbox{s.t.} \ & \sum_{v_A \in \{0,1\}^{|A|}} \lambda_{h(v_A).v_A} \leq c_h^L , & 
                                                                          \forall h \in H.
  \end{array}
\end{split}
\end{align}
\begin{align}
  \label{eq:prunedDualUpper}
  \begin{split}
  \alpha_U &=
  \begin{array}[t]{rll}
  \min_{\lambda} \ & \sum_{(v_A,v_B)\in\{0,1\}^{|A|} \times \{0,1\}^{|B|}} p_{v_B.v_A} \lambda_{v_B.v_A} \\
  \mbox{s.t.} \ & \sum_{v_A\in \{0,1\}^{|A|}} \lambda_{h(v_A).v_A} \geq c_h^U , & 
                                                                          \forall h \in H.
  \end{array}
\end{split}
\end{align}
We utilize the fact that in our numerical experiments, we observed that
for both dual LPs, there was always an 
optimal solution that only took values in the set \(\{-1,0,1\}\), and
the fact that in the symbolic bounds introduced by
\citet{balke94} (see, also \cite{bareinboimtransfer2017, pearl2009causality,
  sjolanderbounds2014, sachs2021general}) the probabilities in the input data were combined using
coefficients taking values in  \(\{-1,0,1\}\). In fact, we % believe that
expect
the following conjecture to be true.   

\begin{restatable}[Dual Integrality]{conjecture}{dualintegral}
  \label{conj:dualintegral}
 The dual LPs in (\ref{eq:prunedDualLower}) and (\ref{eq:prunedDualUpper}) have optimal solutions which only take values in $\{-1,0,1\}$.
\end{restatable}

\begin{table}[t]
\begin{center}
\begin{tabular}{||c|c|c|c|c||} 
 \hline
 \multicolumn{1}{|c|}{Graph}&
 \multicolumn{1}{|c|}{\shortstack{\small \% instances with 
 \\ \small $\alpha_L^G = \alpha_L$}}&
 \multicolumn{1}{|c|}{\shortstack{\small \% instances with 
\\ \small $\alpha_U^G = \alpha_U$}}&
\multicolumn{1}{|c|}{\shortstack{\small \% instances with 
\\ \small $\epsilon_L \leq 10\%$}}&
\multicolumn{1}{|c|}{\shortstack{\small \% instances with 
\\ \small $\epsilon_U \leq 10\%$}} \\ [0.5ex]

 \hline\hline
 Ex A & 100 & 100 & 100 &100\\
 \hline
 Ex B & 99 & 86 & 99 & 94\\
 \hline
 Ex C & 100 & 84 &100& 86\\
 \hline
 Ex F & 100 & 100 &100& 100\\
 \hline
 Ex G & 100 & 100 & 100 & 100\\
 \hline
\end{tabular}
\end{center}
\caption{
Quality of the bounds computed by the greedy algorithm over 100 instances of Examples~A, B, C, F and G (details in
    Appendix~\ref{causalinferenceprobs}). Here \(\alpha_L^G\) (resp. 
    \(\alpha_U^G\)) 
    denotes the lower (resp. upper) 
    bound % from the
    computed by the 
    greedy heuristic, \(\alpha_L\) (resp. 
    \(\alpha_U\)) 
    denotes the lower (resp. upper) 
    bound % from the
    computed by solving the LP, 
    $\epsilon_L =
    (\alpha_L-\alpha^G_L)/\alpha_L$ (resp. 
    $\epsilon_U = (\alpha^G_U - \alpha_U)/\alpha_U$) is a measure of
    suboptimality of the lower (resp. upper)  bound computed by the greedy heuristic.}
\label{tab:greedytable}
\end{table}

\begin{algorithm}[t]
\caption{Greedy Heuristic}

Let the permutation which sorts the conditional probabilities in descending order be $i_1,\ldots,i_{2^{|N|}}$.\;

\SetKwFunction{FLB}{GreedyLowerBound}
  \SetKwProg{Fn}{Function}{:}{}
  \Fn{\FLB{}}{
    Initialize $\lambda = -\ones$.\;
    
    \For{$j=1,..,2^{|N|}$}{

    \While{$\lambda$ is feasible}{ 

        $\lambda_{i_j} = \lambda_{i_j}+1$
    }
    }
    }

\SetKwFunction{FLB}{GreedyUpperBound}
  \SetKwProg{Fn}{Function}{:}{}
  \Fn{\FLB{}}{
    Initialize $\lambda = \ones$. \;
    
    \For{$j=1,..,2^{|N|}$}{

    \While{$\lambda$ is feasible}{ 

        $\lambda_{i_j} = \lambda_{i_j}-1$
    }
    }
    }
\label{algo:greedy}    
\end{algorithm}

We % benchmark the greedy heuristic
tested Algorithm~\ref{algo:greedy} on  
% by computing bounds for 
$100$ instances of
each of the examples in Appendix~\ref{causalinferenceprobs} % which do not
% satisfy the conditions in Section~\ref{sec:general,closedform}.
for which bounds were not available in closed form. The results reported
in Table
\ref{tab:greedytable} % reports the results
are 
for Examples A, B, C, F and G for which the LP
can be solved. 
% that the results of the 
We see that bounds from the
greedy heuristic 
matches the LP bounds in most instances for these problems.
Recall that one
can compute the optimal bounds for Examples~A, F and G only after pruning the LP.
In Table~\ref{tab:greedytable}, $\epsilon_L = 1-\frac{\alpha^G_L}{\alpha_L}$ and
$\epsilon_U = \frac{\alpha^G_U}{\alpha_U}-1$ denote the relative errors of
the lower and upper bounds, respectively. % Note that at least $86\%$ percent of
% instances have relative errors which are within 0.1 of the optimal
% solution.
We see that the lower bound is always within \(10\%\) of the true value,
whereas the upper bound is within \(10\%\) for at least \(86\%\) of the
cases. 
See Appendix for the empirical distribution function of errors. Furthermore, 
the greedy heuristic yields 
non-trivial bounds for Examples D and E, where the corresponding pruned LP
is too large to be solved to optimality. 
\section{Conclusion}
\label{sec:conclusion}
In this work, we compute bounds for the expected value of some outcome variables \(V_O\) if we intervene on variables \(V_I\), given the values of variables \(V_A\) are known, via linear programming. We show how to leverage structural properties of these LPs to significantly reduce their size. We also show how to construct these LPs efficiently. 
As a direct consequence of our results, bounds for causal queries can
be computed for graphs of much larger size. We show that there are
examples of causal inference problems for which bounds could be computed
only after the pruning we introduce. Our structural results also allow us
to characterize a set of causal inference problems for which the bounds
can be computed in closed form. This class includes as a special case
extensions of problems considered in the multiple confounded treatments 
literature~\citep{pmlr-v139-wang21c,ranganath2019multiple,  
  janzing2018,pmlr-v89-d-amour19a,tran2017implicit}. 
We show that bounds for queries containing additional observations about the unit can be computed by solving fractional
LPs~\citep{bitran1973linear}. These 
fractional LPs are special because the denominator is restricted to be
non-negative. This allows us to homogenize the problem into a LP with one
additional constraint, 
and extend the structural results % presented
% here will extend to the setting where the query contains an
% observation.
obtained for queries without additional observations. We also show the 
significant runtime improvement provided by our methods compared to benchmarks in numerical experiments and extend the results to the finite data setting. Finally, for causal inference without additional observations, we propose a very efficient greedy heuristic that produces very high quality bounds, and scales to problems that are several orders of magnitude larger than those for which the pruned LPs can be solved.
% We are currently considering the following extension. 
% The constraints in the  dual LPs are all packing constraints (for
% \(\alpha_L\)) or covering constraints (for \(\alpha_U\)); however, the
% variables are free. % However, if Conjecture~\ref{conj:dualintegral} is
% % true, the feasible set of the LP can be assumed to be bounded.
% We are investigating whether the variables can be bounded. % Then 
% % This could
% % potentially be used to compute
% Then one could design
% fast approximation algorithms~\citep{bienstock2006approximating}.

% %\nocite{langley00}
% \newpage

%\bibliographystyle{icml2022} % We choose the "plain" reference style
%\bibliography{sample} % Entries are in the refs.bib file

%%%%%%%%%%%%%%%%%%%%%%%%%%%%%%%%%%%%%%%%%%%%%%%%%%%%%%%%%%%%%%%%%%%%%%%%%%%%%%%
%%%%%%%%%%%%%%%%%%%%%%%%%%%%%%%%%%%%%%%%%%%%%%%%%%%%%%%%%%%%%%%%%%%%%%%%%%%%%%%
% APPENDIX
%%%%%%%%%%%%%%%%%%%%%%%%%%%%%%%%%%%%%%%%%%%%%%%%%%%%%%%%%%%%%%%%%%%%%%%%%%%%%%%
%%%%%%%%%%%%%%%%%%%%%%%%%%%%%%%%%%%%%%%%%%%%%%%%%%%%%%%%%%%%%%%%%%%%%%%%%%%%%%%
\newpage
\appendix
\onecolumn
\section{Basic results}
\label{app:critical}

\begin{lemma}[Partition of $R$]
\label{lemma:partition}
\(R = \cup_{h \in H} R_h\) is a partition of \(R\).
\end{lemma}
\begin{proof}
Since $R_h \subseteq R$ for all $h \in H$, it follows that \(\cup_{h \in
  H} R_h \subseteq R\). % Next, we show that \(R \subseteq \cup_{h \in H} R_h\)
% and that for any \(h_1, h_2 \in H\) such that \(h_1 \neq h_2\), \(R_{h_1}
% \cap R_{h_2} = \emptyset\).
Next, we show that $R \subseteq \cup_h R_h$. Fix $\bar{r} \in R$. Then $\bar{r} \in R_{h}$ for some hyperarc $h$ such that
$h(v_A) = F_B(v_A,
\bar{r})$ for all \(v_A \in \{0,1\}^{|A|}\). Thus, $R \subseteq \cup_h R_h$.

% satisfies \(r \in
% % R_h\) where  
% \ms{\begin{equation}
%   \label{eq:simpleHyperarcDef}
%   R_h = \left\{r \in R: F_B(V_A = v_A, r) = h(v_A), \forall v_A \in \{0,1\}^{\abs{A}}\right\}
% \end{equation}}
% \ms{we have \(R \subseteq \cup_h R_h\).
Next, suppose there exist $h_1 \neq h_2$ such that 
% Suppose there exists \(r \in R\) and \(h_1, h_2 \in H\) such that \(h_1
% \neq h_2\) and \(r \in R_{h_1} \cap R_{h_2}\). Then we have}
$R_{h_1} \cap R_{h_2} \neq \emptyset$. Then for all $r \in R_{h_1} \cap
R_{h_2}$, and all $v_A \in \{0,1\}^{\abs{A}}$, we have  $h_i(v_A) = F_B(v_A, r)$, $i = 1,2$. Thus, it follows that $h_1(v_A) = h_2(v_A)$ for all $v_A
\in \{0,1\}^{\abs{A}}$. A contradiction.  % and 
% \[
%   h_1(v_A) = F_{B}(V_A = v_A, \bar{r}) = h_1(v_A) = h_2(v_A).
% \]
% \ms{since \(r\) indexes functions mapping \(V_A\) to \(V_B\).
% This is a contradiction since \(h_1 \neq h_2\). Hence, \(R_{h_1} \cap
% R_{h_2} = \emptyset\). 
\end{proof}

\begin{lemma}[Critical Intervention Variables]
  \label{lem:critical}
   Let $G^{do(V_I = q_I)}$ denote the mutilated graph after intervention
  \(do(V_I = q_I)\), i.e. variables $V_{I}$ no longer have any incoming arcs, and let \(V_{C(\mcal{Q})}\) denote the set of variables in $V_A \cup V_B$ that have a path to some variable in $V_O$ in  $G^{do(V_I = q_I)}$.Then 
  \[\mathbb{P}(V_O(V_{I} = q_{I})=q_O|V_A = q_A) = \mathbb{P}(V_O(V_{I \cap C(\mcal{Q})} = q_{I \cap C(\mcal{Q})})=q_O|V_A = q_A)\] 
\end{lemma}
\begin{proof}
By definition, for every \(V \in V_{I \setminus C(\mcal{Q})}\), there is no directed path from \(V\) to a variable in \(V_O\) in \(G^{do(V_I = q_I)}\). It thus follows that they do not influence the value of \(V_O\) in the intervention. It thus follows that \(\mathbb{P}(V_O(V_{I} = q_{I})=q_O|V_A = q_A) = \mathbb{P}(V_O(V_{I \cap C(\mcal{Q})} = q_{I \cap C(\mcal{Q})}) = q_O|V_A = q_A)\).
\end{proof}

\section{Examples of Causal Inference Problems}
\label{causalinferenceprobs}
In this section, we report the causal graph structure and the data
generation process for the \(5\) examples in Table~\ref{tab:sizetable}.

\begin{figure}[htp!]
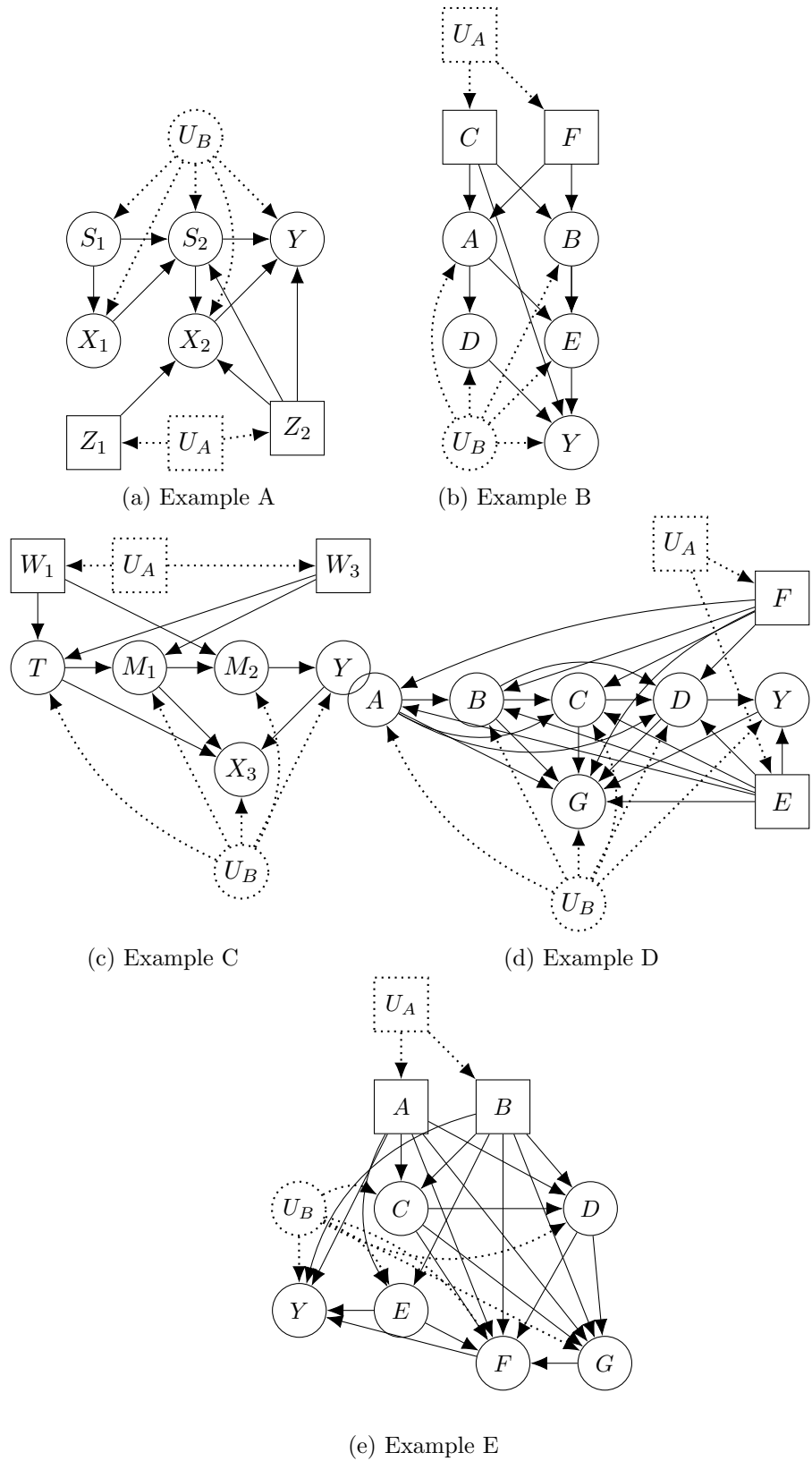


  \begin{subfigure}{.3\textwidth}
    \centering
    \begin{influence-diagram}
      \node (S1) [] {$S_1$};
      \node (S2) [right = of S1] {$S_2$};
      \node (X1) [below = of S1] {$X_1$};
      \node (Y) [right = of S2] {$Y$};
      \node (X2) [below = of S2] {$X_2$};
      \node (Z1) [below = of X1, decision] {$Z_1$};
      \node (Z2) [below = 2 of Y, decision] {$Z_2$};
      \node (UA) [below = of X2, information, decision] {$U_A$};
      \node (UB) [above = of S2, information] {$U_B$};
      
      \edge {Z1} {X2};
      \edge {Z2} {Y};
      \edge {Z2} {X2};
      \edge {Z2} {S2};
      \edge {S1} {S2};
      \edge {S2} {Y};
      \edge {S2} {X2};
      \edge {S1} {X1};
      \edge {X1} {S2};
      \edge {X2} {Y};
      \edge[information] {UA} {Z1};
      \edge[information] {UA} {Z2};
      \edge[information] {UB} {S1};
      \edge[information] {UB} {X1};
      \edge[information] {UB} {S2};
      \edge[information] {UB} {Y};

      \path (UB) edge[information, ->, bend left=30] (X2);
    \end{influence-diagram}
    \caption{Example A}
    \label{fig:ExampleA}
  \end{subfigure}
  \begin{subfigure}{.3\textwidth}
    \centering
    \begin{influence-diagram}
      \node (A) [] {$A$};
      \node (B) [right = of A] {$B$};
      \node (C) [above = of A, decision] {$C$};
      \node (D) [below = of A] {$D$};
      \node (E) [below = of B] {$E$};
      \node (F) [above = of B, decision] {$F$};
      \node (Y) [below = of E] {$Y$};
      \node (UA) [above = of C, information, decision] {$U_A$};
      \node (UB) [below = of D, information] {$U_B$};
      \edge {F} {B};
      \edge {F} {A};
      \edge {C} {A};
      \edge {C} {B};
      \edge {A} {D};
      \edge {A} {E};
      \edge {B} {E};
      \edge {E} {Y};
      \edge {B} {E};
      \edge {D} {Y};
      \edge {C} {Y}
      \edge[information] {UA} {F};
      \edge[information] {UA} {C};
      \edge[information] {UB} {D};
      \edge[information] {UB} {E};
      \edge[information] {UB} {B};
      \edge[information] {UB} {Y};
      \path (UB) edge[information, ->, bend left=30] (A);
    \end{influence-diagram}
    \caption{Example B}
    \label{fig:ExampleB}
  \end{subfigure}
  \newline
  \begin{subfigure}{.3\textwidth}
    \centering
    \begin{influence-diagram}
      \node (T) [] {$T$};
      \node (M1) [right = of A] {$M_1$};
      \node (M2) [right = of M1] {$M_2$};
      \node (Y) [right = of M2] {$Y$};
      \node (W1) [above = of T, decision] {$W_1$};
      \node (W3) [above = of Y, decision] {$W_3$};
      \node (X3) [below = of M2] {$X_3$};
      \node (UA) [above = of B, information, decision] {$U_A$};
      \node (UB) [below = of X3, information] {$U_B$};
      \edge {T} {M1};
      \edge {M1} {M2};
      \edge {M1} {X3};
      \edge {M2} {Y};

      \edge {W1} {T};
      \edge {W1} {M2};
      \edge {W3} {T};
      \edge {W3} {M1};
      \edge {T} {X3};
      \edge {Y} {X3};
      \edge[information] {UA} {W1};
      \edge[information] {UA} {W3};
      \edge[information] {UB} {M1};
      \edge[information] {UB} {Y};
      \edge[information] {UB} {X3};
      \path (UB) edge[information, ->, bend left=20] (T);
      \path (UB) edge[information, ->, bend right=30] (M2);
    \end{influence-diagram}
    \caption{Example C}
    \label{fig:ExampleC}
  \end{subfigure}
  \begin{subfigure}{.5\textwidth}
    \centering
    \begin{influence-diagram}
      \node (T) [] {$A$};
      \node (M1) [right = of T] {$B$};
      \node (M2) [right = of M1] {$C$};
      \node (M3) [right = of M2] {$D$};
      \node (Y) [right = of M3] {$Y$};
      \node (W1) [below = of Y, decision] {$E$};
      \node (W3) [above = of Y, decision] {$F$};
      \node (X3) [below = of M2] {$G$};
      \node (UA) [above = 1.5 of M3, information, decision] {$U_A$};
      \node (UB) [below = of X3, information] {$U_B$};
      \edge {T} {M1};
      \edge {M1} {M2};

      \edge {M1} {X3};
      \edge {M2} {M3};
      \edge {M2} {X3};
      \edge {M3} {Y};
      \edge {M3} {X3};

      \edge {W1} {T};
      \edge {W1} {M2};
      \edge {W1} {M3};
      \edge {W1} {M1};
      \edge {W1} {Y};
      \edge {W1} {X3};

      \edge {W3} {M1};
      \edge {W3} {M2};
      \edge {W3} {M3};
      \edge {T} {X3};
      
      \edge {Y} {X3};

      \edge[information] {UA} {W1};
      \edge[information] {UA} {W3};
      \edge[information] {UB} {M1};
      \edge[information] {UB} {Y};
      \edge[information] {UB} {X3};
    
      \edge[information] {UB} {M3};
      \path (UB) edge[information, ->, bend left=20] (T);
      \path (UB) edge[information, ->, bend right=30] (M2);
      \path (W3) edge[->, bend right=10] (T);
      \path (W3) edge[->, bend right=20] (X3);

      \path (M1) edge[->, bend left=30] (M3);

      \path (T) edge[->, bend right=30] (M2);
     \path (T) edge[->, bend right=30] (M3);
      %\path (W1) edge[->, bend right=10] (X2);
      
    \end{influence-diagram}
    \caption{Example D}
    \label{fig:ExampleD}
  \end{subfigure}
  \centering
\small
  \begin{subfigure}{0.3 \textwidth}
    \centering
    \begin{influence-diagram}
      \node (A) [decision] {$A$};
      \node (B) [right = of A, decision] {$B$};
      \node (C) [below = of A] {$C$};
      \node (D) [right = 2 of D] {$D$};
      \node (E) [below = of C] {$E$};
      \node (F) [right = of E, below=3 of B] {$F$};
      \node (G) [right = of F] {$G$};
      \node (Y) [left = of E] {$Y$};
      \node (UA) [above = of A, information, decision] {$U_A$};
      \node (UB) [left = of C, information] {$U_B$};
      
      \edge {C} {F};
      \edge {C} {D};
      
      \edge {D} {F};
      
      \edge {D} {G};
      \edge {E} {F};
      \edge {G} {F};
      \edge {C} {G};
  
      \edge {A} {D};
      
      \edge {A} {F};
      \edge {A} {C};
      \edge {A} {G};
      \edge {B} {E};
      \edge {B} {F};
      \edge {B} {C};
      \edge {B} {D};
      \edge {B} {G};
    
      \edge {E} {Y};
      \edge {F} {Y};
     \edge {A} {Y};
      \edge[information] {UA} {A};
      \edge[information] {UA} {B};
      
      \edge[information] {UB} {G};
      \edge[information] {UB} {Y};
      \path (UB) edge[information, ->, bend left=30] (C);
      \path (UB) edge[information, ->, bend right=30] (D);
  
      \path (UB) edge[information, ->, bend left=20] (E);
      \path (UB) edge[information, ->, bend left=20] (F);
      \path (A) edge[->, bend right=30] (E);
    \path (B) edge[->, bend right=30] (Y);
      
    \end{influence-diagram}
    \caption{Example E}
    \label{fig:ExampleE}
  \end{subfigure}
  \caption{Examples of Causal Inference Problems}
  \label{fig:ExampleGraphs}
\end{figure}

\begin{figure}[htp!]
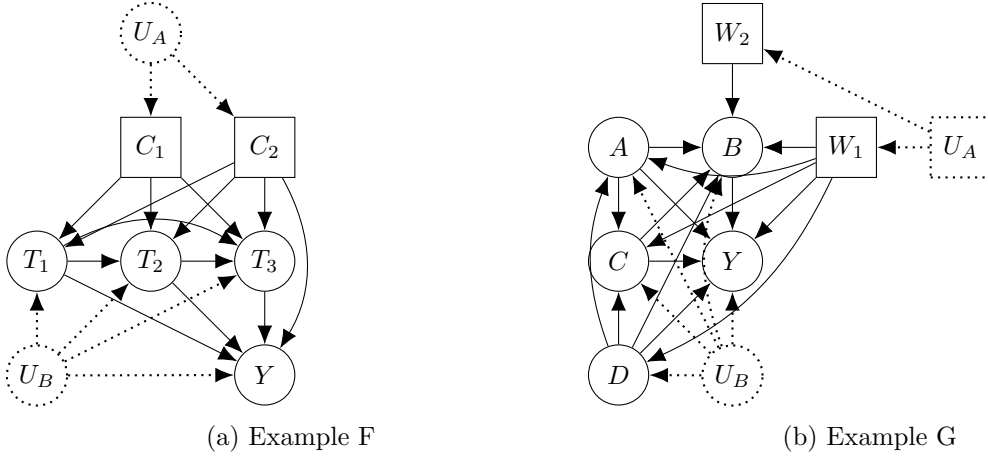

\small

  \begin{subfigure}{.5\textwidth}
    %\centering
    \begin{influence-diagram}
    \node (T1) [] {$T_{1}$};
    \node (T2) [right = of T1] {$T_{2}$};
    \node (T3) [right = of T2] {$T_{3}$};

    \node (Y) [below = of T3] {$Y$};
    \node (C1) [above = of T2, decision] {$C_1$};
    \node (C2) [above = of T3, decision] {$C_2$};
    \node (UB) [below = of T1, information] {$U_B$};
    \node (UA) [above = of C1, information] {$U_A$};
    
    \edge {T1} {Y};
    \edge {T2} {Y};
    \edge {T3} {Y};

    \edge {C1} {T1};
    \edge {C1} {T2};
    \edge {C1} {T3};

    \edge {C2} {T1};
    \edge {C2} {T2};
    \edge {C2} {T3};

   \edge{T1} {T2};
    \edge {T2} {T3};
    
    \edge[information] {UB} {T1};
    \edge[information] {UB} {T2};
    \edge[information] {UB} {T3};
    
    \edge[information] {UB} {Y};
    \edge[information] {UA} {C1};
    \edge[information] {UA} {C2};

    \path (T1) edge[->, bend left=30] (T3);
 
    \path (C2) edge[->, bend left=30] (Y);

  \end{influence-diagram}
    \caption{Example F}
    \label{fig:ExampleF}
  \end{subfigure}
  %\newline
  \small
  \begin{subfigure}{.5\textwidth}
   % \centering
    \begin{influence-diagram}
      \node (A) [] {$A$};
      \node (B) [right = of A] {$B$};
      \node (C) [below = of A] {$C$};
      \node (Y) [right = of C] {$Y$};

     \node (W1) [right = of B, decision] {$W_1$};
      \node (W2) [above = of B, decision] {$W_2$};
      \node (D) [below = of C] {$D$};
      \node (UA) [right = of W1, information, decision] {$U_A$};
      \node (UB) [below = of Y, information] {$U_B$};
      \edge {A} {B};
      \edge {C} {B};
      \edge {B} {Y};
      \edge {D} {Y};
      \edge {D} {B};
      \edge {D} {C};
      \edge {C} {Y};
      \edge {A} {Y};
      \edge {A} {C};
      \edge {W1} {B};
      \edge {W1} {Y};
      \edge {W1} {C};
      \edge {W2} {B};

      \edge[information] {UA} {W1};
      \edge[information] {UA} {W2};
      \edge[information] {UB} {A};
      \edge[information] {UB} {Y};
      \edge[information] {UB} {C};
      \edge[information] {UB} {D};
      \path (UB) edge[information, ->, bend left=20] (B);
      \path (D) edge[->, bend left=20] (A);
      \path (W1) edge[->, bend left=20] (A);
      \path (W1) edge[->, bend left=20] (D);

    \end{influence-diagram}
    \caption{Example G}
    \label{fig:ExampleG}
  \end{subfigure}
  
  \caption{Additional Causal Inference Problems}
  \label{fig:AdditionalExampleGraphs}
\end{figure}

\subsection*{Example A}
The causal graph for this example is display in
Figure~\ref{fig:ExampleA}. The query is: $P(Y(X_2=1)=1|Z_1=1,Z_2=1)$,
and data generating process used to generate the input data is given by
\begin{eqnarray*}
U_A &\sim& N(0,1)\\
U_B &\sim& N(0,1)\\
Z_1 &\sim& \text{Bernoulli}(\text{logit}^{-1}(U_A))\\
Z_2 &\sim& \text{Bernoulli}(\text{logit}^{-1}(U_A))\\
  S_1 &\sim& \text{Bernoulli}(\text{logit}^{-1}(U_B))\\
X_1 &\sim& \text{Bernoulli}(\text{logit}^{-1}(U_B+S_1))\\
S_2 &\sim& \text{Bernoulli}(\text{logit}^{-1}(S_1+U_B+X_1+Z_2))\\
X_2 &\sim& \text{Bernoulli}(\text{logit}^{-1}(S_2+U_B+Z_1+Z_2))\\
Y &\sim& \text{Bernoulli}(\text{logit}^{-1}(U_B + S_2+X_2+Z_2 ))
\end{eqnarray*}
After sampling $U_A, U_B,Z_1, Z_2$ we compute \\
\(
\mathbb{P}(Y, X_2,S_2,X_1,S_1|Z_2, Z_1) = \mathbb{P}(Y|U_B,S_2,X_2,Z_2)\mathbb{P}(X_2|S_2,Z_1,Z_2,U_B)\mathbb{P}(S_2|Z_2,U_B,S_1,X_1)\mathbb{P}(X_1|S_1,U_B)\\
\mathbb{P}(S_1|U_B)
\)\\
that gives input distribution. In Section \ref{sec:greedy}, we use the observation \(\{Z_1=1, Z_2=1, Y=1\}\).

\subsection*{Example B}
The causal graph for this example is in Figure~\ref{fig:ExampleB} and the query is: $P(Y(A=1,B=1)=1|C=1,F=1)$.
The data generating process used to generate the input information is as follows:
\begin{eqnarray*}
U_A &\sim& N(0,1)\\
U_B &\sim& N(0,1)\\
C &\sim& \text{Bernoulli}(\text{logit}^{-1}(U_A))\\
F &\sim& \text{Bernoulli}(\text{logit}^{-1}(U_A))\\
A &\sim& \text{Bernoulli}(\text{logit}^{-1}(C+F+U_B))\\
B &\sim& \text{Bernoulli}(\text{logit}^{-1}(C+F+U_B))\\
D &\sim& \text{Bernoulli}(\text{logit}^{-1}(A+U_B))\\
E &\sim& \text{Bernoulli}(\text{logit}^{-1}(A+B+U_B))\\
Y &\sim& \text{Bernoulli}(\text{logit}^{-1}(U_B+D+C+E))
\end{eqnarray*}
After sampling $U_A, U_B,C, F$ we compute
\[
\mathbb{P}(A,B,D,E,Y|C,F) =
\mathbb{P}(Y|U_B,D,C,E)\mathbb{P}(E|A,B,U_B)\mathbb{P}(D|A,U_B)\mathbb{P}(B|C,F,U_B)\mathbb{P}(A|C,F,U_B).
\]
that gives the input distribution. In Section \ref{sec:greedy}, we use the observation \(\{C=1, F=1, Y=1\}\).

\subsection*{Example C}
The causal graph for this example is in
Figure~\ref{fig:ExampleC} and the query is: $P(Y(M_1=1)=1|W_1=1,W_3=1)$. The data generating process used to generate the input information is as follows:

\begin{eqnarray*}
U_A &\sim& N(0,1)\\
U_B &\sim& N(0,1)\\
W_1 &\sim& \text{Bernoulli}(\text{logit}^{-1}(U_A))\\
W_3 &\sim& \text{Bernoulli}(\text{logit}^{-1}(U_A))\\
T &\sim& \text{Bernoulli}(\text{logit}^{-1}(W_1+W_3+U_B))\\
M_1 &\sim& \text{Bernoulli}(\text{logit}^{-1}(T+W_3+U_B))\\
M_2 &\sim& \text{Bernoulli}(\text{logit}^{-1}(M_1+W_1+U_B))\\
Y &\sim& \text{Bernoulli}(\text{logit}^{-1}(M_2+U_B))\\
X_3 &\sim& \text{Bernoulli}(\text{logit}^{-1}(U_B+Y+M_1+T))\\
\end{eqnarray*}
After sampling $U_A, U_B,W_1, W_3$ we compute

$\mathbb{P}(T,M_1,M_2,X_3,Y|W_1,W_3)=\mathbb{P}(X_3|U_B,Y,M_1,T)\mathbb{P}(Y|M_2,U_B)\mathbb{P}(M_2|M_1,W_1,U_B)\\\mathbb{P}(M_1|T,W_3,U_B)\mathbb{P}(T|W_1,W_3,U_B)$

that gives the input distribution. In Section \ref{sec:greedy}, we use the observation \(\{W_1=1, W_3=1, Y=1\}\).

\subsection*{Example D}
The causal graph for this example is in
Figure~\ref{fig:ExampleD} and the query is: $P(Y(D=1)=1|E=1,F=1)$. The data generating process used to generate the input information is as follows:
\begin{eqnarray*}
U_A &\sim& N(0,1)\\
U_B &\sim& N(0,1)\\
F &\sim& \text{Bernoulli}(\text{logit}^{-1}(U_A))\\
E &\sim& \text{Bernoulli}(\text{logit}^{-1}(U_A))\\
A &\sim& \text{Bernoulli}(\text{logit}^{-1}(F+E+U_B))\\
B &\sim& \text{Bernoulli}(\text{logit}^{-1}(E+F+A+U_B))\\
C &\sim& \text{Bernoulli}(\text{logit}^{-1}(B+E+F+A+U_B))\\
D &\sim& \text{Bernoulli}(\text{logit}^{-1}(B+E+F+A+C+U_B))\\
Y &\sim& \text{Bernoulli}(\text{logit}^{-1}(E+D+U_B))\\
G &\sim& \text{Bernoulli}(\text{logit}^{-1}(A+B+C+D+Y+E+F+U_B))\\
\end{eqnarray*}
After sampling $U_A, U_B,E, F$ we compute

$\mathbb{P}(G,Y,D,C,B,A|E,F)=\mathbb{P}(G|U_B,A,B,C,D,Y,E,F)\mathbb{P}(Y|E,D,U_B)\mathbb{P}(D|B,E,F,A,C,U_B)\\
\mathbb{P}(C|B,E,F,A,U_B)\mathbb{P}(B|E,F,A,U_B)\mathbb{P}(A|E,F,U_B)$

that gives the input distribution. In Section \ref{sec:greedy}, we use the observation \(\{E=1, F=1, Y=1\}\).

\subsection*{Example E}
The causal graph for this example is in
Figure~\ref{fig:AdditionalExampleGraphs} and the query is: $P(Y(C=1,F=1)=1|A=1,B=1)$. The data generating process used to generate the input information is as follows:
\begin{eqnarray*}
U_A &\sim& N(0,1)\\
U_B &\sim& N(0,1)\\
A &\sim& \text{Bernoulli}(\text{logit}^{-1}(U_A))\\
B &\sim& \text{Bernoulli}(\text{logit}^{-1}(U_A))\\
C &\sim& \text{Bernoulli}(\text{logit}^{-1}(A+B+U_B))\\
D &\sim& \text{Bernoulli}(\text{logit}^{-1}(A+C+B+U_B))\\
E &\sim& \text{Bernoulli}(\text{logit}^{-1}(A+B+U_B))\\
F &\sim& \text{Bernoulli}(\text{logit}^{-1}(A+C+B+D+E+G+U_B))\\
G &\sim& \text{Bernoulli}(\text{logit}^{-1}(U_B+A+B+C+D))\\
Y &\sim& \text{Bernoulli}(\text{logit}^{-1}(U_B+A+E+B+F))\\
\end{eqnarray*}
After sampling $U_A, U_B,A, B$ we compute\\
\(\mathbb{P}(C,D,E,F,G,Y|A,B)=\mathbb{P}(C|A,B,U_B)\mathbb{P}(D|A,C,B,U_B)\mathbb{P}(E|A,B,U_B)\mathbb{P}(F|A,C,B,D,E,G,U_B)\\
\mathbb{P}(G|A,B,C,D,U_B)\mathbb{P}(Y|U_B,A,E,B,F)
\)\\
that gives the input distribution. In Section \ref{sec:greedy}, we use the observation \(\{A=1, B=1, Y=1\}\).

\subsection*{Example F}
The causal graph for this example is in
Figure~\ref{fig:ExampleF} and the query is: $P(Y(T_1=1,T_2=1,T_3=1)=1|C_1=1,C_2=1)$. The data generating process used to generate the input information is as follows:
\begin{eqnarray*}
U_A &\sim& N(0,1)\\
U_B &\sim& N(0,1)\\
C_1 &\sim& \text{Bernoulli}(\text{logit}^{-1}(U_A))\\
C_2 &\sim& \text{Bernoulli}(\text{logit}^{-1}(U_A))\\
T_1 &\sim& \text{Bernoulli}(\text{logit}^{-1}(C_1+C_2+U_B))\\
T_2 &\sim& \text{Bernoulli}(\text{logit}^{-1}(C_1+C_2+U_B))\\
T_3 &\sim& \text{Bernoulli}(\text{logit}^{-1}(C_1+C_2+U_B))\\
Y &\sim& \text{Bernoulli}(\text{logit}^{-1}(U_B+T_1+T_2+T_3+C_2))\\
\end{eqnarray*}
After sampling $U_A, U_B,C_1, C_2$ we compute\\
\(\mathbb{P}(T_1,T_2,T_3|C_1,C_2)=\mathbb{P}(T_1|C_1,C_2,U_B)\mathbb{P}(T_2|C_1,C_2,U_B)\mathbb{P}(T_3|C_1,C_2,U_B)\mathbb{P}(Y|T_1,T_2,T_3,,C_2,U_B)\)
that gives the input distribution. In Section \ref{sec:greedy}, we use the observation \(\{C_1=1, C_2=1, Y=1\}\).

\subsection*{Example G}
The causal graph for this example is in
Figure~\ref{fig:ExampleG} and the query is: $P(Y(B=1)=1|W_1=1,W_2=1)$. The data generating process used to generate the input information is as follows:
\begin{eqnarray*}
U_A &\sim& N(0,1)\\
U_B &\sim& N(0,1)\\
W_1 &\sim& \text{Bernoulli}(\text{logit}^{-1}(U_A))\\
W_2 &\sim& \text{Bernoulli}(\text{logit}^{-1}(U_A))\\
D &\sim& \text{Bernoulli}(\text{logit}^{-1}(W_1+U_B))\\
A &\sim& \text{Bernoulli}(\text{logit}^{-1}(W_1+D+U_B))\\
C &\sim& \text{Bernoulli}(\text{logit}^{-1}(A+D+W_1+U_B))\\
B &\sim& \text{Bernoulli}(\text{logit}^{-1}(D+C+A+W_1+W_2+U_B))\\
Y &\sim& \text{Bernoulli}(\text{logit}^{-1}(U_B+D+C+A+B+W_1))\\
\end{eqnarray*}
After sampling $U_A, U_B,W_1, W_2$ we compute\\
\(\mathbb{P}(D,A,C,B,Y|W_1,W_2)=\mathbb{P}(D|W_1,U_B)\mathbb{P}(A|W_1,D,U_B)\mathbb{P}(C|A,D,W_1,U_B)\mathbb{P}(B|D,C,A,W_1,W_2,U_B)\)\\\(\mathbb{P}(Y|D,C,A,B,W_1,U_B)
\)\\
that gives the input distribution. In Section \ref{sec:greedy}, we use the observation \(\{W_1=1, W_2=1, Y=1\}\).
% You can have as much text here as you want. The main body must be at most $8$ pages long.
% For the final version, one more page can be added.
% If you want, you can use an appendix like this one, even using the one-column format.
%%%%%%%%%%%%%%%%%%%%%%%%%%%%%%%%%%%%%%%%%%%%%%%%%%%%%%%%%%%%%%%%%%%%%%%%%%%%%%%
%%%%%%%%%%%%%%%%%%%%%%%%%%%%%%%%%%%%%%%%%%%%%%%%%%%%%%%%%%%%%%%%%%%%%%%%%%%%%%%

%\CCprobability*

%\CChyperarc*

%\SCprobability*

%\SChyperarc*

%\thmCCprobabilityRN*

%\thmCChRN*

%\SCprobability*

%\thmSCprobability*

%\thmSCh*

%\thmspecialFracProblem*
\newpage
\section{Empirical CDF for Error of Greedy Heuristic}
\label{ecdf}

\begin{figure}[htp!]
    \centering
    \begin{subfigure}{.33\textwidth}
    \includegraphics[height=1.6in, width=\textwidth]{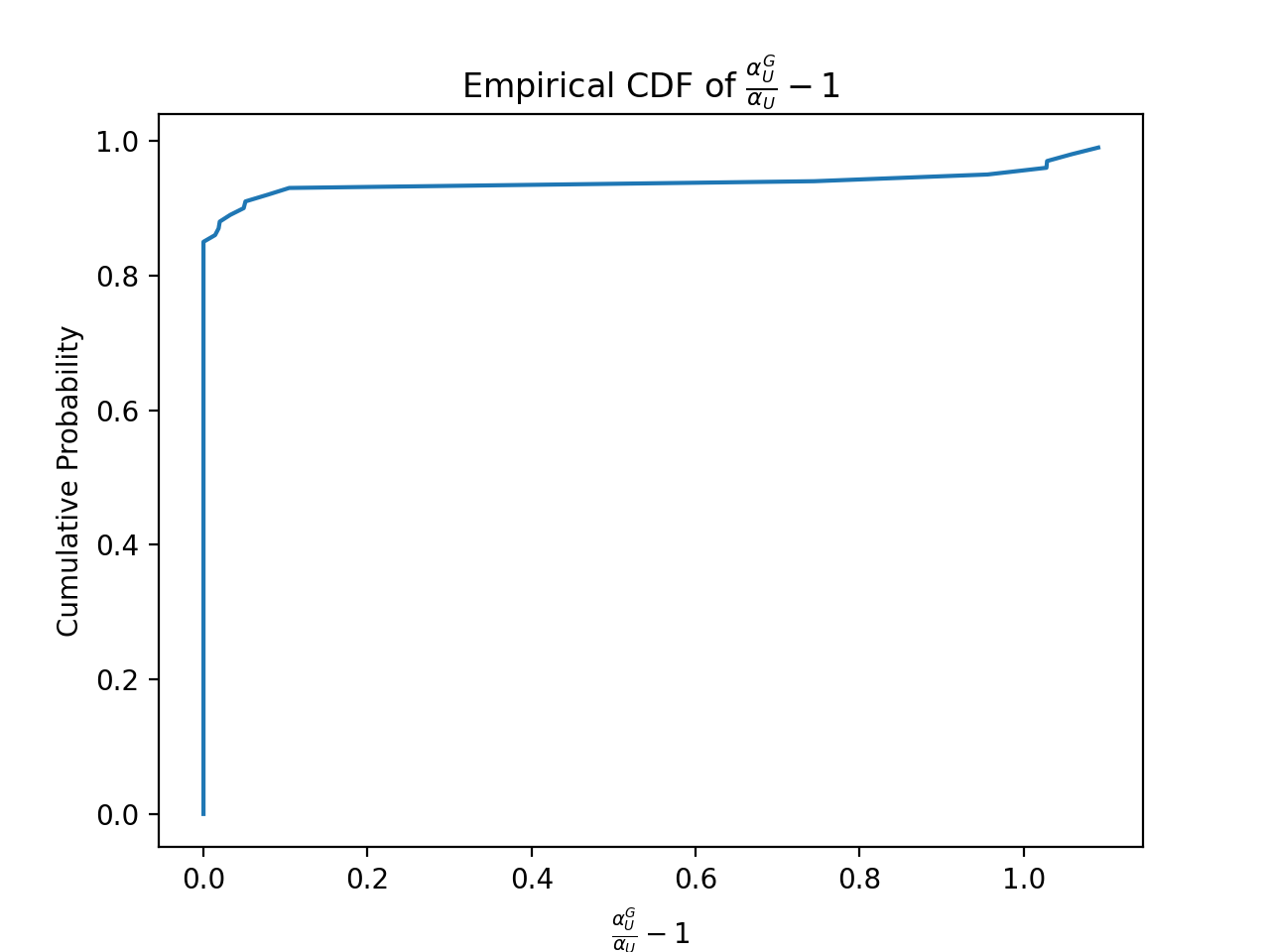}
    \caption{Empirical CDF of the Relative Error of \(\alpha_U\) for Example B}
    \end{subfigure}
    \begin{subfigure}{.33\textwidth}
    \includegraphics[height=1.6in, width=\textwidth]{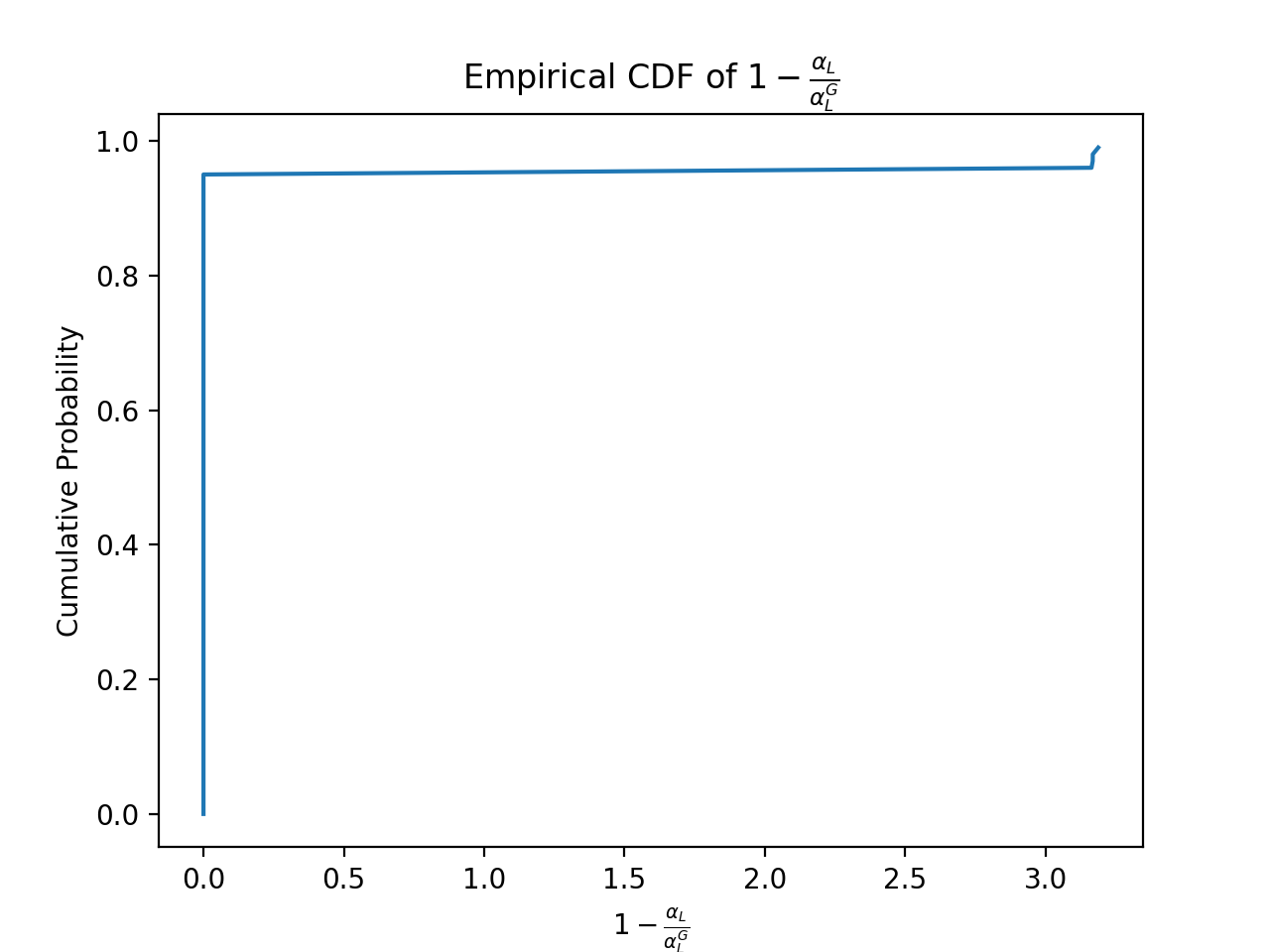}
    \caption{Empirical CDF for Relative Error of
      \(\alpha_L\) for Example B}
  \end{subfigure}
  \centering
    \begin{subfigure}{.33\textwidth}
    \includegraphics[height=1.6in,width=\textwidth]{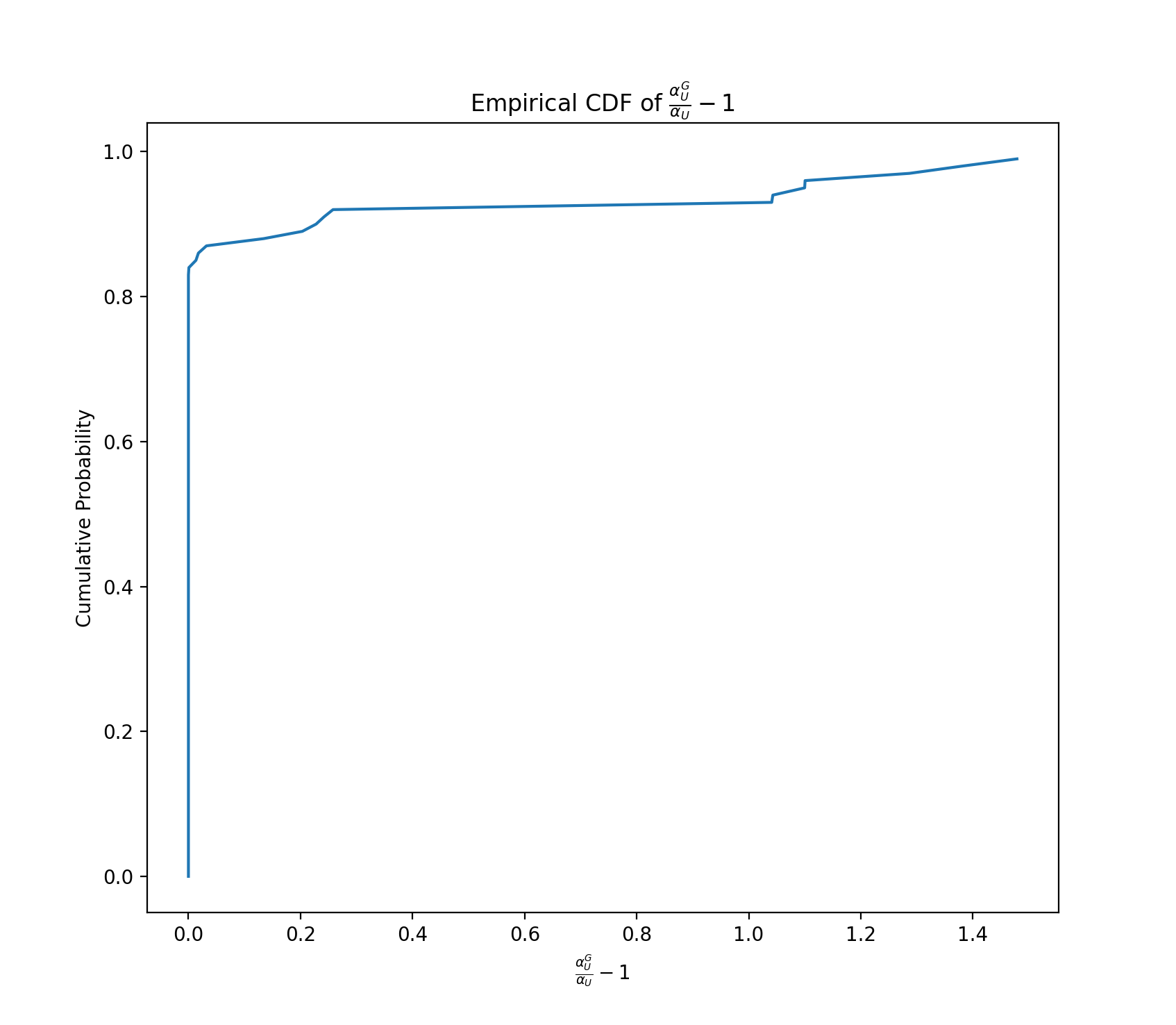}
    \caption{Empirical CDF of the Relative Error of \(\alpha_U\) for Example C}
    \end{subfigure}
 
\caption{Empirical Distribution Functions of Errors for Examples}
\end{figure}

% Acknowledgements should go at the end, before appendices and references

% Manual newpage inserted to improve layout of sample file - not
% needed in general before appendices/bibliography.

% Note: in this sample, the section number is hard-coded in. Following
% proper LaTeX conventions, it should properly be coded as a reference:

%In this appendix we prove the following theorem from
%Section~\ref{sec:textree-generalization}:
\newpage
\bibliography{main}

\begin{thebibliography}{31}
\providecommand{\natexlab}[1]{#1}
\providecommand{\url}[1]{\texttt{#1}}
\expandafter\ifx\csname urlstyle\endcsname\relax
  \providecommand{\doi}[1]{doi: #1}\else
  \providecommand{\doi}{doi: \begingroup \urlstyle{rm}\Url}\fi

\bibitem[Andersen(2013)]{andersen2013expect}
Holly Andersen.
\newblock When to expect violations of causal faithfulness and why it matters.
\newblock \emph{Philosophy of Science}, 80\penalty0 (5):\penalty0 672--683,
  2013.

\bibitem[Balke and Pearl(1994)]{balke94}
A.~Balke and J.~Pearl.
\newblock Counterfactual probabilities: Computational methods, bounds and
  applications.
\newblock In \emph{Uncertainty Proceedings 1994}, pages 46--54. Elsevier, 1994.

\bibitem[Bertsimas and Tsitsiklis(1997)]{linearprogbertsimas}
Dimitris Bertsimas and John Tsitsiklis.
\newblock \emph{Introduction to Linear Optimization}.
\newblock Athena Scientific, 1st edition, 1997.
\newblock ISBN 1886529191.

\bibitem[Bitran and Novaes(1973)]{bitran1973linear}
G.~R. Bitran and A.~G. Novaes.
\newblock Linear programming with a fractional objective function.
\newblock \emph{Operations Research}, 21\penalty0 (1):\penalty0 22--29, 1973.

\bibitem[Charnes and Cooper(1962)]{charnes1962programming}
A.~Charnes and W.~W. Cooper.
\newblock Programming with linear fractional functionals.
\newblock \emph{Naval Research logistics quarterly}, 9\penalty0 (3-4):\penalty0
  181--186, 1962.

\bibitem[D'Amour(2019)]{pmlr-v89-d-amour19a}
A.~D'Amour.
\newblock On multi-cause approaches to causal inference with unobserved
  counfounding: Two cautionary failure cases and a promising alternative.
\newblock In Kamalika Chaudhuri and Masashi Sugiyama, editors,
  \emph{Proceedings of the Twenty-Second International Conference on Artificial
  Intelligence and Statistics}, volume~89 of \emph{Proceedings of Machine
  Learning Research}, pages 3478--3486. PMLR, 16--18 Apr 2019.
\newblock URL \url{https://proceedings.mlr.press/v89/d-amour19a.html}.

\bibitem[Duarte et~al.(2021)Duarte, Finkelstein, Knox, Mummolo, and
  Shpitser]{duarte2021}
G.~Duarte, N.~Finkelstein, D.~Knox, J.~Mummolo, and I.~Shpitser.
\newblock An automated approach to causal inference in discrete settings.
\newblock \emph{arXiv preprint arXiv:2109.13471}, 2021.

\bibitem[Evans(2012)]{evans2012}
R.~J. Evans.
\newblock Graphical methods for inequality constraints in marginalized {DAGs}.
\newblock In \emph{2012 IEEE International Workshop on Machine Learning for
  Signal Processing}, pages 1--6, 2012.
\newblock \doi{10.1109/MLSP.2012.6349796}.

\bibitem[Finkelstein and Shpitser(2020)]{finkelstein20a}
N.~Finkelstein and I.~Shpitser.
\newblock Deriving bounds and inequality constraints using logical relations
  among counterfactuals.
\newblock In \emph{Conference on Uncertainty in Artificial Intelligence}, pages
  1348--1357. PMLR, 2020.

\bibitem[Finkelstein et~al.(2021)Finkelstein, Adams, Saria, and
  Shpitser]{finkelstein21b}
N.~Finkelstein, R.~Adams, S.~Saria, and I.~Shpitser.
\newblock Partial identifiability in discrete data with measurement error.
\newblock In Cassio de~Campos and Marloes~H. Maathuis, editors,
  \emph{Proceedings of the Thirty-Seventh Conference on Uncertainty in
  Artificial Intelligence}, volume 161 of \emph{Proceedings of Machine Learning
  Research}, pages 1798--1808. PMLR, 27--30 Jul 2021.
\newblock URL \url{https://proceedings.mlr.press/v161/finkelstein21b.html}.

\bibitem[Geiger and Meek(2013)]{geiger2013}
D.~Geiger and C.~Meek.
\newblock Quantifier elimination for statistical problems, 2013.
\newblock URL \url{https://arxiv.org/abs/1301.6698}.

\bibitem[Imai and Jiang(2019)]{imai2019discussion}
K.~Imai and Z.~Jiang.
\newblock Discussion of "the blessings of multiple causes" by wang and blei,
  2019.

\bibitem[Imbens and Rubin(2015)]{imbens2015causal}
G.~W. Imbens and D.~B. Rubin.
\newblock \emph{Causal inference in statistics, social, and biomedical
  sciences}.
\newblock Cambridge University Press, 2015.

\bibitem[Janzing and Schölkopf(2018)]{janzing2018}
D.~Janzing and B.~Schölkopf.
\newblock Detecting confounding in multivariate linear models via spectral
  analysis.
\newblock \emph{Journal of Causal Inference}, 6\penalty0 (1):\penalty0
  20170013, 2018.
\newblock \doi{doi:10.1515/jci-2017-0013}.
\newblock URL \url{https://doi.org/10.1515/jci-2017-0013}.

\bibitem[Kilbertus et~al.(2020)Kilbertus, Kusner, and Silva]{kilbertus2020}
N.~Kilbertus, M.~J. Kusner, and R.~Silva.
\newblock A class of algorithms for general instrumental variable models.
\newblock In H.~Larochelle, M.~Ranzato, R.~Hadsell, M.F. Balcan, and H.~Lin,
  editors, \emph{Advances in Neural Information Processing Systems}, volume~33,
  pages 20108--20119. Curran Associates, Inc., 2020.
\newblock URL
  \url{https://proceedings.neurips.cc/paper/2020/file/e8b1cbd05f6e6a358a81dee52493dd06-Paper.pdf}.

\bibitem[Ogburn et~al.(2019)Ogburn, Shpitser, and Tchetgen]{OST2019blessings}
E.~L. Ogburn, I.~Shpitser, and E.~J. Tchetgen.
\newblock Comment on “blessings of multiple causes”.
\newblock \emph{Journal of the American Statistical Association}, 114\penalty0
  (528):\penalty0 1611--1615, 2019.
\newblock \doi{10.1080/01621459.2019.1689139}.
\newblock URL \url{https://doi.org/10.1080/01621459.2019.1689139}.

\bibitem[Pearl(2009)]{pearl2009causality}
J.~Pearl.
\newblock \emph{Causality: Models, Reasoning and Inference}.
\newblock Cambridge University Press, USA, 2nd edition, 2009.
\newblock ISBN 052189560X.

\bibitem[Poderini et~al.(2020)Poderini, Chaves, Agresti, Carvacho, and
  Sciarrino]{poderini2020}
D.~Poderini, R.~Chaves, I.~Agresti, G.~Carvacho, and F.~Sciarrino.
\newblock Exclusivity graph approach to instrumental inequalities.
\newblock In Ryan~P. Adams and Vibhav Gogate, editors, \emph{Proceedings of The
  35th Uncertainty in Artificial Intelligence Conference}, volume 115 of
  \emph{Proceedings of Machine Learning Research}, pages 1274--1283. PMLR,
  22--25 Jul 2020.
\newblock URL \url{https://proceedings.mlr.press/v115/poderini20a.html}.

\bibitem[Ranganath and Perotte(2019)]{ranganath2019multiple}
Rajesh Ranganath and Adler Perotte.
\newblock Multiple causal inference with latent confounding, 2019.

\bibitem[Richardson et~al.(2014)Richardson, Hudgens, Gilbert, and
  Fine]{richardson2014}
A.~Richardson, M.~G. Hudgens, P.B. Gilbert, and J.P. Fine.
\newblock Nonparametric bounds and sensitivity analysis of treatment effects.
\newblock \emph{Stat Sci}, 29\penalty0 (4):\penalty0 596--618, 2014.
\newblock \doi{10.1214/14-STS499}.

\bibitem[Sachs et~al.(2020)Sachs, Gabriel, and Sjolander]{sachs20}
M.~C. Sachs, E.~E. Gabriel, and A.~Sjolander.
\newblock Symbolic computation of tight causal bounds.
\newblock \emph{Biometrika}, 103\penalty0 (1):\penalty0 1--19, 2020.

\bibitem[Sachs et~al.(2022)Sachs, Jonzon, Sj{\"o}lander, and
  Gabriel]{sachs2021general}
Michael~C Sachs, Gustav Jonzon, Arvid Sj{\"o}lander, and Erin~E Gabriel.
\newblock A general method for deriving tight symbolic bounds on causal
  effects.
\newblock \emph{Journal of Computational and Graphical Statistics}, pages
  1--10, 2022.

\bibitem[Shridharan and Iyengar(2022)]{pmlr-v162-shridharan22a}
Madhumitha Shridharan and Garud Iyengar.
\newblock Scalable computation of causal bounds.
\newblock In \emph{Proceedings of the 39th International Conference on Machine
  Learning}, volume 162 of \emph{Proceedings of Machine Learning Research},
  pages 20125--20140. PMLR, 17--23 Jul 2022.
\newblock URL \url{https://proceedings.mlr.press/v162/shridharan22a.html}.

\bibitem[Sjölander et~al.(2014)Sjölander, Lee, Källberg, and
  Pawitan]{sjolanderbounds2014}
A.~Sjölander, W.~Lee, H.~Källberg, and Y.~Pawitan.
\newblock Bounds on causal interactions for binary outcomes.
\newblock \emph{Biometrics}, 70\penalty0 (3):\penalty0 500--505, 2014.
\newblock ISSN 0006341X, 15410420.
\newblock URL \url{http://www.jstor.org/stable/24538083}.

\bibitem[Tran and Blei(2017)]{tran2017implicit}
D.~Tran and D.~M. Blei.
\newblock Implicit causal models for genome-wide association studies, 2017.

\bibitem[Wang and Blei(2019{\natexlab{a}})]{blei2019blessings}
Y.~Wang and D.~M. Blei.
\newblock The blessings of multiple causes.
\newblock \emph{Journal of the American Statistical Association}, 114\penalty0
  (528):\penalty0 1574--1596, 2019{\natexlab{a}}.

\bibitem[Wang and Blei(2021)]{pmlr-v139-wang21c}
Y.~Wang and D.~M. Blei.
\newblock A proxy variable view of shared confounding.
\newblock In Marina Meila and Tong Zhang, editors, \emph{Proceedings of the
  38th International Conference on Machine Learning}, volume 139 of
  \emph{Proceedings of Machine Learning Research}, pages 10697--10707. PMLR,
  18--24 Jul 2021.
\newblock URL \url{https://proceedings.mlr.press/v139/wang21c.html}.

\bibitem[Wang and Blei(2019{\natexlab{b}})]{blessingsrejoinder2019}
Yixin Wang and David~M. Blei.
\newblock The blessings of multiple causes: Rejoinder.
\newblock \emph{Journal of the American Statistical Association}, 114\penalty0
  (528):\penalty0 1616--1619, 2019{\natexlab{b}}.
\newblock \doi{10.1080/01621459.2019.1690841}.
\newblock URL \url{https://doi.org/10.1080/01621459.2019.1690841}.

\bibitem[Zhang and Bareinboim(2017)]{bareinboimtransfer2017}
J.~Zhang and E.~Bareinboim.
\newblock Transfer learning in multi-armed bandits: A causal approach.
\newblock In \emph{Proceedings of the Twenty-Sixth International Joint
  Conference on Artificial Intelligence, {IJCAI-17}}, pages 1340--1346, 2017.
\newblock \doi{10.24963/ijcai.2017/186}.
\newblock URL \url{https://doi.org/10.24963/ijcai.2017/186}.

\bibitem[Zhang and Bareinboim(2021)]{bareinboim2021}
J.~Zhang and E.~Bareinboim.
\newblock Bounding causal effects on continuous outcome.
\newblock \emph{Proceedings of the AAAI Conference on Artificial Intelligence},
  35\penalty0 (13):\penalty0 12207--12215, May 2021.
\newblock URL \url{https://ojs.aaai.org/index.php/AAAI/article/view/17449}.

\bibitem[Zhang et~al.(2022)Zhang, Tian, and Bareinboim]{pmlr-v162-zhang22ab}
Junzhe Zhang, Jin Tian, and Elias Bareinboim.
\newblock Partial counterfactual identification from observational and
  experimental data.
\newblock In Kamalika Chaudhuri, Stefanie Jegelka, Le~Song, Csaba Szepesvari,
  Gang Niu, and Sivan Sabato, editors, \emph{Proceedings of the 39th
  International Conference on Machine Learning}, volume 162 of
  \emph{Proceedings of Machine Learning Research}, pages 26548--26558. PMLR,
  17--23 Jul 2022.
\newblock URL \url{https://proceedings.mlr.press/v162/zhang22ab.html}.

\end{thebibliography}
\end{document}